\documentclass{article}

     \PassOptionsToPackage{numbers, compress}{natbib}
 \usepackage[preprint]{neurips_2026}

\usepackage[utf8]{inputenc} 
\usepackage[T1]{fontenc}    
\usepackage{hyperref}       
\usepackage{url}            
\usepackage{booktabs}       
\usepackage{amsfonts}       
\usepackage{nicefrac}       
\usepackage{microtype}      
\usepackage{xcolor}         

\usepackage{comment}        

\usepackage{amsmath,amssymb, mathtools}
\usepackage{amsthm}

\newtheorem{lemma}{Lemma}
\newtheorem{proposition}{Proposition}

\usepackage{enumitem}

\newtheorem{innercustomlem}{Lemma}
\newenvironment{customlem}[1]
  {\renewcommand\theinnercustomlem{#1}\innercustomlem}
  {\endinnercustomlem}

\newtheorem{innercustomprop}{Proposition}
\newenvironment{customprop}[1]
  {\renewcommand\theinnercustomprop{#1}\innercustomprop}
  {\endinnercustomprop}

\DeclareMathOperator{\E}{\mathbb{E}}  
\DeclareMathOperator{\IndicatorFunc}{\mathbf{1}}  

\usepackage{graphicx}
\usepackage{booktabs}
\usepackage{multirow}
\usepackage{xcolor}
\usepackage{colortbl}
\usepackage{array}

\usepackage{algorithm}
\usepackage{algorithmic}
\usepackage{amsmath}

\usepackage{booktabs}
\usepackage{array}
\usepackage{wrapfig}

\makeatletter
\newcommand{\appendixtoc}{%
  \section*{Appendices}%
  \@starttoc{app}%
}

\newcommand{\appsection}[1]{%
  \section{#1}%
  \addcontentsline{app}{section}{\protect\numberline{\thesection}#1}%
}

\newcommand{\appsubsection}[1]{%
  \subsection{#1}%
  \addcontentsline{app}{subsection}{\protect\numberline{\thesubsection}#1}%
}

\newcommand{\appsubsubsection}[1]{%
  \subsubsection{#1}%
  \addcontentsline{app}{subsubsection}{\protect\numberline{\thesubsubsection}#1}%
}
\makeatother

\usepackage{tikz}
\usepackage{xcolor}
\usetikzlibrary{calc}

\definecolor{rbfmheavy}{HTML}{00E676}
\definecolor{fbil}{HTML}{1A73E8}
\definecolor{rbfmlight}{HTML}{FF4D00}
\definecolor{drbc}{HTML}{8E24AA}
\definecolor{bedroil}{HTML}{FFD600}

\newcommand{\legendcircle}[2]{%
  \tikz[baseline=-0.6ex]{
    \draw[#1, line width=1.2pt] (0,0) -- (0.5cm,0);
    \fill[#1] (0.25cm,0) circle (1.5pt);
  }%
  \hspace{0.35em}#2%
}

\newcommand{\legendtriangle}[2]{%
  \tikz[baseline=-0.6ex]{
    \draw[#1, line width=1.2pt] (0,0) -- (0.5cm,0);
    \path[draw=#1, fill=#1] (0.25cm,0.08cm) -- (0.19cm,-0.04cm) -- (0.31cm,-0.04cm) -- cycle;
  }%
  \hspace{0.35em}#2%
}

\newcommand{\legendsquare}[2]{%
  \tikz[baseline=-0.6ex]{
    \draw[#1, line width=1.2pt] (0,0) -- (0.5cm,0);
    \path[draw=#1, fill=#1] (0.22cm,-0.04cm) rectangle (0.28cm,0.02cm);
  }%
  \hspace{0.35em}#2%
}

\newcommand{\legenddiamond}[2]{%
  \tikz[baseline=-0.6ex]{
    \draw[#1, line width=1.2pt] (0,0) -- (0.5cm,0);
    \path[draw=#1, fill=#1]
      (0.25cm,0.08cm) --
      (0.19cm,0) --
      (0.25cm,-0.08cm) --
      (0.31cm,0) -- cycle;
  }%
  \hspace{0.35em}#2%
}

\newcommand{\legendcross}[2]{%
  \tikz[baseline=-0.6ex]{
    \draw[#1, line width=1.2pt] (0,0) -- (0.5cm,0);
    \draw[#1, line width=1pt]
      (0.22cm,-0.05cm) -- (0.28cm,0.05cm)
      (0.22cm,0.05cm) -- (0.28cm,-0.05cm);
  }%
  \hspace{0.35em}#2%
}

\definecolor{runblue}{RGB}{219,232,248}
\definecolor{walkgreen}{RGB}{220,240,225}
\definecolor{fliporange}{RGB}{252,235,215}
\definecolor{standpurple}{RGB}{235,225,248}

\usepackage{subcaption} 

\title{When Dynamics Shift, Robust Task Inference Wins: Offline Imitation Learning with Behavior Foundation Models Revisited}

\author{
  Rishabh Agrawal \quad Rahul Jain \quad Ashutosh Nayyar \\
  University of Southern California \\
  \texttt{rishabha@usc.edu, rahul.jain@usc.edu, ashutosn@usc.edu}
}

\hypersetup{
  hidelinks
}

\begin{document}

\maketitle

\begin{abstract}
  Behavior Foundation Models (BFMs) enable scalable imitation learning (IL) by pretraining task-agnostic representations that can be rapidly adapted to new tasks. However, existing BFMs assume fixed environment dynamics, limiting their robustness under real-world shifts such as changes in friction, actuation, or sensor noise. We address this by formulating BFM task-inference as a robust minimax optimization problem, enabling adaptation to worst-case dynamics perturbations without modifying pretraining. To the best of our knowledge, this is the first BFM-based framework that achieves robustness to dynamics shifts while relying solely on offline data from a single nominal environment. Our approach significantly outperforms standard BFM and robust offline IL baselines under dynamics shifts. These results demonstrate that robust policy can be achieved entirely at task-inference time, improving the practicality of BFMs in dynamic settings.
\end{abstract}

\section{Introduction}

Deploying autonomous agents in complex, unstructured environments demands learning behavioral policies that generalize broadly without hand-crafted reward functions or controllers~\citep{ho2016generative, agrawal2025markov}. Imitation Learning (IL), wherein an agent acquires behavior directly from expert demonstrations, offers a compelling alternative, enabling remarkable successes in robotic manipulation~\citep{an2025dexterous}, language generation~\citep{achiam2023gpt, rafailov2023direct}, and autonomous control~\citep{codevilla2019exploring}. Yet conventional IL methods learn each task in isolation, requiring retraining or large demonstration sets for every new behavior~\citep{duan2017one, cho2024meta}, limiting their scalability.

\begin{figure}[t]
\centering
\includegraphics[width=\linewidth]{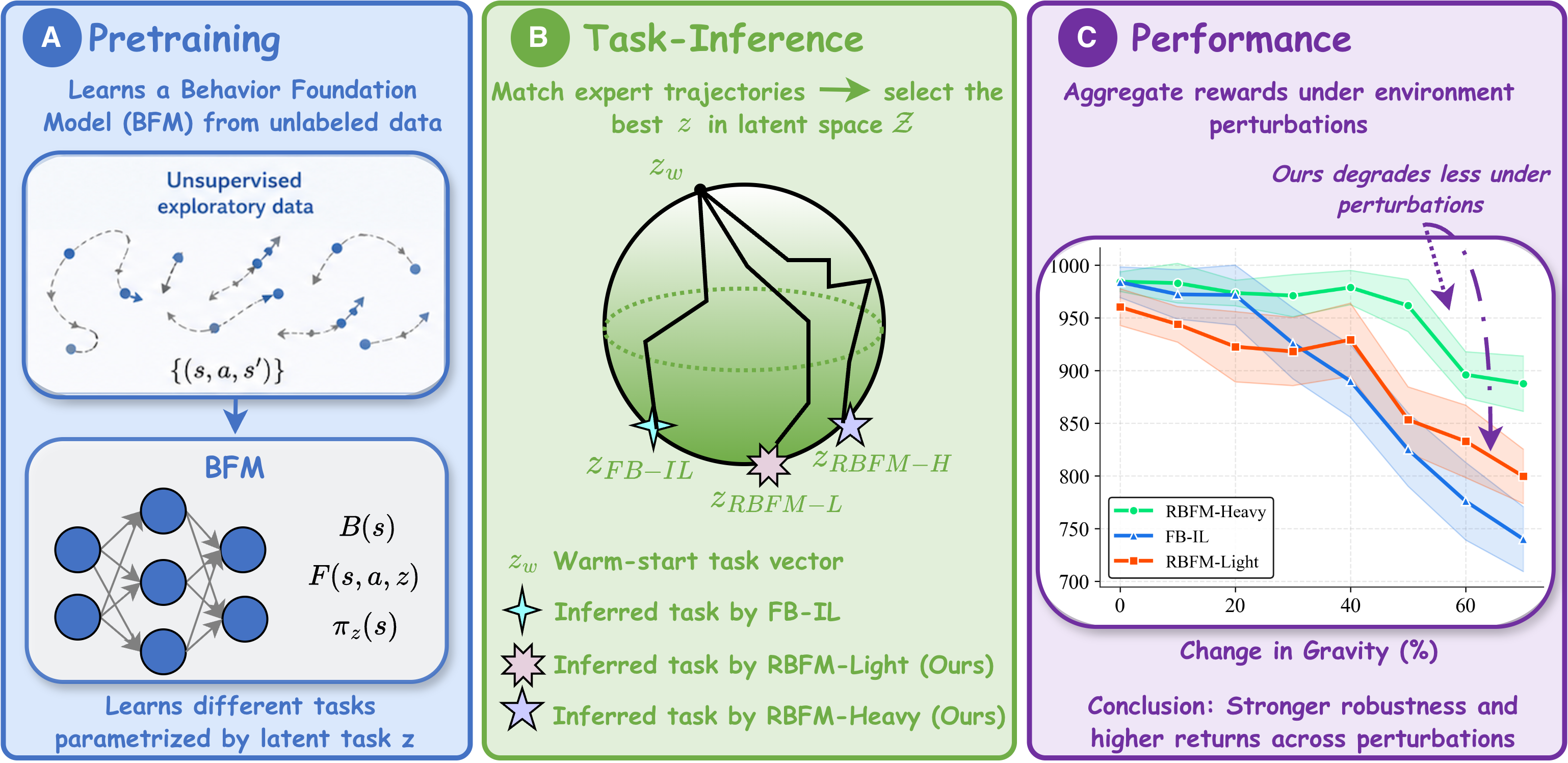}
\caption{Overview of the proposed RBFM framework. FB-IL, RBFM-Light, and RBFM-Heavy share an identical pretraining procedure; robustness is introduced only at the task inference stage. RBFM variants produce policies significantly more robust to dynamic perturbations than FB-IL.}
\label{fig:bfm_approach}
\end{figure}

Behavior Foundation Models (BFMs)~\citep{touati2021learning, touati2023does} address this scalability bottleneck by pretraining a single model on diverse exploratory offline data in an unsupervised manner, capturing task-agnostic environment representations via occupancy measures or successor measures. Forward-backward (FB) decompositions~\citep{pirotta2024fast} factorize these successor measures into policy-dependent and global components, yielding linear representations that generalize across tasks. At inference time, BFMs adapt rapidly to new tasks from limited demonstrations by optimizing a low-dimensional task embedding $z$, producing competitive imitation policies within seconds to minutes without full retraining~\citep{pirotta2024fast}. This \textit{pretrain once, adapt anywhere} paradigm is a compelling solution for scalable multi-task IL in fixed environments.

However, BFMs fundamentally assume that the dynamics at pretraining and deployment coincide, an assumption routinely violated in practice due to sim-to-real discrepancies, mechanical wear, sensor noise, and environmental variation~\citep{tiwari2026reinforcement, hu2025impact}. 
Recent work demonstrates that standard BFMs fail to adapt when environment dynamics shift \citep{bobrin2026zeroshot, jeen2025zero}. Proposed remedies: dynamics-conditioned belief encoders~\citep{bobrin2026zeroshot} and memory-augmented architectures~\citep{jeen2025zero}, both require pretraining data spanning diverse dynamics variations to learn meaningful representations, a requirement that is expensive, time-consuming, and often infeasible in real-world settings~\citep{peng2018sim, aljalbout2025reality}. This raises the following question:
\begin{center}
\textit{Can BFMs be made robust to environment dynamics shift when offline data from only a single nominal environment is available?}
\end{center}
We answer affirmatively by formulating BFM task inference as a robust minimax optimization problem. Keeping the pretrained model fixed, we let an adversary perturb the environment within a specified uncertainty set while the learner minimizes the induced adversarial imitation loss, yielding a task embedding that hedges against plausible dynamics shifts. To preserve the fast task-inference characteristic of BFMs, we propose two variants (Figure~\ref{fig:bfm_approach}): \textit{RBFM-Light} augments the inference objective with a single robustness parameter optimized jointly with $z$, achieving robustness with minimal computational overhead. \textit{RBFM-Heavy} further constrains the policy-induced occupancy measure to lie within the feasible distribution set, enforcing realizability and improving robustness at the cost of a larger optimization problem. Both variants require only lightweight modifications to the task-inference stage, with no changes to pretraining. Empirically, RBFM-Light and RBFM-Heavy consistently and substantially outperform standard BFMs (FB-IL) under diverse dynamics perturbations across multiple environments, while remaining competitive in the nominal setting (data-collecting environment) and orders of magnitude faster than single-task robust IL baselines such as DRBC \citep{panaganti2023distributionally} and BE-DROIL \citep{agrawal2025balance}. To summarize, our contributions are as follows:
\begin{enumerate}
    \item We propose a robust formulation of BFM task inference via minimax optimization over an uncertainty set of environment dynamics, enabling robustness to dynamics shift without any modification to BFM pretraining.
    \item We introduce \textit{RBFM-Light} and \textit{RBFM-Heavy}, two algorithms that trade computational efficiency against robustness while preserving the fast task-inference property of BFMs.
    \item We demonstrate consistent and substantial robustness gains across three morphologically diverse environments under a wide range of perturbations. Our approach requires only a single exploratory dataset collected in the nominal setting for pretraining, unlike \citep{bobrin2026zeroshot,jeen2025zero}, which rely on data from multiple diverse environments, and only a small set of expert demonstrations for task inference to produce a robust policy. Importantly, no expert data from perturbed environments is required at any stage.
    \item To the best of our knowledge, this is the first work to achieve robustness to environment dynamics variation within the BFM framework without requiring pretraining or expert data from multiple dynamics regimes. This increases the range of real-world deployments for BFM-based methodology.
\end{enumerate}

\section{Preliminaries}
\noindent\textbf{Markov decision processes.}
We consider a discounted, reward-free Markov decision process (MDP) $\mathcal{M} = (\mathcal{S}, \mathcal{A}, P, \gamma, \mu)$. Here, $\mathcal{S}$ and $\mathcal{A}$ denote the state and action spaces, $P(\mathrm{d}s' \mid s,a)$ is the transition kernel, $\gamma \in (0,1)$ is the discount factor, and the initial state is drawn as $s_0 \sim \mu$. A policy $\pi$ defines a distribution over actions conditioned on the state, denoted by $\pi(a|s)$. 
Starting from $s_0 \sim \mu$, trajectories $(s_t,a_t)_{t \ge 0}$ evolve according to $s_t \sim P(\mathrm{d}s_t \mid s_{t-1}, a_{t-1})$ and $a_t \sim \pi(\cdot \mid s_t)$.

\noindent\textbf{Successor measure.}
To characterize the evolution of future states under a policy $\pi$, we consider the successor measure $M^\pi(\cdot \mid s,a)$ when taking action $a$ in state $s$, which aggregates discounted future state visitation. For any measurable set $X \subset \mathcal{S}$,
\begin{equation}
M^\pi(X \mid s,a)
=
\sum_{t=0}^{\infty} \gamma^t
\mathbb{P}_\pi(s_{t+1} \in X \mid s,a).
\end{equation}
This construction naturally separates the effect of the dynamics from the specification of a task. In particular, for any reward function $r : \mathcal{S} \to \mathbb{R}$, the associated Q-function can be written as
\begin{equation}
Q_r^\pi(s,a)
=
\mathbb{E}_\pi\!\left[
\sum_{t=0}^{\infty} \gamma^t r(s_{t+1})
\;\middle|\; s,a
\right]
=
\int_{\mathcal{S}} M^\pi(\mathrm{d}s' \mid s,a)\, r(s').
\end{equation}
Thus, once the successor measure is known, value functions for different rewards can be obtained without modifying the policy-dependent component.

\noindent\textbf{Forward--backward (FB) parameterization.}
In practice, directly representing $M^\pi$ is intractable in large or continuous state spaces. A common approach is therefore to approximate it using a low-rank decomposition. Concretely, Behavior Foundation Models (BFMs) represents $M^\pi$ as
\begin{equation}
M^\pi(X \mid s,a)
\approx
\int_{s' \in X} F^\pi(s,a)^\top B(s') \, \phi(\mathrm{d}s'),
\end{equation}
where $F^\pi : \mathcal{S}\times\mathcal{A} \to \mathbb{R}^d$ and $B : \mathcal{S} \to \mathbb{R}^d$ denote forward and backward embeddings, respectively, and $\phi$ is a reference distribution over states. Under this factorization, the dependence on the reward reduces to a linear projection:
\begin{equation}
Q_r^\pi(s,a) = F^\pi(s,a)^\top z,
\qquad
z = \mathbb{E}_{s \sim \phi}[B(s)\, r(s)].
\end{equation}

To enable generalization across tasks, both the policy and the forward embedding are conditioned on a latent vector $z \in \mathcal{Z} \subseteq \mathbb{R}^d$ \citep{touati2023does}. This yields the parameterization
\begin{equation}
\begin{cases}
M^{\pi_z}(X \mid s,a)
\approx
\int_{s' \in X} F(s,a,z)^\top B(s') \, \phi(\mathrm{d}s'), \\
\pi_z(s) = \arg\max_a F(s,a,z)^\top z.
\end{cases}
\end{equation}
In practice, the latent space $\mathcal{Z}$ is typically constrained (e.g., to a bounded subset of $\mathbb{R}^d$). In the \textit{pretraining} stage, the embeddings are learned from an offline dataset with distribution $\phi$ by minimizing a temporal-difference objective derived from the Bellman equation \citep{touati2023does}:
\begin{align}
\mathcal{L}_{\mathrm{FB}}(F,B)
&=
\mathbb{E}_{\substack{z \sim \nu, (s,a,s') \sim \phi\\ s^+ \sim \phi,\; a' \sim \pi_z(\cdot \mid s')}} 
\Big[
\big(
F(s,a,z)^\top B(s^+)
-
\gamma \bar{F}(s',a',z)^\top \bar{B}(s^+)
\big)^2
\Big] \nonumber \\
&\quad
- 2 \, \mathbb{E}_{\substack{z \sim \nu, (s,a,s') \sim \phi}}
\big[ F(s,a,z)^\top B(s') \big],
\label{eqn:bfm_pretraining}
\end{align}
where $\bar{F}$ and $\bar{B}$ denote stop-gradient operators and $\nu$ is a distribution over latent variables. For continuous action spaces, the maximization in $\pi_z$ is replaced by an actor network trained to maximize:
\begin{equation}
\mathcal{L}_{\mathrm{actor}}(\pi)
=
- \mathbb{E}_{\substack{z \sim \nu, s \sim \phi,\; a \sim \pi_z(\cdot \mid s)}}
\big[ F(s,a,z)^\top z \big].
\label{eqn:actor_pretraining}
\end{equation}

The overall procedure consists of two phases. In the first phase (pretraining), we learn the forward and backward embeddings, along with the actor (when applicable), offline from a fixed exploratory dataset by minimizing $\mathcal{L}_{\mathrm{FB}}$ and $\mathcal{L}_{\mathrm{actor}}$, without access to task-specific rewards. In the second phase, we perform task-inference by selecting a latent variable $z$ that induces a policy $\pi_z$. The latent $z$ can be inferred from rewards, goals, or expert demonstrations; since our focus is imitation learning, in this work, we will consider the reward-free setting with expert demonstration data.

\noindent\textbf{Learning from demonstrations.}
Since no reward function is available for task-inference, learning must rely solely on expert demonstrations, without any additional interaction with the environment in the offline setting. To formalize this, we describe policies through their discounted occupancy measures under a nominal transition function $T^o$ as $\rho^{\pi}_{T^{o}}(s,a,s')
=
(1-\gamma)\,
\mathbb{E}_\pi\!\left[
\sum_{t=0}^{\infty}
\gamma^t
\IndicatorFunc_{s_t = s, a_t = a, s_{t+1} = s'}
\right]$. The marginal state–action and state-only occupancy measures are then given by $\rho^{\pi}_{T^{o}}(s, a) = \sum_{s'} \rho^{\pi}_{T^{o}}(s,a,s')$ and $\rho^{\pi}_{T^{o}}(s) = \sum_{a,s'} \rho^{\pi}_{T^{o}}(s,a,s')$ respectively. Given a dataset of trajectories generated by an expert policy $\pi_D$, the goal is to identify a latent task-vector $z$ such that the induced policy $\pi_z$ matches the expert behavior. A natural objective is
\begin{equation}
\arg\min_{z \in \mathcal{Z}}\;
\mathbb{E}_{s \sim \rho^{\pi_D}_{T^{o}}}
\big[
\mathcal{L}(\pi_z(\cdot|s), \pi_D(\cdot|s))
\big],
\end{equation}
where $\mathcal{L}$ is the \textit{imitation loss} (e.g., mean squared error or KL divergence) measuring the gap between learner and expert policies. We refer to this objective as FB-IL, which serves as our primary baseline.

\noindent\textbf{Robust formulation.}
Finally, we account for uncertainty in the transition dynamics by introducing an ambiguity or uncertainty set over Transition kernels. Specifically, we consider
\begin{equation}
\mathcal{T}(\varepsilon') =
\left\{
\left( T_{s,a} \right)_{(s,a)\in \mathcal{S}\times\mathcal{A}}
\;\middle|\;
T_{s,a} \in \Delta(\mathcal{S}),\;
D_{\mathrm{TV}}\!\left(T_{s,a}, T^{o}_{s,a}\right) \le \varepsilon',
\;\forall (s,a)\in \mathcal{S}\times\mathcal{A}
\right\}.
\end{equation}

This induces a distributionally robust objective in which performance is evaluated under the worst-case transition model within the ambiguity set. Formally, we consider \citep{agrawal2025balance}
\begin{equation}
\arg\min_{z \in \mathcal{Z}}
\;
\max_{T \in \mathcal{T}}
\;
\E_{s \sim \rho^{\pi_D}_{T}}
\big[
\mathcal{L}(\pi_z(\cdot|s), \pi_D(\cdot|s))
\big],
\label{eq:robust_il}
\end{equation}

where $\rho^{\pi_D}_{T}$ denotes the discounted state distribution induced by the expert policy $\pi_D$ under transition kernel $T \in \mathcal{T}$. The objective admits a natural minimax interpretation: an adversary selects a transition model $T \in \mathcal{T}(\varepsilon')$ that maximizes the imitation loss, while the learner chooses a latent variable $z$ (and corresponding induced-policy $\pi_z$) to minimize this worst-case loss. Crucially, this game is played in the offline regime, where only trajectories generated by $\pi_D$ under $T^o$ are available.

\section{Fast Robust Offline Imitation Learning} \label{sec:methodolgy}
We propose two complementary formulations for robust and computationally efficient task-inference of pre-trained Behavioral Foundation Models (BFMs) under transition uncertainty. The first variant, \textbf{RBFM-Light} (\S\ref{subsec:rbfm_lite}), relaxes dynamics perturbations into a state-occupancy uncertainty set, yielding a tractable offline objective. The second variant, \textbf{RBFM-Heavy} (\S\ref{subsec:rbfm_heavy}), enforces robustness through explicit Bellman flow constraints, leading to less conservative policies at the expense of increased computational complexity.

\subsection{RBFM-Light} \label{subsec:rbfm_lite}
We begin by characterizing the sensitivity of the discounted state-occupancy measure $\rho_T^\pi(s)$ to perturbations in the transition kernel $T$.

\begin{lemma}
For any policy $\pi$ and transition kernel $T \in \mathcal{T}(\varepsilon')$, the following holds:
\begin{equation*}
    D_{\mathrm{TV}}(\rho^\pi_T(s), \rho^\pi_{T^o}(s)) \le \frac{\gamma \varepsilon'}{1 - \gamma}.
\end{equation*}
\label{lemma:D_TV_mismatch_s}
\end{lemma}
Lemma~\ref{lemma:D_TV_mismatch_s} provides a uniform bound on the deviation between occupancy measures induced by perturbed and nominal dynamics. This enables replacing uncertainty over transition kernels with an uncertainty set defined directly in occupancy space, yielding a tractable relaxation of the robust objective in \eqref{eq:robust_il}. Concretely, we define $\mathcal{D}^{\pi_D}_{\mathrm{light}}(\varepsilon_l) 
= 
\Big\{
\rho^{\pi_D}_{T}(s) : 
D_{\mathrm{TV}}\!\left(\rho^{\pi_D}_{T}(s),\, \rho^{\pi_D}_{T^{o}}(s)\right)
\le \frac{\gamma\varepsilon'}{1 - \gamma} = \varepsilon_l
\Big\}$, which contains all occupancy measures induced by admissible kernels $T \in \mathcal{T}(\varepsilon')$. By construction, optimizing over $\mathcal{D}^{\pi_D}_{\mathrm{light}}(\varepsilon)$ upper bounds the original worst-case objective over $\mathcal{T}(\varepsilon')$. This reformulation removes explicit dependence on the unknown dynamics and provides a purely offline optimization problem. The resulting objective is given by
\begin{equation}
\begin{aligned}
\text{RBFM-Light}
&:=
\min_{z} \; \max_{\rho^{\pi_D}_{T} \ge 0} 
\; \mathbb{E}_{s \sim \rho^{\pi_D}_{T}}\!\left[\mathcal{L}(\pi_z(\cdot|s), \pi_D(\cdot|s))\right]
\quad
\text{s.t.} 
& D_{TV}\!\left(\rho^{\pi_D}_{T}(s) \,\|\, 
\rho^{\pi_D}_{T^{o}}(s)\right)
\le \varepsilon_l
\end{aligned}
\label{eqn:rbfm_light}
\end{equation}
The inner maximization admits a tractable dual reformulation, yielding the equivalent optimization problem provided in Proposition \ref{prop:light_sol}.
\begin{proposition}
The solution to the constrained optimization problem in \eqref{eqn:rbfm_light} is given by solving:
\begin{equation*}
\min_{z}\min_{\lambda \in \mathbb{R}}
\left\{\mathbb{E}_{s \sim \rho^{\pi_D}_{T^{o}}}\big[(\mathcal{L}(\pi_z(\cdot|s), \pi_D(\cdot|s)) - \lambda)_+\big]
+ \varepsilon_l\left(\max_{\rho^{\pi_D}_{T^{o}}(s) > 0}
\mathcal{L}(\pi_z(\cdot|s), \pi_D(\cdot|s)) - \lambda\right)_{+} + \lambda\right\}
\label{eqn:rbfm_light_sol}
\end{equation*}
\label{prop:light_sol}
\end{proposition}
The proof of Proposition \ref{prop:light_sol} is provided in the Appendix \ref{sec:appendix_proofs}. Objective defined in Proposition $\ref{prop:light_sol}$ can be optimized to solve the robust imitation learning problem in \eqref{eq:robust_il}, using the offline data only available from the nominal environment $T^{o}$. We further specialize the objective under standard assumptions commonly used in offline imitation learning.

\begin{proposition}
Suppose that the expert and learner policies $\pi_D$ and $\pi_z$ are deterministic and that the action space $\mathcal A$ is bounded. Let $\mathcal{L}(\pi_z(\cdot|s), \pi_D(\cdot|s))=\|\pi_z(s)-a\|_2^2$ and $
L:=\sup_{(s,a)\sim\mathcal \rho^{\pi_D}_{T^{o}}}\|\pi_z(s)-\pi_D(s)\|_2^2$.
Then the optimization problem in Proposition \ref{prop:light_sol} can be simplified to
\begin{equation*}
\resizebox{1.0\textwidth}{!}{$\displaystyle
\min_{z}\min_{\lambda\in[0,(1+\varepsilon)L]}
\left\{\mathbb{E}_{s \sim \rho^{\pi_D}_{T^{o}}}\big[(\mathcal{L}(\pi_z(\cdot|s), \pi_D(\cdot|s)) - \lambda)_+\big]
+ \varepsilon_l\left(\max_{\rho^{\pi_D}_{T^{o}}(s) > 0}
\mathcal{L}(\pi_z(\cdot|s), \pi_D(\cdot|s)) - \lambda\right)_{+} + \lambda\right\}
$}
\label{eqn:rbfm_light_sol_fast}
\end{equation*}
\label{prop:fast_rbfm_light}
\end{proposition}
Proposition \ref{prop:fast_rbfm_light} reduces the search space of the dual variable $\lambda$ from $\mathbb{R}$ to a bounded interval, enabling faster optimization. The proof is provided in the Appendix \ref{sec:appendix_proofs}.

\subsection{RBFM-Heavy} \label{subsec:rbfm_heavy}

While RBFM-Light provides a tractable offline solution by operating directly in occupancy space, it allows arbitrary perturbations within a TV ball around the nominal occupancy measure. In particular, the uncertainty set $\mathcal{D}^{\pi_D}_{\mathrm{light}}(\varepsilon)$ may include occupancy measures that assign probability mass to transitions that are not realizable under the expert policy in the underlying MDP. This can lead to overly conservative solutions. To address this limitation, we refine the uncertainty set by incorporating structural constraints that ensure realizability. We begin by considering the extension of Lemma~\ref{lemma:D_TV_mismatch_s} to joint state-action-next-state occupancy measures.
\begin{lemma}
For any policy $\pi$ and transition kernel $T \in \mathcal{T}(\varepsilon')$, the following holds:
\begin{equation*}
    D_{\mathrm{TV}}(\rho^\pi_T(s,a,s'), \rho^\pi_{T^o}(s,a,s')) \le \frac{\varepsilon'}{1 - \gamma}.
\end{equation*}
\label{lemma:D_TV_mismatch_s_a_s_prime}
\end{lemma}

Building on Lemma~\ref{lemma:D_TV_mismatch_s_a_s_prime}, we enforce a \emph{Bellman flow conservation} condition~\citep{puterman2014markov, altman2021constrained}, ensuring that $\rho^{\pi_D}_{T}(s,a,s')$ corresponds to a valid stationary occupancy measure induced by the expert policy $\pi_D$ under transition kernel $T$. This preserves the temporal structure of the MDP and rules out infeasible perturbations. In addition, we generalize the divergence constraint by replacing the TV distance with an $f$-divergence (see Appendix \ref{subsec:app_f_div} for details). Upon incorporating these changes, we obtain the following constrained minimax formulation:
\begin{equation}
\begin{aligned}
\text{RBFM-Heavy}
&:=
\min_{z} \; \max_{\rho^{\pi_D}_{T} \ge 0} 
\; \mathbb{E}_{s \sim \rho^{\pi_D}_{T}}\!\left[\mathcal{L}(\pi_z(\cdot|s), \pi_D(\cdot|s))\right]
\\
\text{s.t.} \quad
& (1) D_{f}\!\left(\rho^{\pi_D}_{T}(s,a,s') \,\|\, 
\rho^{\pi_D}_{T^{o}}(s,a,s')\right)
\le \varepsilon,
\\
& (2) \sum_{s'} \rho^{\pi_D}_{T}(s,a,s')
    = (1-\gamma)\mu(s)\pi_D(a|s)
    + \gamma \pi_D(a|s)\!
    \sum_{\tilde{s},\tilde{a}}
    \rho^{\pi_D}_{T}(\tilde{s},\tilde{a},s)
    \quad \forall (s,a).
\end{aligned}
\label{eqn:rbfm_heavy}
\end{equation}
To solve \eqref{eqn:rbfm_heavy}, we fix the task vector $z$ (hence the corresponding induced policy $\pi_z$) and consider the corresponding dual formulation:
\begin{equation}
\begin{aligned}
\min_{Q,\tau\ge 0}\; \max_{\rho^{\pi_D}_{T} \ge 0}
&\; \E_{s \sim \rho^{\pi_D}_{T}}\!\big[L_{\pi_z}(s)\big]
- \tau \!\left(
D_f\!\left(\rho^{\pi_D}_{T}\!\left(s,a,s'\right)\,\big\|\, \rho^{\pi_D}_{T^o}\!\left(s,a,s'\right)\right) - \varepsilon
\right) \\
&\hspace{-0.5cm}-\sum_{s,a} Q(s,a)\!\left[
\sum_{s'} \rho^{\pi_D}_{T}(s,a,s') 
- (1-\gamma)\mu(s)\pi_D(a|s)
- \gamma \pi_D(a|s)\!\sum_{\tilde{s},\tilde{a}}
\rho^{\pi_D}_{T}(\tilde{s},\tilde{a},s)
\right].
\end{aligned}
\label{eqn:lagrangian_regularized_bc}
\end{equation}

where $L_{\pi_z}(s) = \mathcal{L}(\pi_z(\cdot|s), \pi_D(\cdot|s))$, $Q(s,a)\in\mathbb{R}$ is the Lagrange multiplier for the Bellman flow constraint, and $\tau\ge0$ is the multiplier for the $f$-divergence constraint. Since the inner maximization in~\eqref{eqn:rbfm_heavy} is convex and admits a strictly feasible point (e.g., $\rho^{\pi_D}_{T^o}$), Slater’s condition~\citep{boyd2004convex} implies strong duality. Thus, \eqref{eqn:lagrangian_regularized_bc} attains the same optimal value as the original problem. After algebraic simplification, we obtain:
\begin{equation}
\begin{aligned}
\min_{Q,\tau\ge 0}\; \max_{\rho^{\pi_D}_{T} \ge 0}
&\; (1-\gamma)\E_{s\sim\mu,\,a\sim\pi_D(\cdot|s)}[Q(s,a)]
- \tau\, D_{f}\!\left(\rho^{\pi_D}_{T}(s,a,s') \,\|\, \rho^{\pi_D}_{T^{o}}(s,a,s')\right) + \epsilon\tau \\
&\quad + \E_{s,a,s'\sim \rho^{\pi_D}_{T}}\!\Big[L_{\pi_z}(s)
+ \gamma \E_{a'\sim\pi_D(\cdot|s')}[Q(s',a')] - Q(s,a)\Big].
\end{aligned}
\label{eqn:simplified_lagrangian_regularized_bc}
\end{equation}

This formulation still depends on the unknown occupancy $\rho^{\pi_D}_{T}(s,a,s')$ for arbitrary $T \in \mathcal{T}(\epsilon')$, and is therefore not directly amenable to offline optimization. To address this, we introduce the importance ratio $
w(s,a,s') = \frac{\rho^{\pi_D}_{T}(s,a,s')}{\rho^{\pi_D}_{T^o}(s,a,s')}$. Using the definition of $f$-divergence, we rewrite \eqref{eqn:simplified_lagrangian_regularized_bc} as
\begin{equation}
\begin{aligned}
\min_{Q,\tau\ge 0}\; \max_{w \ge 0}\; 
& (1-\gamma)\E_{s\sim\mu,\,a\sim\pi_D(\cdot|s)}[Q(s,a)] + \epsilon\tau \\
&\quad + \E_{s,a,s'\sim \rho^{\pi_D}_{T^o}}\!\left[-\tau\, f(w(s,a,s')) + w(s,a,s')\, c_{Q,\pi_z}(s,a,s')\right],
\end{aligned}
\label{eqn:final_lagrangian_constrained_bc}
\end{equation}
where $c_{Q,\pi_z}(s,a,s') = L_{\pi_z}(s) + \gamma\E_{a'\sim\pi_D(\cdot|s')}[Q(s',a')] - Q(s,a).$

This reformulation transforms the original maximization over unknown occupancy measures $\rho^{\pi_D}_{T}$ for any $T \in \mathcal{T}(\epsilon')$ into an objective that depends only on samples from the nominal distribution $\rho^{\pi_D}_{T^o}$. As a result, the problem becomes tractable in a purely offline setting using expert demonstrations alone. We now turn to solving the inner maximization over the importance weights $w$, which admits a closed-form solution.

\begin{proposition}
\label{proposition:optimal_importance_weight}
For $\tau > 0$, the inner maximization in~\eqref{eqn:final_lagrangian_constrained_bc} admits the solution
\[
w^{\star}_{Q,\tau,\pi_z}(s,a,s') =
\max\!\left(0,\,(f')^{-1}\!\left(\frac{c_{Q,\pi_z}(s,a,s')}{\tau}\right)\right).
\]
For $\tau = 0$, $w^{\star}_{Q,\tau,\pi_z}(s,a,s') = +\infty$ if $c_{Q,\pi_z}(s,a,s')>0$ and $0$ otherwise.
\end{proposition}

The proof of Proposition~\ref{proposition:optimal_importance_weight} is provided in Appendix \ref{sec:appendix_proofs}. Substituting $w^{\star}_{Q,\tau,\pi_z}$ into~\eqref{eqn:final_lagrangian_constrained_bc} reduces the problem to a minimization over $Q$ and $\tau$. This results in a bilevel-like optimization with two coupled subproblems: estimating $(Q,\tau)$ and updating the task parameters.
\begin{equation}
\begin{alignedat}{1}
\min_{Q,\tau\ge 0}\quad 
& (1-\gamma)\E_{s\sim\mu,\,a\sim\pi_D(\cdot|s)}[Q(s,a)] + \epsilon\tau \\
&\quad + \E_{(s,a,s')\sim \rho^{\pi_D}_{T^o}}\!\left[-\tau f\!\left(w^{\star}_{Q,\tau,\pi_z}(s,a,s')\right)
    + w^{\star}_{Q,\tau,\pi_z}(s,a,s')\, c_{Q,\pi_z}(s,a,s')\right] \\[3mm]
\min_{z}\quad 
& \E_{(s,a,s')\sim \rho^{\pi_D}_{T^o}}\!\left[w^{\star}_{Q,\tau,\pi_z}(s,a,s')\, L_{\pi_z}(s)\right].
\end{alignedat}
\label{eqn:final_optimization_problem}
\end{equation}
Optimization alternates between updates of $(Q,\tau)$ and the task-vector $z$ until convergence, yielding a policy $\pi_z$ that is robust to worst-case transition perturbations in the uncertainty set.

\begin{figure}[t]
\centering
\setlength{\tabcolsep}{4pt}
\renewcommand{\arraystretch}{0.0}

\begin{tabular}{>{\centering\arraybackslash}m{0.5cm}
                >{\centering\arraybackslash}m{0.30\textwidth}
                >{\centering\arraybackslash}m{0.30\textwidth}
                >{\centering\arraybackslash}m{0.30\textwidth}}

&
\textbf{\hspace{0.6cm}Tilted Gravity} &
\textbf{\hspace{0.4cm}Ground Contact Stiffness} &
\textbf{\hspace{0.5cm}Actuator Ctrlrange} \\[2pt]

\cellcolor{runblue}\rotatebox{90}{\small\textbf{Run}} &
\cellcolor{runblue}\includegraphics[width=\linewidth]{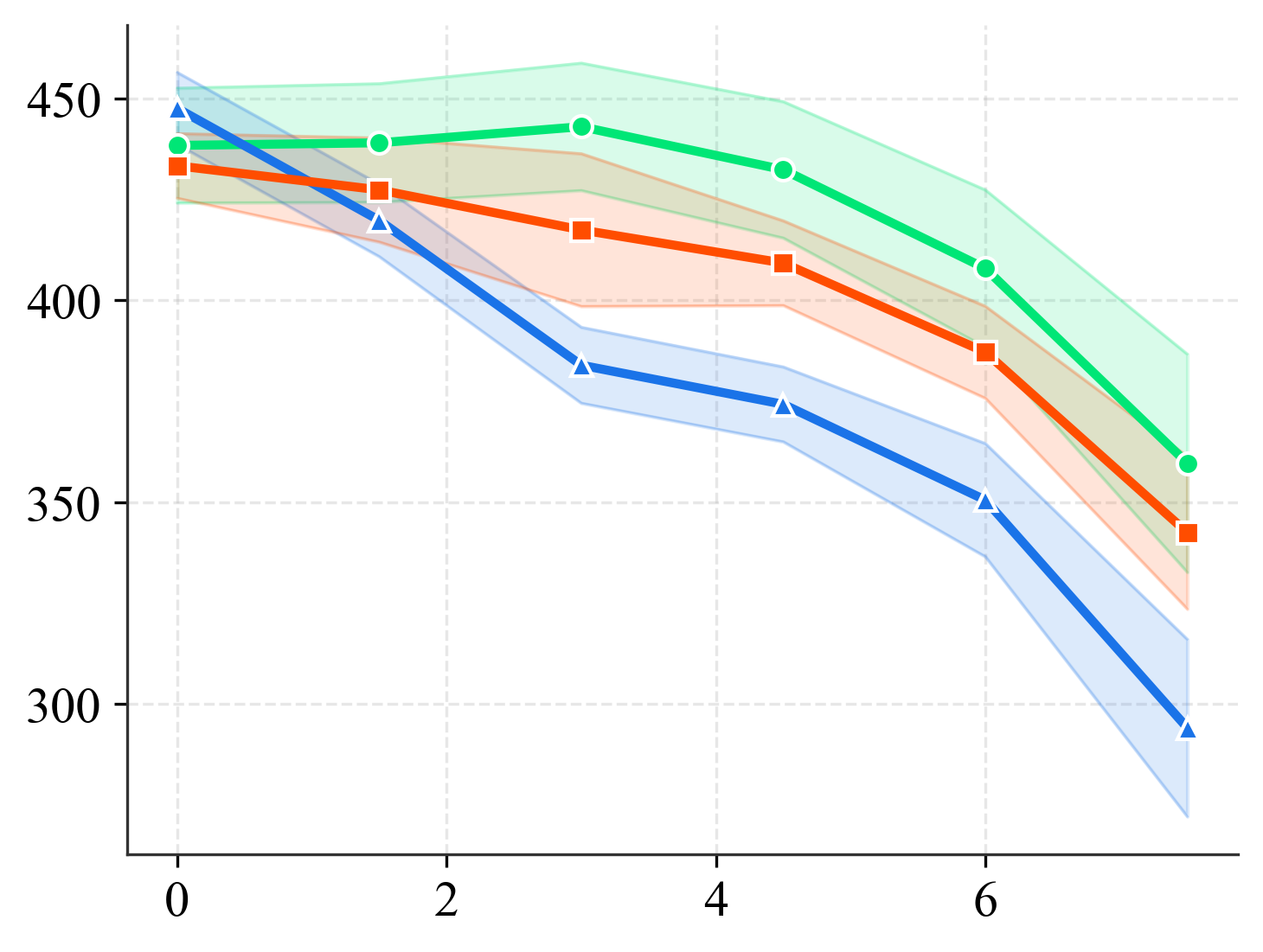} &
\cellcolor{runblue}\includegraphics[width=\linewidth]{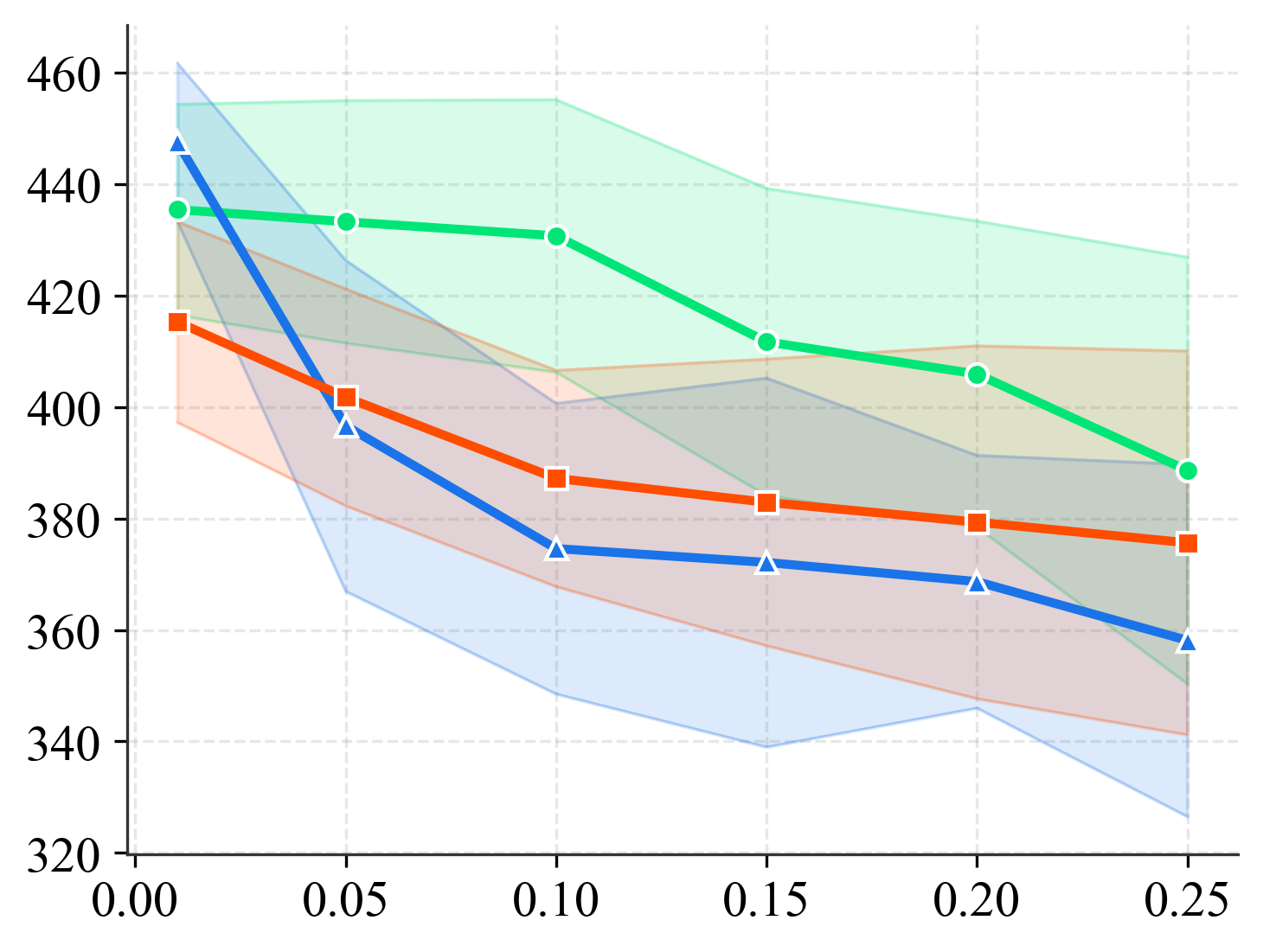} &
\cellcolor{runblue}\includegraphics[width=\linewidth]{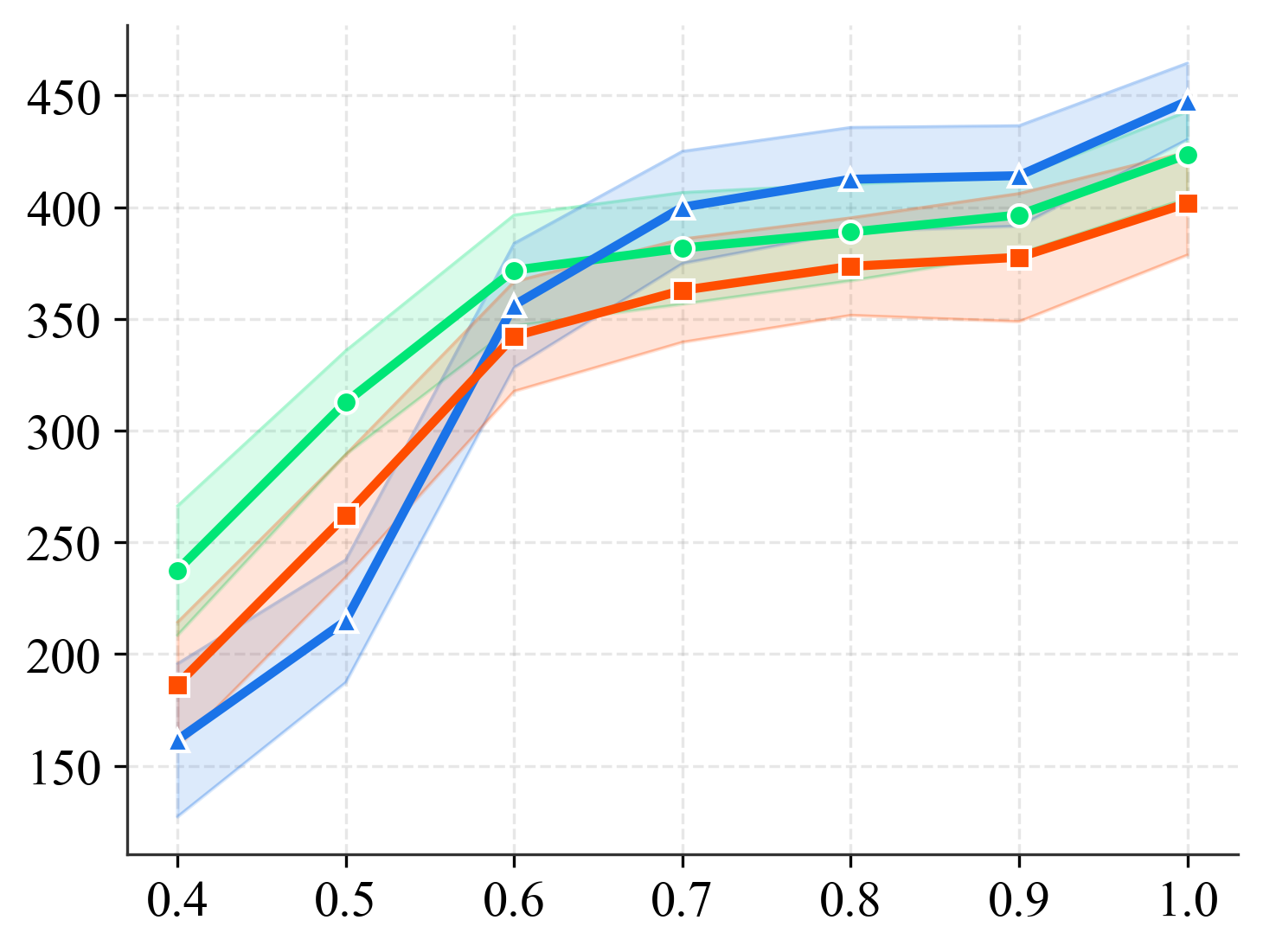} \\[1pt]

\cellcolor{walkgreen}\rotatebox{90}{\small\textbf{Walk}} &
\cellcolor{walkgreen}\includegraphics[width=\linewidth]{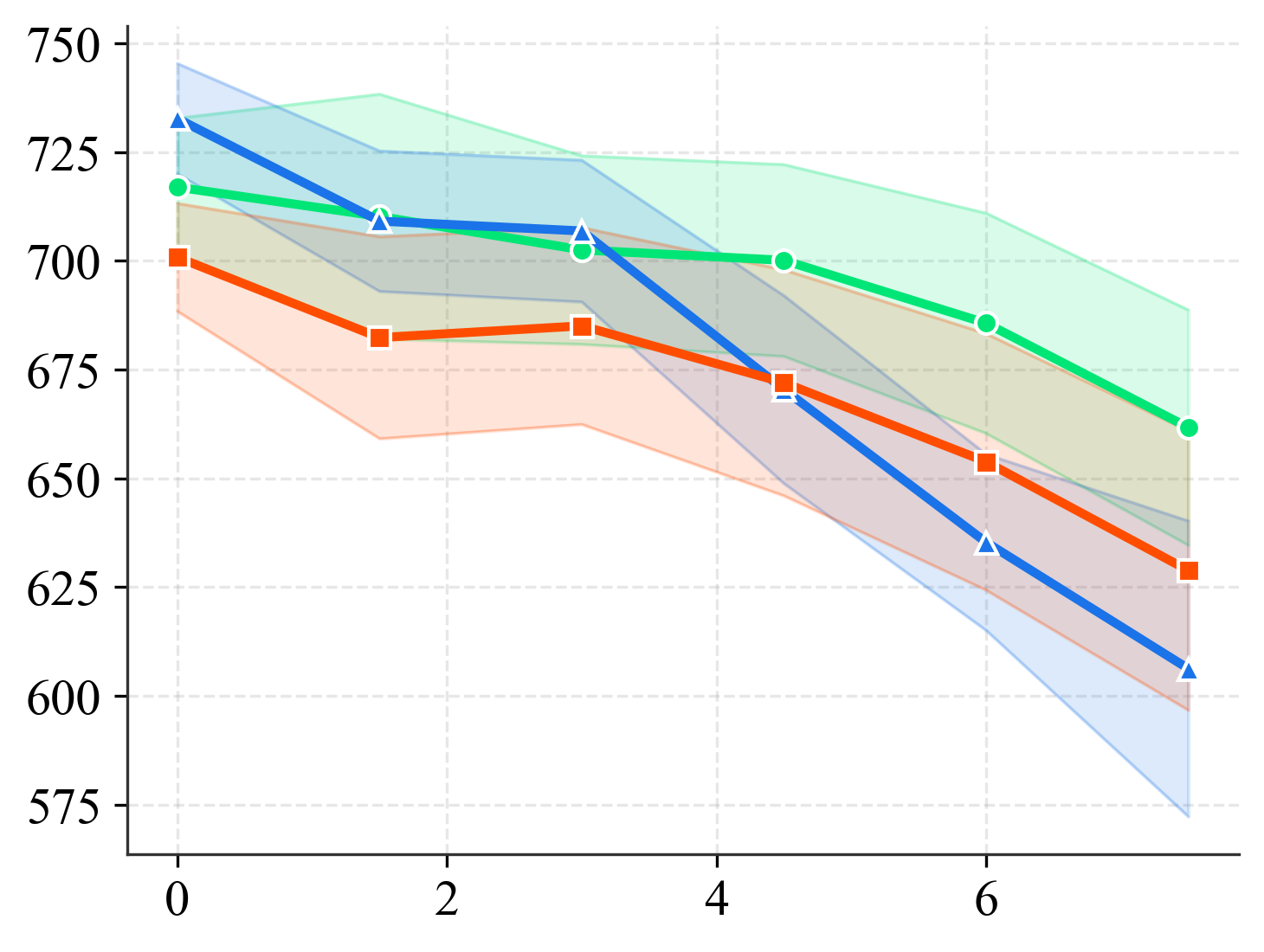} &
\cellcolor{walkgreen}\includegraphics[width=\linewidth]{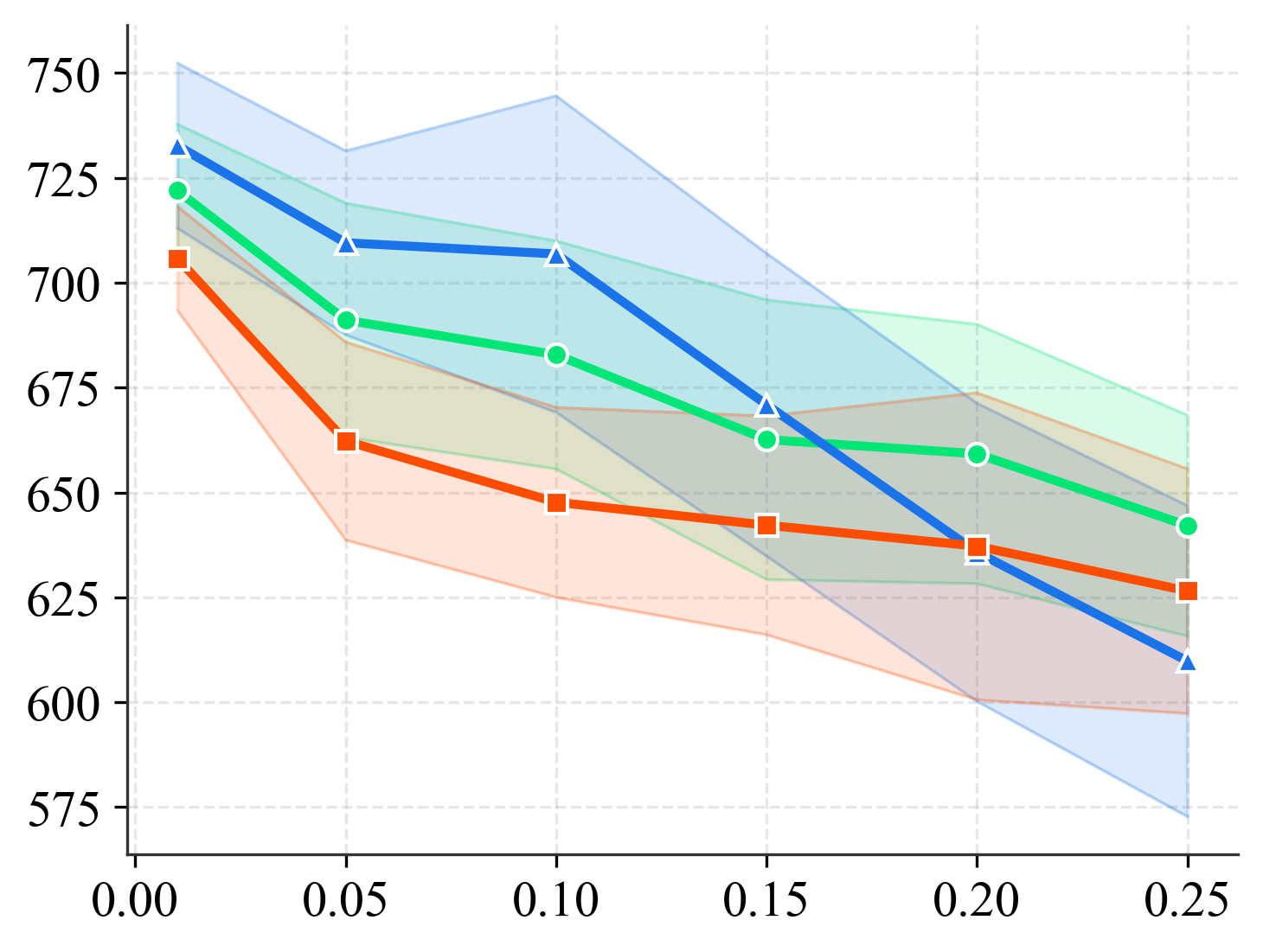} &
\cellcolor{walkgreen}\includegraphics[width=\linewidth]{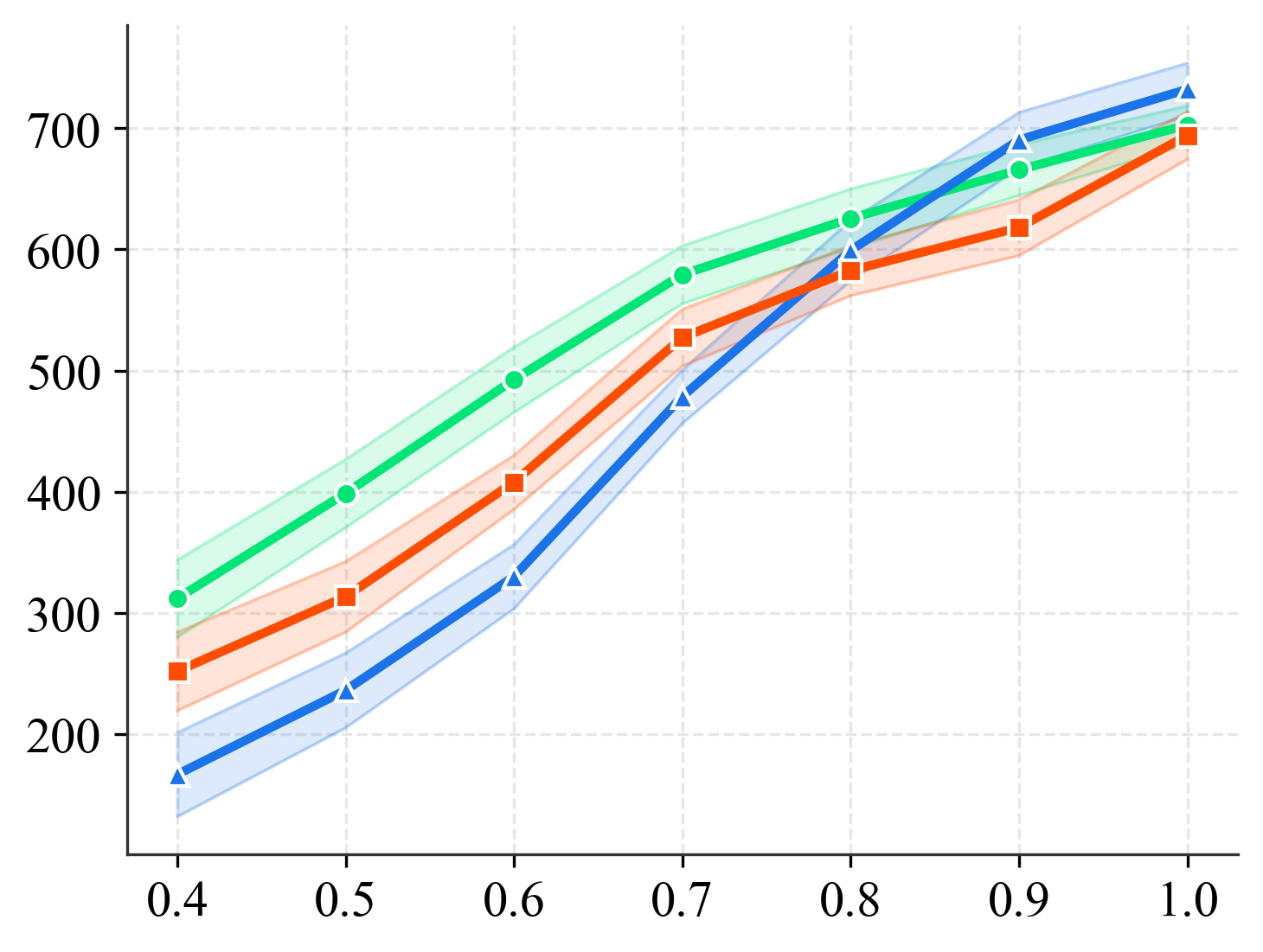} \\[1pt]

\cellcolor{fliporange}\rotatebox{90}{\small\textbf{Jump}} &
\cellcolor{fliporange}\includegraphics[width=\linewidth]{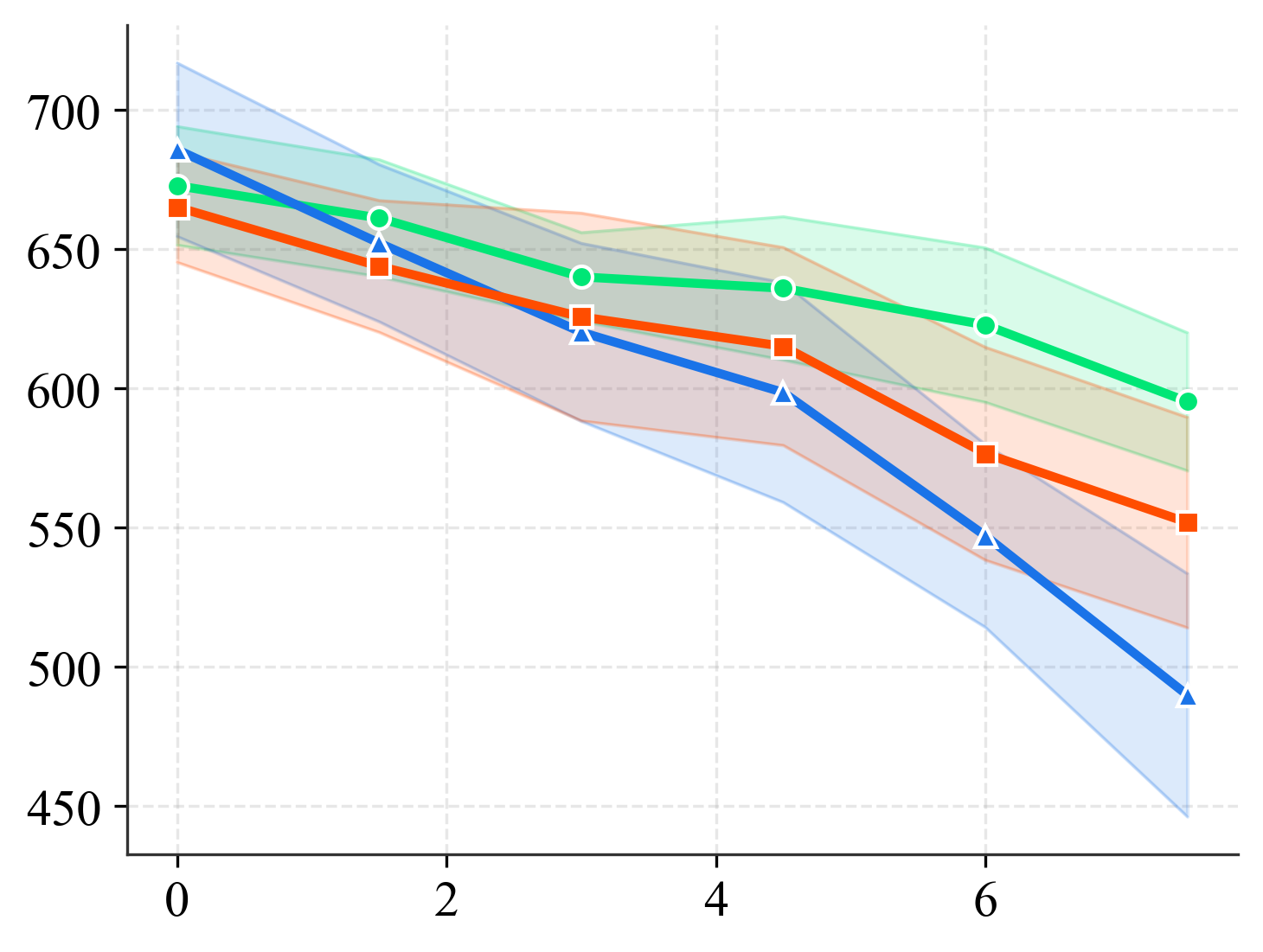} &
\cellcolor{fliporange}\includegraphics[width=\linewidth]{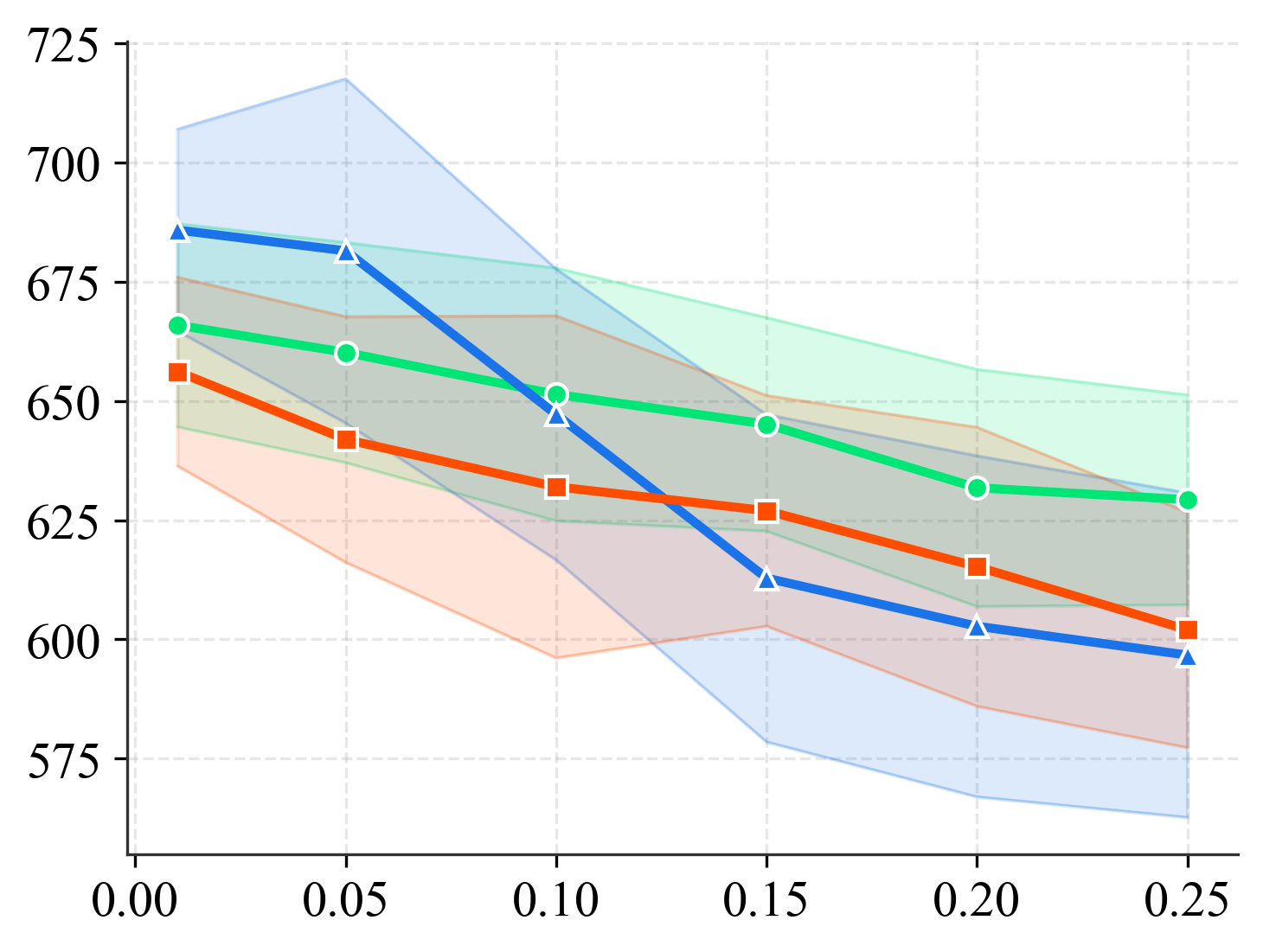} &
\cellcolor{fliporange}\includegraphics[width=\linewidth]{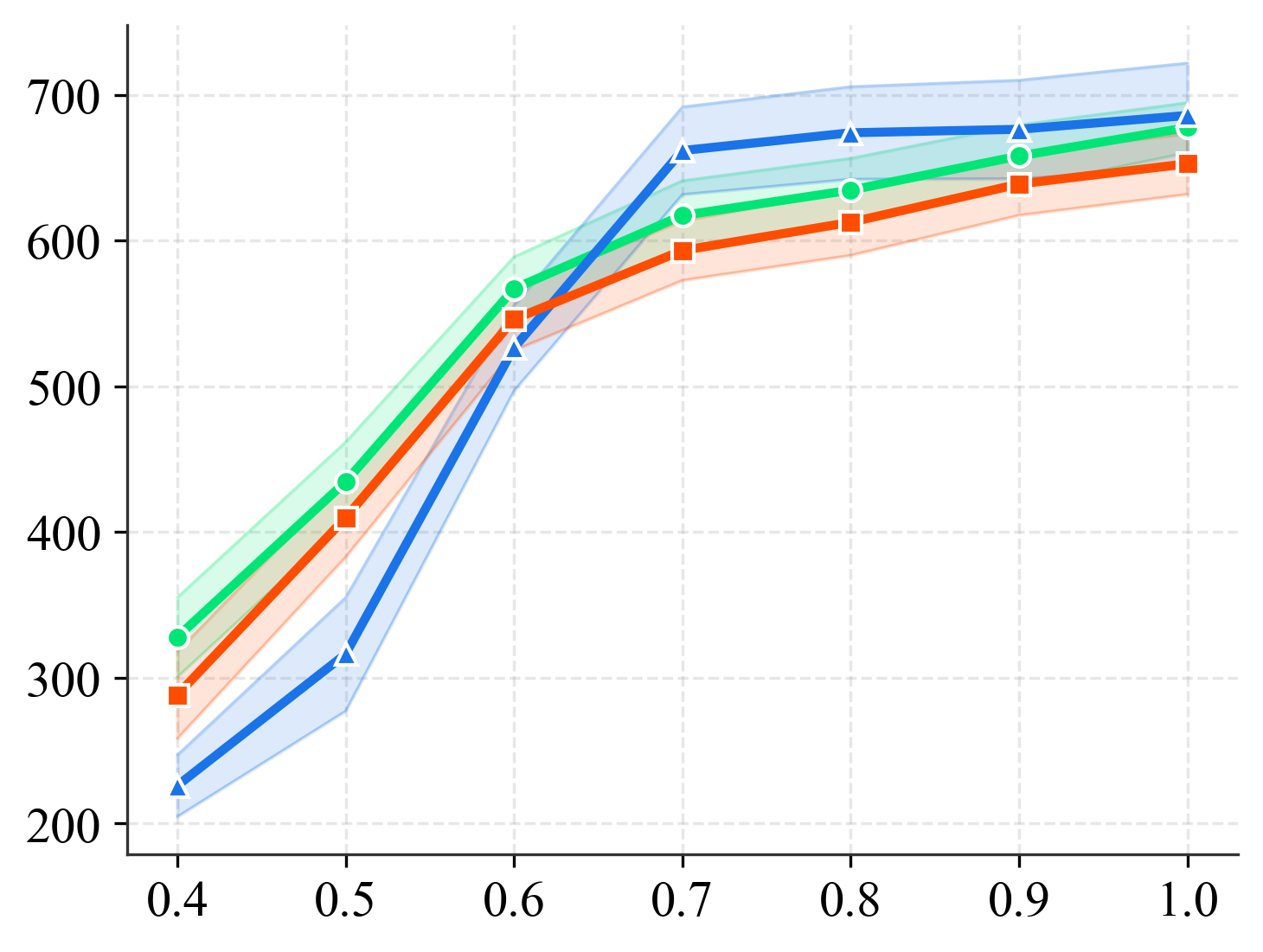} \\[1pt]

\cellcolor{standpurple}\rotatebox{90}{\small\textbf{Stand}} &
\cellcolor{standpurple}\includegraphics[width=\linewidth]{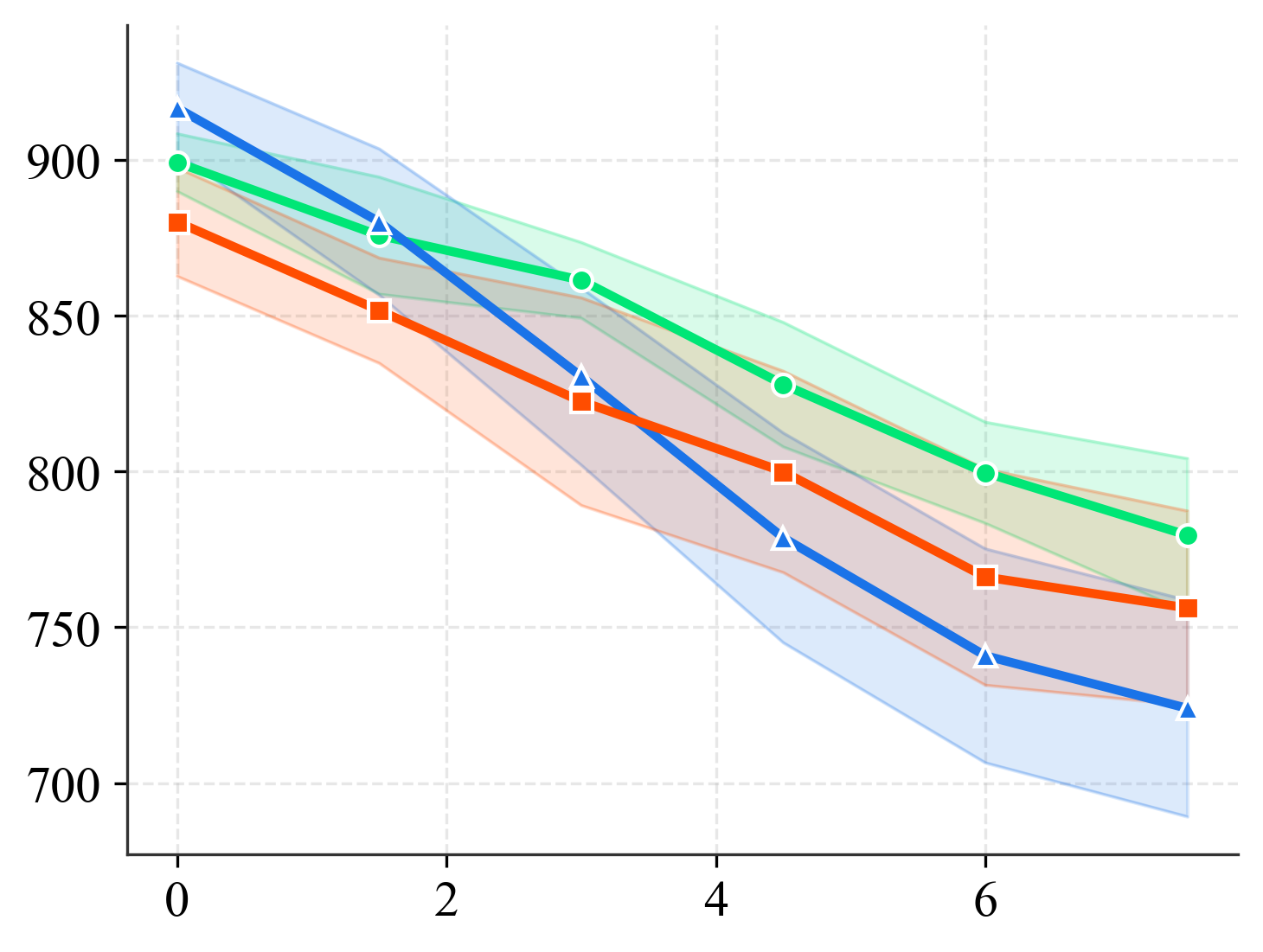} &
\cellcolor{standpurple}\includegraphics[width=\linewidth]{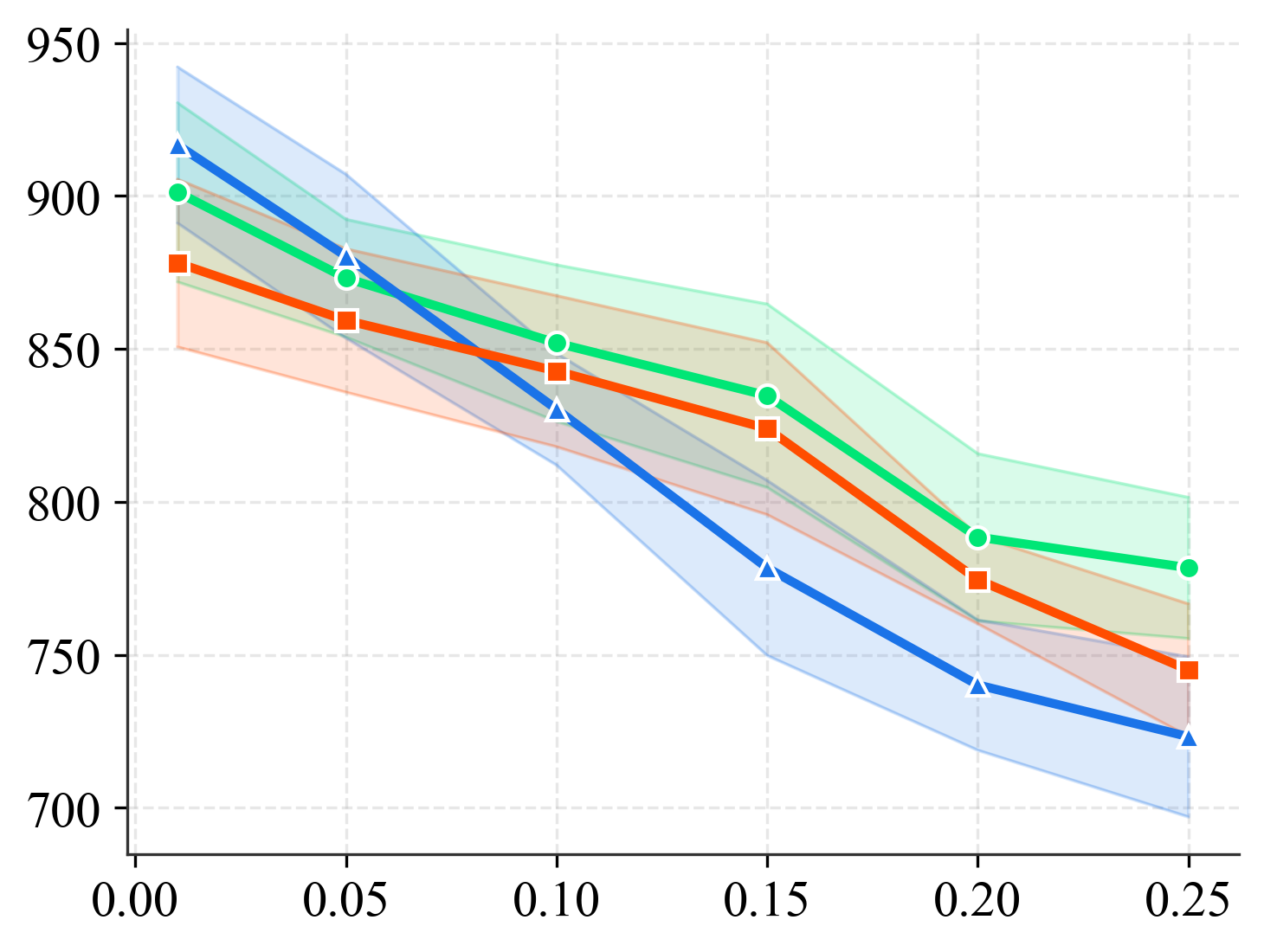} &
\cellcolor{standpurple}\includegraphics[width=\linewidth]{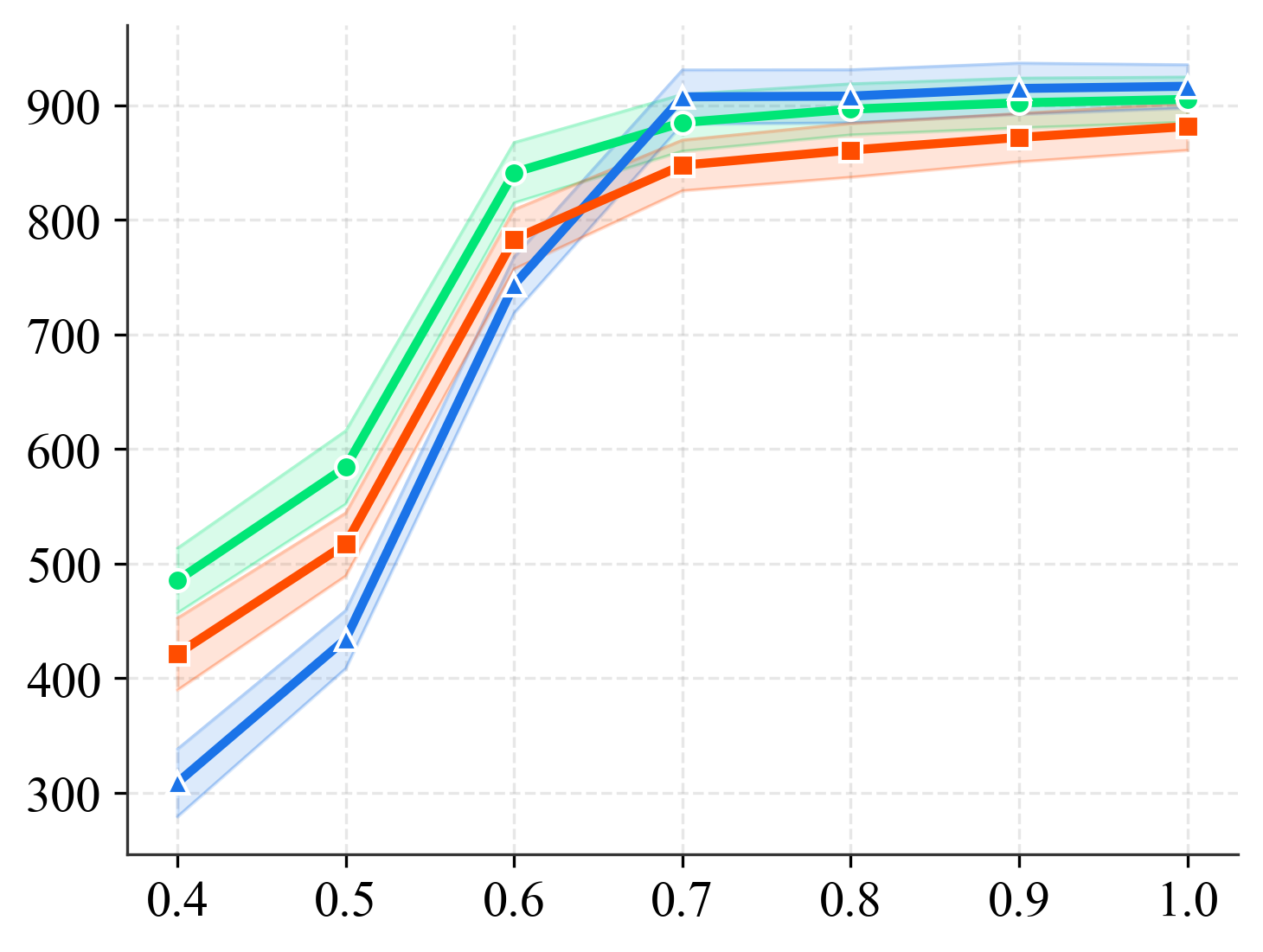} \\
\end{tabular}

\vspace{0.2cm}
\begin{minipage}{0.95\textwidth}
\centering
\footnotesize
\legendcircle{rbfmheavy}{RBFM-Heavy}\hspace{3.0em}
\legendtriangle{fbil}{FB-IL}\hspace{3.0em}
\legendsquare{rbfmlight}{RBFM-Light}
\end{minipage}

\caption{Quadruped performance under dynamics perturbations (95\% Confidence Interval), pretrained on Quadruped RND dataset. Rows: tasks; columns: perturbation types with absolute values of the perturbed parameter. 
Nominal conditions correspond to the leftmost point for the first two perturbations and the rightmost for actuator control range. Average return on y-axis.}
\label{fig:quadruped}
\end{figure}

\section{Experiments} \label{sec:experiments}

We evaluate RBFM-Light (\S\ref{subsec:rbfm_lite}) and RBFM-Heavy (\S\ref{subsec:rbfm_heavy}) on robustness to dynamics perturbations, seeking answers to five questions: \textbf{(Q1)} How do RBFM-Light and RBFM-Heavy compare with FB-IL under dynamics perturbations? \textbf{(Q2)} How do they compare with single-task robust offline IL algorithms? \textbf{(Q3)} Are RBFM-Light and RBFM-Heavy computationally efficient? \textbf{(Q4)} How does the quality and quantity of pretraining data affect robustness? \textbf{(Q5)} How does the robustness parameter $\varepsilon$ affect performance?

\noindent\textbf{Setup.} We evaluate on the ExORL benchmark~\citep{yarats2022don}, which provides datasets collected via unsupervised exploration on the DeepMind Control Suite~\citep{tassa2018deepmind}. Following prior BFM works~\citep{pirotta2024fast, jeen2024zero}, we use pretraining datasets from RND~\citep{burda2018exploration}, APS~\citep{liu2021aps}, and PROTO~\citep{yarats2021reinforcement}. Following~\citep{rupf2025optimistic}, experiments span three domains with twelve tasks: \textit{Walker} (Stand, Walk, Run, Flip), \textit{Quadruped} (Stand, Walk, Run, Jump), and \textit{Cheetah} (Walk, Run, Walk\_Backward, Run\_Backward). For each domain, we introduce physics perturbations: gravity, body mass, and joint friction loss for Walker; tilted gravity, ground contact stiffness, and actuator control range for Quadruped; and actuator strength, joint friction loss, and range of motion for Cheetah. For Walker, we evaluate all combinations of three pretraining datasets, yielding $36$ tasks. For Quadruped and Cheetah, we use RND-pretrained data only, for a total of $60$ tasks across all domains. A pretrained BFM generates $200$ expert trajectories per task; imitation learning uses $4$ randomly sampled trajectories per task. Performance is measured as average return over $100$ evaluation episodes, averaged across $5$ random seeds per environment--task--perturbation combination. We compare against FB-IL for (Q1), and additionally against DRBC~\citep{panaganti2023distributionally} and BE-DROIL~\citep{agrawal2025balance} for (Q2). Full experimental setup details are provided in Appendix~\ref{sec:appendix_exp_setup}.

\noindent\textbf{Implementation Details.} For all BFM-based methods, we follow~\citet{pirotta2024fast} and warm-start $z$ via $z_w = \mathbb{E}_{\tau}\!\left[\frac{\sum_{t\geq0}B(s_{t+1})}{l(\tau)}\right]$. The term $(1-\gamma)\mathbb{E}_{s \sim \mu,\, a \sim \pi_D(\cdot|s)}[Q(s,a)]$ in~\eqref{eqn:final_optimization_problem} requires sampling from the initial-state distribution $\mu$. Since each expert dataset contains only $4$ initial states, we follow~\citet{agrawal2025balance} and treat all trajectory states as effective initial states to improve coverage and stabilize estimation. Complete implementation details are provided in Appendix~\ref{sec:appendix_implementation_details}.

\begin{figure}[t]
\centering
\begin{subfigure}{0.24\textwidth}
    \centering
    \includegraphics[width=\linewidth]{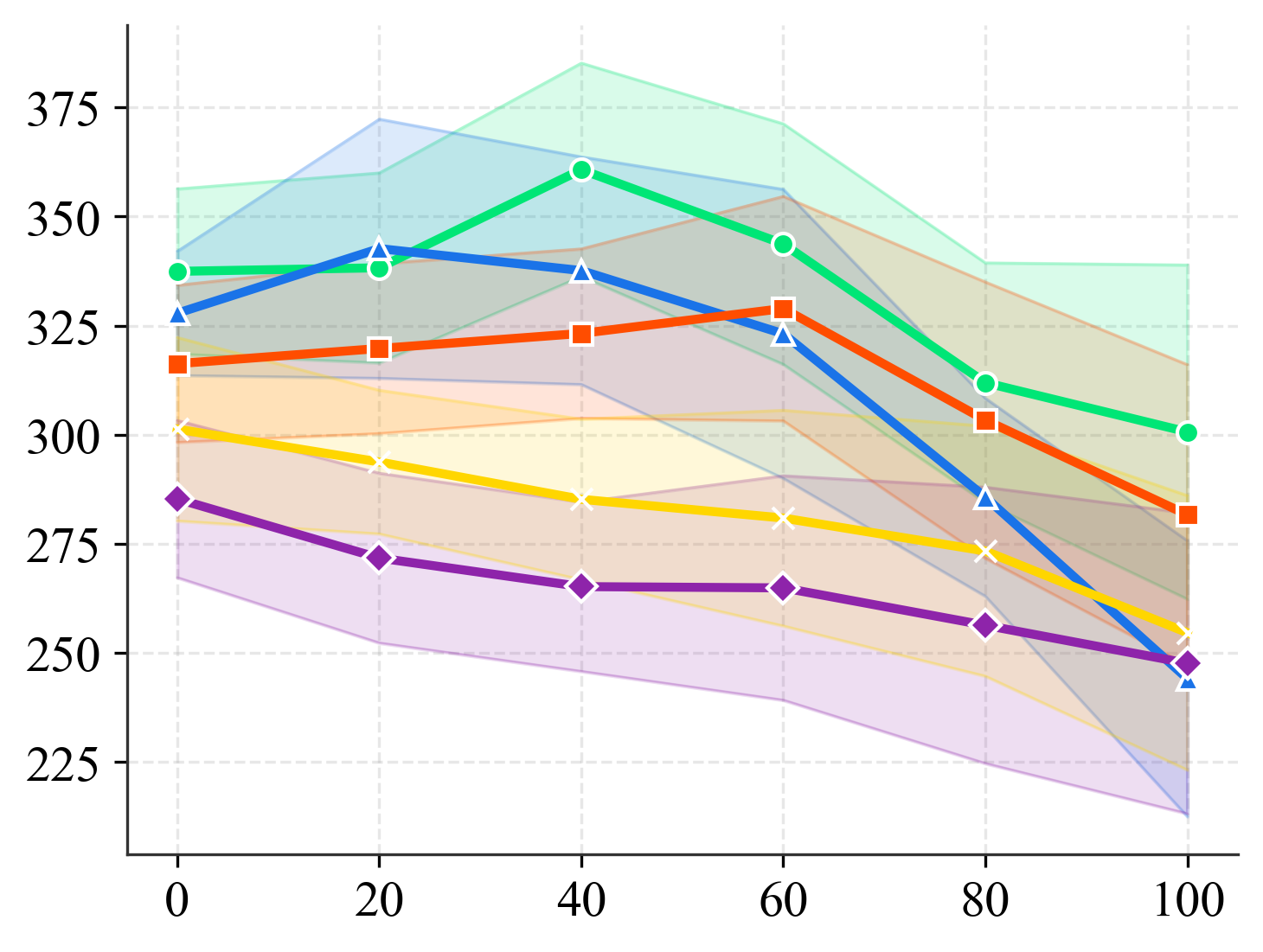}
    \caption{Run}
\end{subfigure}
\hfill
\begin{subfigure}{0.24\textwidth}
    \centering
    \includegraphics[width=\linewidth]{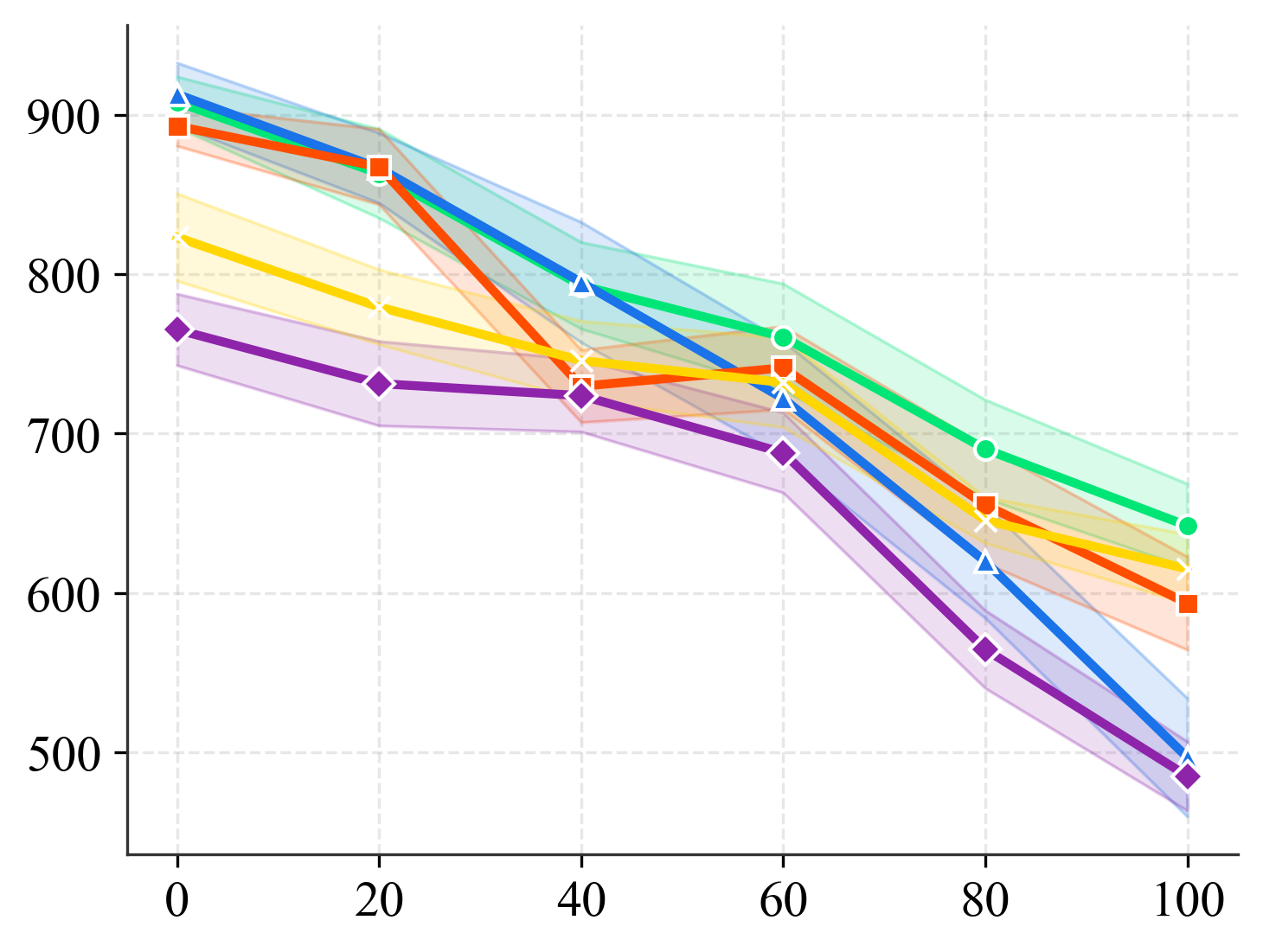}
    \caption{Walk}
\end{subfigure}
\hfill
\begin{subfigure}{0.24\textwidth}
    \centering
    \includegraphics[width=\linewidth]{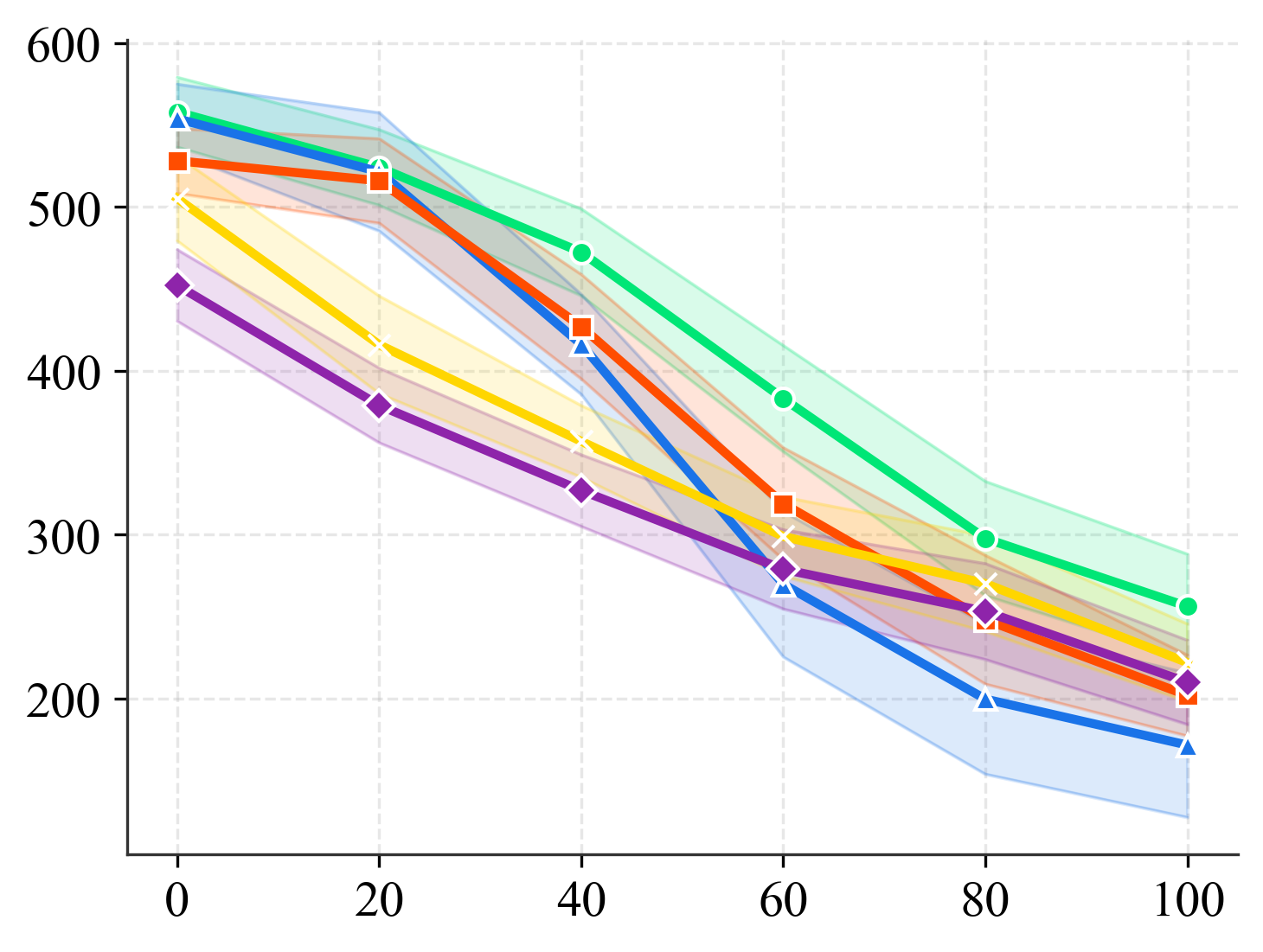}
    \caption{Flip}
\end{subfigure}
\hfill
\begin{subfigure}{0.24\textwidth}
    \centering
    \includegraphics[width=\linewidth]{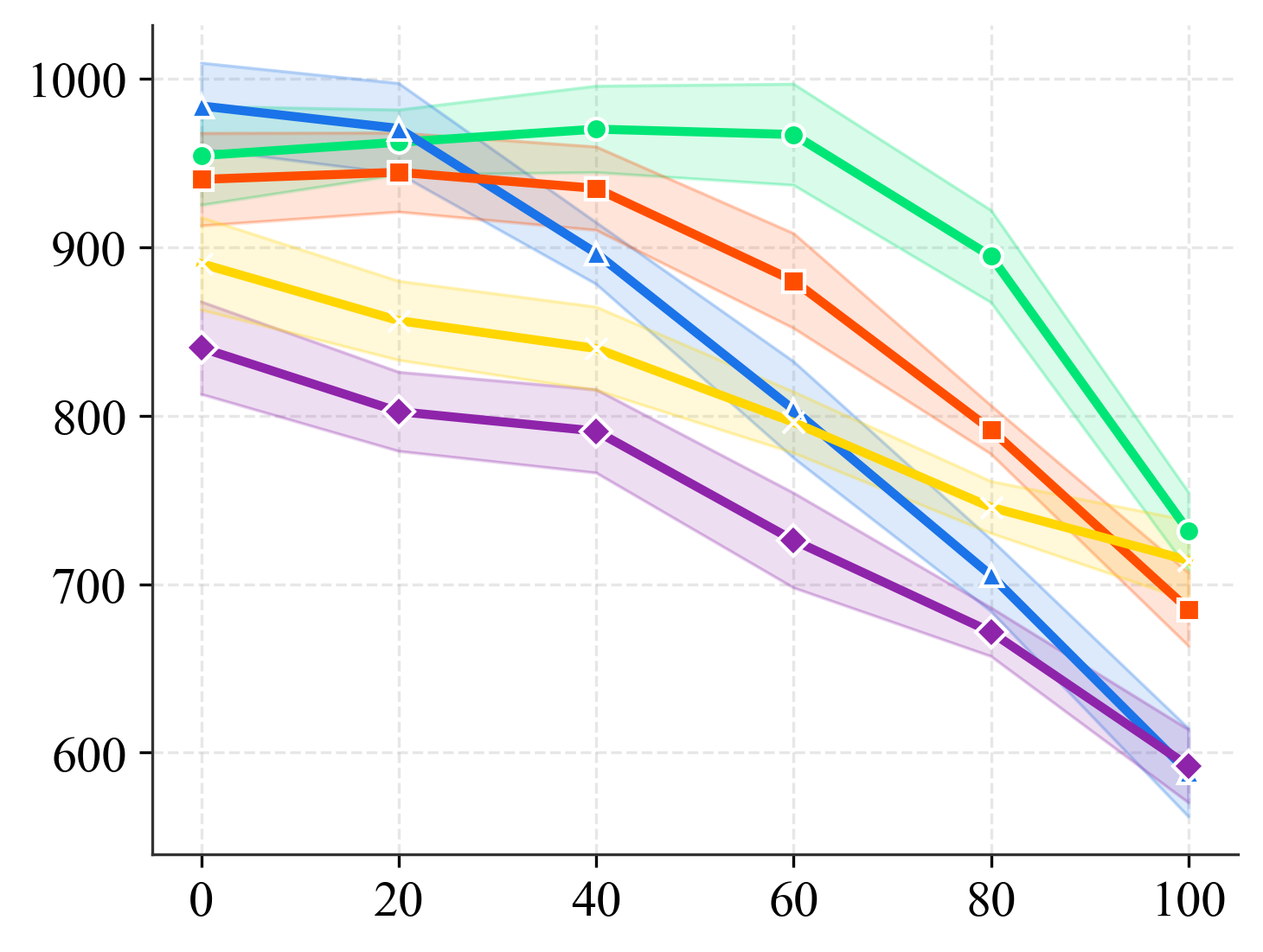}
    \caption{Stand}
\end{subfigure}
\vspace{0.2cm}
\begin{minipage}{0.95\textwidth}
\centering
\footnotesize
\legendcircle{rbfmheavy}{RBFM-Heavy}\hspace{1.2em}
\legendtriangle{fbil}{FB-IL}\hspace{1.2em}
\legendsquare{rbfmlight}{RBFM-Light}\hspace{1.2em}
\legenddiamond{drbc}{DRBC}\hspace{1.2em}
\legendcross{bedroil}{BE-DROIL}
\end{minipage}
\caption{Average return (y-axis) vs.\ body mass perturbation (x-axis, \% increase from nominal) on the Walker environment across four tasks, pretrained on RND dataset. 
}
\label{fig:walker_comparison_with_baselines}
\end{figure}

\noindent\textbf{Results.} \textbf{(Q1)} Across all three morphologies and perturbation types, both RBFM variants consistently outperform FB-IL as perturbation severity increases, with RBFM-Heavy achieving the highest returns, RBFM-Light intermediate, and FB-IL exhibiting the steepest degradation across the majority of settings (Figure~\ref{fig:quadruped} for Quadruped; Walker and Cheetah in Figures~\ref{fig:walker} and~\ref{fig:cheetah}, Appendix~\ref{subsec:app_q1_discussion}). 
\begin{wraptable}{r}{0.26\textwidth}
\centering
\renewcommand{\arraystretch}{1.1}
\begin{tabular}{@{}l c@{}}
\toprule
\textbf{Algorithm} & \textbf{Time} \\
\midrule
DRBC       & $5\text{h}\,23\text{m}$ \\
BE-DROIL   & $6\text{h}\,10\text{m}$ \\
FB-IL      & $1\text{m}$ \\
RBFM-Light & $5\text{m}$ \\
RBFM-Heavy & $10\text{m}$ \\
\bottomrule
\end{tabular}
\vspace{0.3em}
\caption{Per-task inference time averaged across environments and tasks.}
\label{tab:il_time}
\vspace{-0.35cm}
\end{wraptable}
The robustness gap widens monotonically with perturbation magnitude, consistent with our robust objectives: by optimizing $z$ against a worst-case distribution within an uncertainty set around the empirical expert distribution, RBFM explicitly accounts for transition distribution shift at evaluation time. RBFM-Heavy outperforms RBFM-Light because its realizability constraint yields a less conservative solution, whereas RBFM-Light's unconstrained objective over-penalizes transitions unlikely under the true perturbation distribution; FB-IL, optimizing solely under the empirical expert distribution, carries no margin for model mismatch and degrades most severely. These findings hold across perturbation types spanning distinct physical failure modes, motor degradation, joint resistance, kinematic restriction, gravitational disturbance, and terrain compliance, demonstrating that RBFM's robustness benefits generalize across morphologies and perturbation types (see Appendix~\ref{subsec:app_q1_discussion} for per-environment details).

\noindent\textbf{(Q2)} Figure~\ref{fig:walker_comparison_with_baselines} compares RBFM variants and FB-IL against single-task distributionally robust offline IL baselines DRBC and BE-DROIL on the Walker environment under body mass perturbation. At nominal and moderate perturbation levels, all BFM-based methods substantially outperform the single-task baselines, an advantage attributable to the multi-task pretraining of the FB models, which induces a task-agnostic representation that generalizes across reward structures, consistent with the nominal-setting findings of~\citep{pirotta2024fast}. At higher perturbation levels, BE-DROIL matches or marginally outperform FB-IL, demonstrating that per-task distributional robustness does confer an advantage over non-robust BFM inference when dynamics shifts are non-trivial. Crucially, the proposed RBFM framework that combines multi-task BFM pretraining with robust task inference, maintains consistent superiority over both FB-IL and the single-task baselines throughout the full perturbation sweep.


\begin{figure}[t]
    \centering
    \begin{subfigure}[t]{0.325\textwidth}
        \centering
        \includegraphics[width=\linewidth]{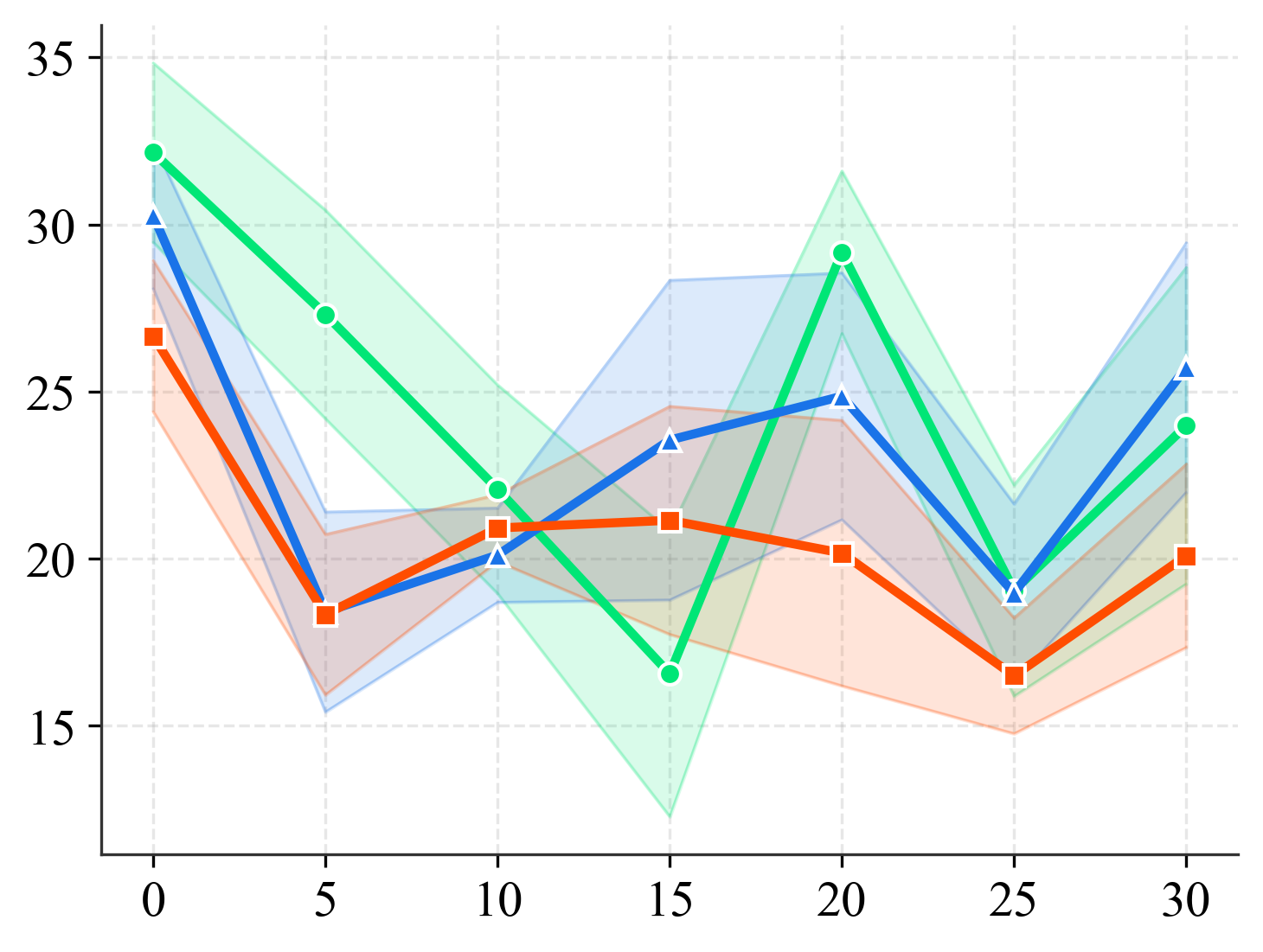}
        \caption{100k}
    \end{subfigure}\hfill
    \begin{subfigure}[t]{0.325\textwidth}
        \centering
        \includegraphics[width=\linewidth]{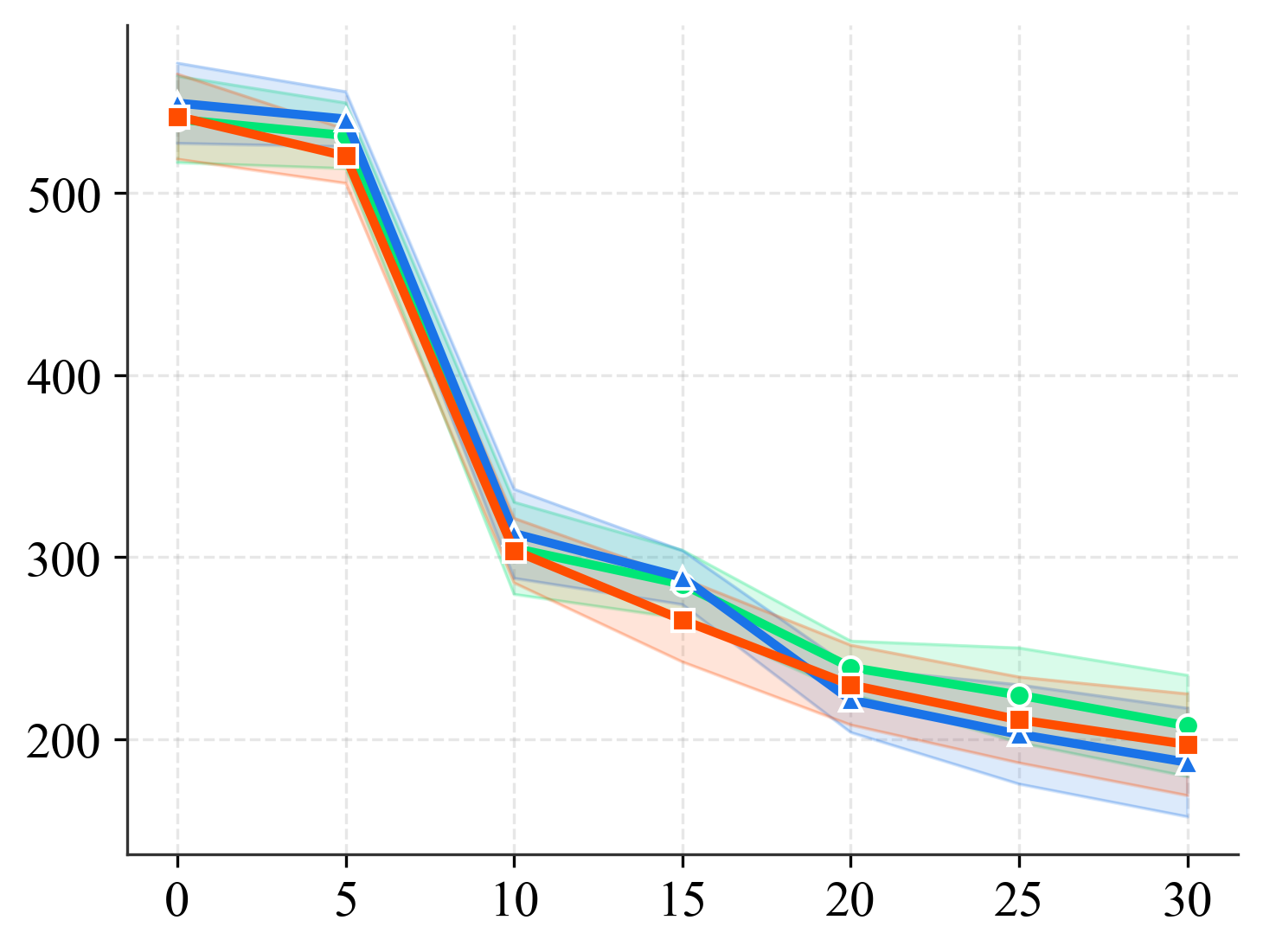}
        \caption{500k}
    \end{subfigure}\hfill
    \begin{subfigure}[t]{0.325\textwidth}
        \centering
        \includegraphics[width=\linewidth]{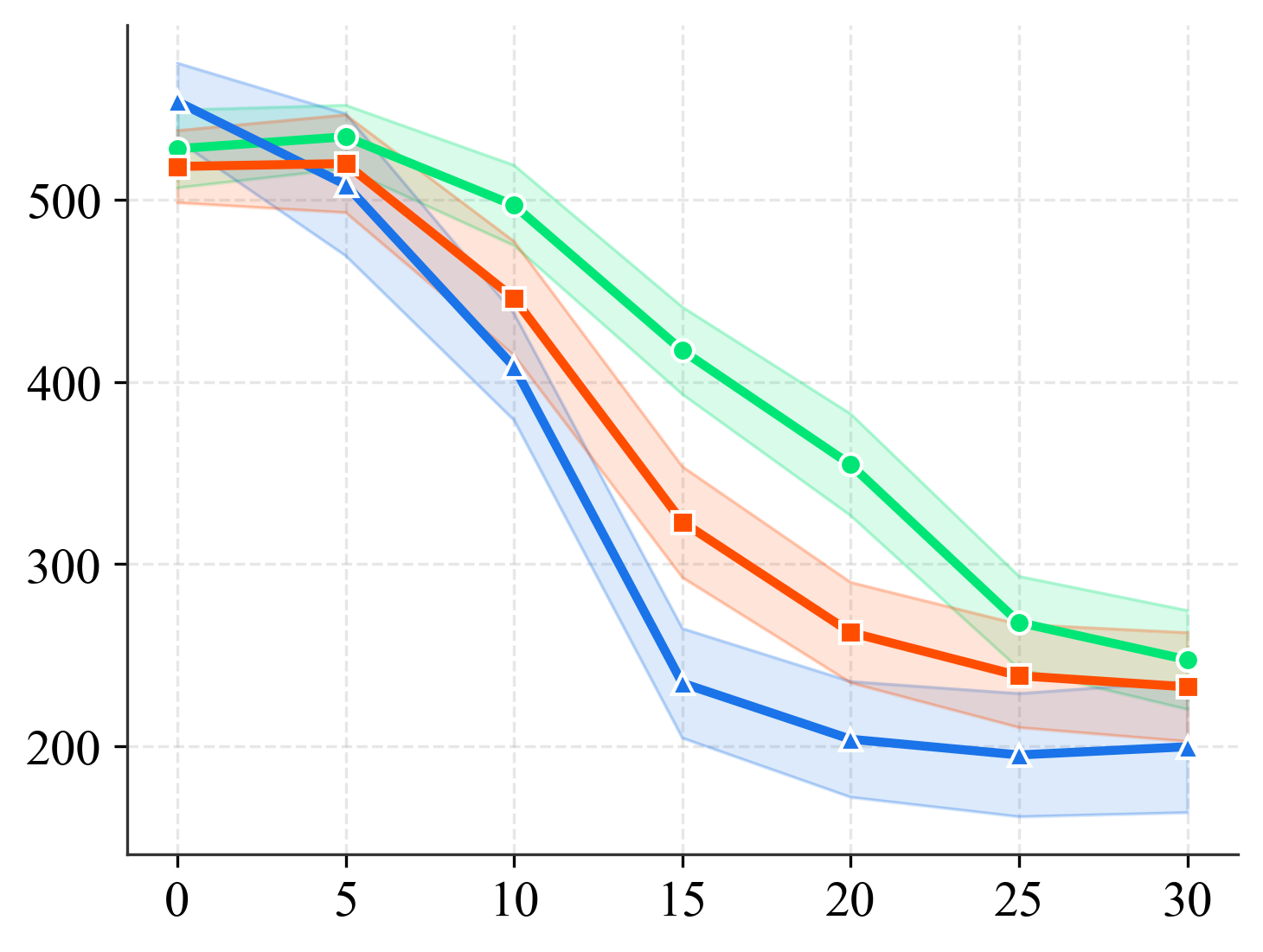}
        \caption{1M}
    \end{subfigure}

    \vspace{0.2cm}

    \begin{minipage}{0.95\textwidth}
        \centering
        \footnotesize
        \legendcircle{rbfmheavy}{RBFM-Heavy}\hspace{1.2em}
        \legendtriangle{fbil}{FB-IL}\hspace{1.2em}
        \legendsquare{rbfmlight}{RBFM-Light}
    \end{minipage}

    \caption{Average return (y-axis) vs.\ joint friction loss perturbation (x-axis, absolute Nm per joint) on Walker flip task across three RND pretraining dataset sizes ($100$k, $500$k, $1$M transitions).
    }
\label{fig:walker_comparison_with_varying_pretrain_datasize}
\end{figure}

\noindent\textbf{(Q3)} Table~\ref{tab:il_time} reports wall-clock time to compute an imitation policy for a single task, averaged across environments and tasks. Single-task robust baselines (DRBC: $5\text{h}\,23\text{m}$; BE-DROIL: $6\text{h}\,10\text{m}$) incur substantial per-task optimization cost, whereas BFM-based methods amortize representation learning over a one-time pretraining phase, reducing task-specific inference to minutes: FB-IL ($1\text{m}$), RBFM-Light ($5\text{m}$), RBFM-Heavy ($10\text{m}$). Since pretraining is a fixed cost shared across all downstream tasks, per-task inference time is the operationally relevant deployment metric and RBFM variants are orders of magnitude faster than single-task robust baselines while achieving superior robustness (Q1 \& Q2), demonstrating that RBFM attains distributional robustness without compromising the computational efficiency inherent to the BFM framework.

\begin{wrapfigure}{r}{0.37\columnwidth}
\centering
\includegraphics[width=\linewidth]{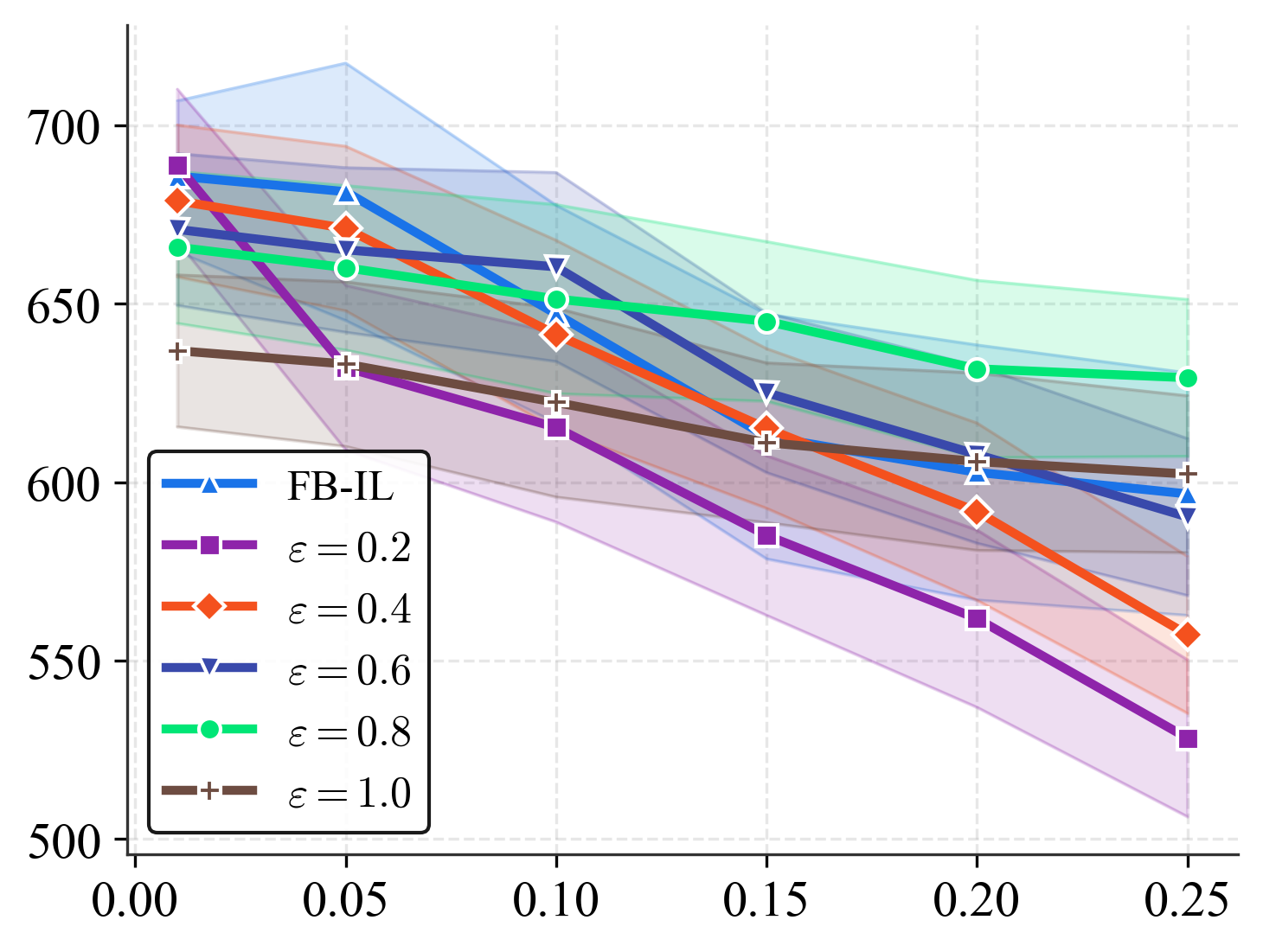}
\caption{Average return (y-axis) vs.\ ground contact stiffness (x-axis, absolute perturbation) on Quadruped jump under varying $\varepsilon$ for RBFM-Heavy.}
\label{fig:epsilon_perturbation}
\vspace{-2em}
\end{wrapfigure}

\noindent\textbf{(Q4)} We investigate two axes of pretraining data quality: source and quantity. For data source, we evaluate on Walker using APS- and PROTO-generated datasets (Appendix~\ref{subsec:app_different_pre_source}, Figures~\ref{fig:walker_aps} and~\ref{fig:walker_proto}); trends are consistent with (Q1), confirming robustness rankings across pretraining sources. For data quantity, we vary pretraining dataset size across 100k, 500k, and 1M transitions on the Walker flip task under joint friction loss (Figure~\ref{fig:walker_comparison_with_varying_pretrain_datasize}). At 100k, performance is inconsistent with no clear method separation; at 500k trends begin to stabilize; by 1M a consistent separation emerges. This progression underscores that robust task inference critically depends on the quality of the learned FB representations, which require sufficient pretraining data to become expressive.

\noindent\textbf{(Q5)} We vary $\varepsilon \in \{0.2, 0.4, 0.6, 0.8, 1.0\}$ and evaluate RBFM-Heavy on the Quadruped jump task under ground contact stiffness perturbations (Figure~\ref{fig:epsilon_perturbation}). Smaller $\varepsilon$ yields higher nominal returns but sharp degradation under perturbation, as the uncertainty set fails to cover the perturbed transition distribution; larger $\varepsilon$ induces more gradual decay at the cost of modest nominal performance loss. The value $\varepsilon=0.8$ achieves the best trade-off and is used in all experiments.

\section{Related Work}
Offline imitation learning is commonly approached via Behavioral Cloning (BC) \citep{firstBC}, which treats imitation as supervised learning but suffers from covariate shift and compounding errors due to ignoring environment dynamics \citep{bcLimitation1}. The DICE family \citep{Kostrikov2020Imitation,kim2022demodice,mao2024odice} incorporates transition structure by estimating occupancy ratios from offline data. In parallel, robust reinforcement learning addresses transition uncertainty via worst-case optimization over ambiguity sets \citep{nilim2005robust,iyengar2005robust,derman2020distributional,yu2023fast,pmlr-v283-panaganti25a}, but typically requires rewards and online interaction. Offline robust RL \citep{panaganti2022robust} removes interaction but still relies on known rewards, limiting applicability to imitation learning.

Robust imitation learning has largely focused on imperfect demonstrations or training stability \citep{pmlr-v97-wu19a,pmlr-v130-tangkaratt21a,pmlr-v78-laskey17a,pmlr-v15-ross11a}, while domain randomization and meta-IL methods assume access to simulators or fixed dynamics \citep{peng2018sim,huang2021generalization,finn2017one}. Recent approaches improve BC robustness under distributional shift \citep{wu2025robust,seo2024mitigating}, but do not address transition uncertainty. Among methods that do, DRBC \citep{panaganti2023distributionally} and BE-DROIL \citep{agrawal2025balance} formulate robust offline imitation learning under transition ambiguity. However, these approaches operate in a per-task manner, requiring separate optimization for each task. This limits their applicability in multi-task settings, where rapid adaptation to new tasks is essential, thereby motivating the need for a unified and robust framework capable of generalizing across tasks. We refer to Appendix \ref{sec:appendix_related_work} for an extended discussion on related works.

\section{Conclusion}

We revisited Behavior Foundation Models (BFMs) under dynamics shift and showed that standard task inference is inherently brittle to transition mismatch. To address this, we proposed a robust minimax formulation over an uncertainty set of dynamics, and introduced two practical algorithms, \textit{RBFM-Light} and \textit{RBFM-Heavy}, trading off efficiency and robustness without modifying pretraining or requiring multi-environment data. Empirically, both variants significantly improve robustness over standard BFMs and outperform single-task robust offline IL baselines, while preserving fast task inference. Our results highlight a promising direction for scalable imitation learning, where a single pretrained model can be robustly adapted to new tasks using only offline data.


\bibliographystyle{plainnat} 
\bibliography{references_neurips_2026}



\newpage
\appendix


\appendixtoc

\newpage
\appsection{Missing Proofs}\label{sec:appendix_proofs}

\begin{customlem}{1}
For any policy $\pi$ and transition kernel $T \in \mathcal{T}(\varepsilon')$, the following holds:
\begin{equation*}
    D_{\mathrm{TV}}(\rho^\pi_T(s), \rho^\pi_{T^o}(s)) \le \frac{\gamma \varepsilon'}{1 - \gamma}.
\end{equation*}
\end{customlem}
\begin{proof}
    See Lemma $7$ in \citep{panaganti2023distributionally} for details.   
\end{proof}

\begin{lemma}
For any policy $\pi$ and transition kernel $T \in \mathcal{T}(\rho')$, the following holds:
\begin{equation*}
    D_{\mathrm{TV}}(\rho^\pi_T(s,a), \rho^\pi_{T^o}(s,a)) \le \frac{\gamma \varepsilon'}{1 - \gamma}.
\end{equation*}
\label{lemma:D_TV_mismatch_s_a}
\end{lemma}

\begin{proof}
We build upon the Lemma \ref{lemma:D_TV_mismatch_s} that quantifies how uncertainty in the transition model influences the induced occupancy measure on the state-space for a fixed policy. 

Since $\rho^\pi_T(s,a) = \pi(a|s) \rho^\pi_T(s)$, it follows that
\begin{align*}
D_{\mathrm{TV}}(\rho^\pi_T(s,a), \rho^\pi_{T^o}(s,a))
&= \frac{1}{2}\sum_{s,a} |\rho^\pi_T(s,a) - \rho^\pi_{T^o}(s,a)| \\
&= \frac{1}{2}\sum_{s,a} \pi(a|s) |\rho^\pi_T(s) - \rho^\pi_{T^o}(s)| \\
&= D_{\mathrm{TV}}(\rho^\pi_T(s), \rho^\pi_{T^o}(s)) \\
&\le \frac{\gamma \rho'}{1 - \gamma}.
\end{align*}
\end{proof}

\begin{customlem}{2}
For any policy $\pi$ and transition kernel $T \in \mathcal{T}(\varepsilon')$, the following holds:
\begin{equation*}
    D_{\mathrm{TV}}(\rho^\pi_T(s,a,s'), \rho^\pi_{T^o}(s,a,s')) \le \frac{\varepsilon'}{1 - \gamma}.
\end{equation*}
\end{customlem}
\begin{proof}
    We have
\[
\rho^\pi_T(s,a,s') = \rho^\pi_T(s,a) T_{s,a}(s'), \quad
\rho^\pi_{T^o}(s,a,s') = \rho^\pi_{T^o}(s,a) T^o_{s,a}(s').
\]
Then,
\begin{align*}
&\| \rho^\pi_T(s,a,s') - \rho^\pi_{T^o}(s,a,s') \|_1 \\
&= \sum_{s,a,s'} \left| \rho^\pi_T(s,a) T_{s,a}(s') - \rho^\pi_{T^o}(s,a) T^o_{s,a}(s') \right| \\
&\le \sum_{s,a,s'} \rho^\pi_T(s,a) \left| T_{s,a}(s') - T^o_{s,a}(s') \right|
   + \sum_{s,a,s'} \left| \rho^\pi_T(s,a) - \rho^\pi_{T^o}(s,a) \right| T^o_{s,a}(s').
\end{align*}

For the first term, using $\|T_{s,a} - T^o_{s,a}\|_1 = 2 D_{\mathrm{TV}}(T_{s,a}, T^o_{s,a}) \le 2\varepsilon'$,
\[
\sum_{s,a,s'} \rho^\pi_T(s,a) \left| T_{s,a}(s') - T^o_{s,a}(s') \right|
= \sum_{s,a} \rho^\pi_T(s,a) \|T_{s,a} - T^o_{s,a}\|_1
\le 2\varepsilon'.
\]

For the second term, since $\sum_{s'} T^o_{s,a}(s') = 1$, we have
\begin{equation*}
\begin{aligned}
\sum_{s,a,s'} \left| \rho^\pi_T(s,a) - \rho^\pi_{T^o}(s,a) \right| T^o_{s,a}(s')
&= \sum_{s,a} \left| \rho^\pi_T(s,a) - \rho^\pi_{T^o}(s,a) \right|\\
&= \| \rho^\pi_T(s,a) - \rho^\pi_{T^o}(s,a) \|_1.
\end{aligned}
\end{equation*}
By lemma \ref{lemma:D_TV_mismatch_s_a},
\[
D_{\mathrm{TV}}(\rho^\pi_T(s,a), \rho^\pi_{T^o}(s,a)) \le \frac{\gamma \varepsilon'}{1 - \gamma},
\]
which implies
\[
\| \rho^\pi_T(s,a) - \rho^\pi_{T^o}(s,a) \|_1 \le \frac{2\gamma\varepsilon'}{1 - \gamma}.
\]

Combining both terms,
\[
\| \rho^\pi_T(s,a,s') - \rho^\pi_{T^o}(s,a,s') \|_1
\le 2\varepsilon' + \frac{2\gamma\varepsilon'}{1 - \gamma}
= \frac{2\varepsilon'}{1 - \gamma}.
\]

Hence,
\[
D_{\mathrm{TV}}\big(\rho^\pi_T(s,a,s'), \rho^\pi_{T^o}(s,a,s')\big)
= \frac{1}{2} \| \rho^\pi_T(s,a,s') - \rho^\pi_{T^o}(s,a,s') \|_1
\le \frac{\varepsilon'}{1 - \gamma}.
\]
\end{proof}

\begin{lemma}
Let $D_f$ be the f-divergence with $f(t) = \lvert t - 1 \rvert / 2$ corresponding to the total variation (TV) uncertainty set. Then,
\begin{equation}
\min_{D_f(P \,\|\, P_o) \le \rho} \; \mathbb{E}_{P}[l(X)]
=
- \min_{\eta \in \mathbb{R}} \;
\mathbb{E}_{P_o}\big[(\eta - l(X))_+\big]
+ \rho \, (\eta - \min_{x \in \mathcal{X}} l(x))_+ - \eta.
\end{equation}
\label{lem:f_div}
\end{lemma}
\begin{proof}
    See Lemma $5$ in \citep{panaganti2022robust} for details.
\end{proof}

\begin{customprop}{1}
The solution to the constrained optimization problem in \eqref{eqn:rbfm_light} is obtained by solving the following:
\begin{equation*}
\min_{z}\min_{\lambda \in \mathbb{R}}
\left\{\mathbb{E}_{s \sim \rho^{\pi_D}_{T^{o}}}\big[(\mathcal{L}(\pi_z(\cdot|s), \pi_D(\cdot|s)) - \lambda)_+\big]
+ \varepsilon_l\left(\max_{\rho^{\pi_D}_{T^{o}}(s) > 0}
\mathcal{L}(\pi_z(\cdot|s), \pi_D(\cdot|s)) - \lambda\right)_{+} + \lambda\right\}.
\end{equation*}
\end{customprop}
\begin{proof}
    Let us fix the learner's task vector $z$ and therefore its policy $\pi_z$ and define the point-wise loss
\[
\ell_z(s)
:= \mathcal{L}(\pi_z(\cdot|s),\pi_D(\cdot|s))
\]
Now, let $M_z := \max_{s \in \mathcal S:\, \rho^{\pi_D}_{T^o}(s)>0} \ell_z(s)$. Applying the total-variation duality lemma \ref{lem:f_div} to $\ell_z$ with nominal
distribution $\rho^{\pi_D}_{T^o}$ and radius $\varepsilon$ gives
\[
\min_{Q:\,D_{\mathrm{TV}}(Q,\rho^{\pi_D}_{T^o})\le \varepsilon}
\mathbb E_Q[\ell_z(S)]
=
\min_{\lambda\in\mathbb R}
\left\{
\mathbb E_{s\sim \rho^{\pi_D}_{T^o}}\big[(\ell_z(s)-\lambda)_+\big]
+\varepsilon_l (M_z-\lambda)_+
+\lambda
\right\}.
\]
Substituting back the definition of $\ell_z$ and the minimization over task vector $z$ gives
\[
\min_{z}\min_{\lambda\in\mathbb R}
\left\{
\mathbb E_{s\sim \rho^{\pi_D}_{T^o}}\big[(\mathcal{L}(\pi_z(\cdot|s),\pi_D(\cdot|s))-\lambda)_+\big]
+\varepsilon_l\left(\max_{s:\,\rho^{\pi_D}_{T^o}(s)>0}\mathcal{L}(\pi_z(\cdot|s),\pi_D(\cdot|s))-\lambda\right)_+
+\lambda
\right\}.
\]
\end{proof}

\begin{customprop}{2}
Suppose that the expert and learner policies $\pi_D$ and $\pi_z$ are deterministic and that the action space $\mathcal A$ is bounded. Let $\mathcal{L}(\pi_z(\cdot|s), \pi_D(\cdot|s))=\|\pi_z(s)-\pi_D(s)\|_2^2$ and $
L:=\sup_{(s,a)\sim\mathcal \rho^{\pi_D}_{T^{o}}}\|\pi_z(s)-\pi_D(s)\|_2^2$.
Then the optimization problem in Proposition \ref{prop:light_sol} can be simplified to
\begin{equation*}
\min_{z}\min_{\lambda\in[0,(1+\varepsilon_l)L]}
\left\{\mathbb{E}_{s \sim \rho^{\pi_D}_{T^{o}}}\big[(\mathcal{L}(\pi_z(\cdot|s), \pi_D(\cdot|s)) - \lambda)_+\big]
+ \varepsilon_l\left(\max_{\rho^{\pi_D}_{T^{o}}(s) > 0}
\mathcal{L}(\pi_z(\cdot|s), \pi_D(\cdot|s)) - \lambda\right)_{+} + \lambda\right\}.
\end{equation*}
\end{customprop}

\begin{proof}
Let us fix the learner's task vector $z$ and the corresponding policy $\pi_z$ and define the point-wise loss
\[
\ell_z(s)
:= \mathcal{L}(\pi_z(\cdot|s),\pi_D(\cdot|s))
= \|\pi_z(s)-\pi_D(s)\|_2^2.
\]
Since the expert and learner policies are deterministic and the action space
$\mathcal A$ is bounded, the loss is bounded on the support of
$\rho^{\pi_D}_{T^o}$. In particular,
\[
0 \le \ell_z(s) \le L
\qquad \text{for all } s \text{ with } \rho^{\pi_D}_{T^o}(s)>0.
\]

Now, Proposition \ref{prop:light_sol} provides the following optimization problem:
\[
\min_{z}\min_{\lambda\in\mathbb R}
\left\{
\mathbb E_{s\sim \rho^{\pi_D}_{T^o}}\big[(\mathcal{L}(\pi_z(\cdot|s),\pi_D(\cdot|s))-\lambda)_+\big]
+\varepsilon_l\left(\max_{s:\,\rho^{\pi_D}_{T^o}(s)>0}\mathcal{L}(\pi_z(\cdot|s),\pi_D(\cdot|s))-\lambda\right)_+
+\lambda
\right\}.
\]

Next, we define $M_z := \max_{s \in \mathcal S:\, \rho^{\pi_D}_{T^o}(s)>0} \ell_z(s)$ and $g(\lambda)$ as
\[
g(\lambda):=
\mathbb E_{s\sim \rho^{\pi_D}_{T^o}}\big[(\ell_z(s)-\lambda)_+\big]
+\varepsilon_l(M_z-\lambda)_+
+\lambda.
\]

If $\lambda<0$, then $\ell_z(s)-\lambda>0$ and $M_z-\lambda>0$, so
\[
g(\lambda)
=
\mathbb E_{s\sim \rho^{\pi_D}_{T^o}}[\ell_z(s)-\lambda]
+\varepsilon_l(M_z-\lambda)
+\lambda,
\]
which is strictly decreasing in $\lambda$ because $\varepsilon_l>0$.
Hence no minimizer lies below $0$. If $\lambda\ge M_z$, then $\ell_z(s)-\lambda\le 0$ for all $s$ in the support of
$\rho^{\pi_D}_{T^o}$, and also $M_z-\lambda\le 0$, so
\[
g(\lambda)=\lambda,
\]
which is strictly increasing in $\lambda$. Hence no minimizer lies above
$M_z$. Therefore every minimizer belongs to $[0,M_z]$. Since $M_z\le L$, we may restrict the minimization further to the larger interval
$[0,(1+\varepsilon_l)L]$ without changing the value. Thus, we have
\[
\min_{z}\min_{\lambda\in[0,(1+\varepsilon_l)L]}
\left\{
\lambda +
\mathbb{E}_{s \sim \rho^{\pi_D}_{T^{o}}}\big[(\mathcal{L}(\pi_z(\cdot|s), \pi_D(\cdot|s)) - \lambda)_+\big]
+ \varepsilon_l\left(\max_{s \in \mathcal{S}:\rho^{\pi_D}_{T^{o}}(s) > 0}
\mathcal{L}(\pi_z(\cdot|s), \pi_D(\cdot|s)) - \lambda\right)_{+}
\right\}
\]
\end{proof}

\begin{customprop}{3}
For $\tau > 0$, the inner maximization in~\eqref{eqn:final_lagrangian_constrained_bc} admits the solution
\[
w^{\star}_{Q,\tau,\pi_z}(s,a,s') =
\max\!\left(0,\,(f')^{-1}\!\left(\frac{c_{Q,\pi_z}(s,a,s')}{\tau}\right)\right).
\]
For $\tau = 0$, $w^{\star}_{Q,\tau,\pi_z}(s,a,s') = +\infty$ if $c_{Q,\pi_z}(s,a,s')>0$ and $0$ otherwise.
\end{customprop}
\begin{proof}
We first consider the case $\tau > 0$. For fixed $(Q,\tau)$, the optimization over $w \ge 0$ in~\eqref{eqn:final_lagrangian_constrained_bc} reduces to
\[
\max_{w \ge 0} \;
\E_{s,a,s'\sim \rho^{\pi_D}_{T^o}}
\big[-\tau f(w(s,a,s')) + w(s,a,s')\,c_{Q,\pi}(s,a,s')\big].
\]
Since the expectation is separable across $(s,a,s')$, the maximization decomposes pointwise. For each $(s,a,s')$, we solve
\[
\max_{w \ge 0} \;
-\tau f(w) + c_{Q,\pi}(s,a,s')\, w.
\]

Because $f$ is strictly convex, the objective is strictly concave in $w$, and the problem admits a unique maximizer. Ignoring the constraint $w \ge 0$, the first-order optimality condition gives
\[
-\tau f'(w) + c_{Q,\pi}(s,a,s') = 0
\quad \Rightarrow \quad
w = (f')^{-1}\!\left(\frac{c_{Q,\pi}(s,a,s')}{\tau}\right).
\]

Enforcing the constraint $w \ge 0$ amounts to projecting this unconstrained solution onto $\mathbb{R}_+$, yielding
\[
w^{\star}_{Q,\tau,\pi}(s,a,s')
= \max\!\left(0,\,(f')^{-1}\!\left(\frac{c_{Q,\pi}(s,a,s')}{\tau}\right)\right).
\]

For tuples $(s,a,s')$ with $\rho^{\pi_D}_{T^o}(s,a,s') = 0$, the corresponding terms do not affect the objective, and $w$ can be defined arbitrarily.

We now consider the case $\tau = 0$. The objective reduces to
\[
\E_{s,a,s'\sim \rho^{\pi_D}_{T^o}}\!\big[w(s,a,s')\,c_{Q,\pi}(s,a,s')\big],
\]
which is linear in $w$. Maximizing over $w \ge 0$ yields
\[
w^{\star}_{Q,\tau,\pi}(s,a,s') =
\begin{cases}
+\infty, & \text{if } c_{Q,\pi}(s,a,s') > 0,\\
0, & \text{otherwise}.
\end{cases}
\]
\end{proof}


\appsection{Extended Related Work}\label{sec:appendix_related_work}

\noindent\textbf{Offline Imitation Learning.} Behavioral Cloning (BC)~\citep{firstBC} reduces imitation to supervised regression from expert state-action pairs, but its independence from environment dynamics makes it susceptible to compounding errors under covariate shift~\citep{bcLimitation1}. The DICE family of methods~\citep{Kostrikov2020Imitation, kim2022demodice, mao2024odice} addresses this by estimating stationary occupancy ratios between expert and learner distributions, enabling off-policy imitation without explicit policy access. Recent work~\citep{agrawal2024policy, agrawal2025markov, pmlr-v283-agrawal25a} establishes that offline IL must respect a Markov balance condition coupling the expert policy to the underlying transition dynamics. Our formulation shares this occupancy-based perspective but departs fundamentally: rather than correcting for policy mismatch under fixed dynamics, we introduce importance weights through a robust inner maximization over a transition uncertainty set, explicitly accounting for dynamics shift at test time: a regime where BC, DICE, and Markov-balance methods uniformly fail.

\noindent\textbf{Robust Reinforcement Learning.} Robust RL optimizes worst-case returns over an ambiguity set of transition kernels, typically using rectangular uncertainty sets and minimax dynamic programming~\citep{nilim2005robust, iyengar2005robust}. Subsequent work extends this to distributionally robust formulations under total variation, $f$-divergence, and Wasserstein metrics~\citep{derman2020distributional, yu2023fast, pmlr-v283-panaganti25a}. Offline variants estimate worst-case value functions from fixed datasets via nominal-measure reformulations~\citep{panaganti2022robust}, but continue to require known reward signals, rendering them inapplicable to the imitation setting where only expert demonstrations are available. Our work transplants these robust foundations into the IL regime, constructing a fully offline, reward-free formulation that hedges against transition uncertainty using only nominal expert data.

\noindent\textbf{Robust Imitation Learning.} A substantial body of IL work addresses robustness to demonstration imperfections, suboptimal labels, noisy annotations, or covariate shift: through label denoising~\citep{pmlr-v97-wu19a, pmlr-v130-tangkaratt21a}, corrective querying~\citep{pmlr-v15-ross11a}, and noise injection~\citep{pmlr-v78-laskey17a}, but these assume fixed transition dynamics throughout. Meta-IL and multi-task IL frameworks improve cross-task generalization via shared experience~\citep{finn2017one, james2018task, Zhou2020Watch} but similarly do not account for dynamics mismatch. Domain randomization~\citep{peng2018sim, huang2021generalization} perturbs physical parameters during training to improve sim-to-real transfer, but requires a high-fidelity simulator and access to perturbed environments, assumptions that are incompatible with the strictly offline setting we consider. Among methods that explicitly target transition uncertainty, \citet{wu2025robust} enforce Lipschitz regularization against input perturbations under fixed dynamics, while DRIL-DICE~\citep{seo2024mitigating} mitigates covariate shift via $f$-divergence regularization, also under nominal dynamics. RIME~\citep{chae2022robust} learns policies robust across families of MDPs but requires online rollouts for robustness estimation. The most closely related single-task methods are DRBC~\citep{panaganti2023distributionally} and BE-DROIL~\citep{agrawal2025balance}, which robustify offline IL against transition uncertainty; however, both are single-task methods that require independent per-task optimization, incurring prohibitive computational cost in the multi-task regime where policies for many tasks must be recovered efficiently.

\noindent\textbf{Behavior Foundation Models.} BFMs currently achieve state-of-the-art zero-shot performance in RL by pretraining task-agnostic representations that generalize across both goal-conditioned and dense-reward tasks~\citep{touati2021learning, touati2023does}. Conceptually rooted in successor representations~\citep{dayan1993improving}, universal value function approximators~\citep{schaul2015universal}, successor features~\citep{barreto2017successor}, and successor measures~\citep{blier2021learning}, modern BFMs are instantiated as either universal successor features~\citep{borsa2018universal, park2024foundation} or forward-backward (FB) representations~\citep{touati2021learning, jeen2024zero, pirotta2024fast}. Recent work identifies key failure modes of standard BFMs under distributional shift: \citet{bobrin2026zeroshot} show that BFMs conflate policies across dynamics regimes and cannot adapt to unobserved dynamics changes, while \citet{jeen2025zero} demonstrate degradation under partial observability. Both propose architectural remedies: belief encoders and memory-augmented models, respectively, but require pretraining data from multiple dynamics regimes. To the best of our knowledge, this is the first BFM approach that achieves robustness to dynamics variation using only a single nominal offline dataset per environment, with no changes to pretraining and no expert data required under perturbed dynamics.

\appsection{Experimental Setup}\label{sec:appendix_exp_setup}

\appsubsection{ExORL Domains}
We evaluate our methods on three environments from the ExORL benchmark \citep{yarats2022don}, which is built on top of the DeepMind Control Suite \citep{tassa2018deepmind}. 

\appsubsubsection{Walker}
This domain features a two-legged agent that must achieve stable locomotion starting from a crouched configuration. The state space is $24$-dimensional and the action space is $6$-dimensional, comprising joint positions, velocities, and torques. The benchmark defines four tasks: \textit{stand}, \textit{walk}, \textit{run}, and \textit{flip}. The \textit{stand} task rewards maintaining an upright torso with extended legs. The \textit{walk} and \textit{run} tasks extend this objective by encouraging forward motion at different speeds, while \textit{flip} promotes rotational motion of the torso after achieving a standing posture. All reward signals are dense.

\appsubsubsection{Quadruped} This domain involves a four-legged robot navigating within a three-dimensional space. The state and action spaces are $78$- and $12$-dimensional, respectively. We consider four tasks. The \textit{quadruped\_stand} task rewards maintaining an upright posture. The \textit{quadruped\_walk} and \textit{quadruped\_run} tasks additionally encourage forward motion through velocity-based rewards, with \textit{quadruped\_walk} also including a term promoting a minimum center-of-mass height. Finally, \textit{quadruped\_jump} rewards achieving sufficient vertical displacement. Rewards are dense across all tasks.

\appsubsubsection{Cheetah} This domain features a planar biped designed for high-speed locomotion. The state is represented by a $17-$dimensional vector capturing joint positions and velocities, while actions are $6-$dimensional. We evaluate four tasks. In \textit{cheetah\_walk} and \textit{cheetah\_run}, rewards scale with forward velocity up to target speeds of approximately $2\,$m/s and $10\,$m/s, respectively. The tasks \textit{walker\_walk\_backward} and \textit{walker\_run\_backward} instead incentivize achieving specified backward velocities.

\appsubsection{ExORL Datasets}
Training is performed using $1$M transitions (unless otherwise mentioned) sampled uniformly from datasets generated by different unsupervised exploration strategies within ExORL. Specifically, we use data collected by RND \citep{burda2018exploration}, APS \citep{liu2021aps}, and PROTO \citep{yarats2021reinforcement}.

\appsubsection{Environment Perturbations}

We evaluate robustness across three environments: Walker, Quadruped, and Cheetah, each subject to three perturbation types chosen to span distinct physical failure modes. All perturbations are applied at evaluation time only; the BFM pretraining is performed exclusively under nominal dynamics. For each environment, we describe the perturbation mechanism, the parameter modified in the MuJoCo physics model, and the rationale for the chosen sweep range.

\appsubsubsection{Walker}
The Walker is a planar bipedal robot with six actuated joints (hip, knee, ankle $\times$ 2 legs) and four tasks: stand, walk, run, and flip.

\noindent \textbf{Gravity} (+0\% to +35\% of nominal -9.81 m/s²). We scale the vertical component of gravitational acceleration from its nominal value up to a $35$ increase in effective gravitational load. This directly increases the ground reaction forces the robot must overcome to maintain posture and generate forward momentum, stressing both the stance and swing phases of the gait cycle. The range is selected to produce monotonic degradation across all four tasks without inducing physics instability; the flip task is most sensitive due to its requirement for sustained aerial rotation against gravity.

\noindent \textbf{Body Mass }(+0\% to +100\% of nominal). We uniformly scale all non-world body masses: torso, thighs, shins, and feet, by up to $100\%$ of their nominal values. Increased inertia raises the torque demand at every joint during acceleration and deceleration phases, directly stressing the actuator budget. At +100\%, the total robot mass doubles, placing the system well outside the inertial regime experienced during expert demonstration.

\noindent \textbf{Joint Friction Loss} (0 to 30 Nm per joint). We add a constant passive resistive torque opposing motion at all six limb joints, drawn from the MuJoCo \texttt{dof\_frictionloss} parameter (nominal: 0 Nm). Unlike viscous damping, which scales with joint velocity, frictionloss imposes a fixed load regardless of motion speed, effectively raising the torque threshold that must be exceeded before any useful joint motion occurs. At $30$ Nm per joint, the six limb joints collectively impose
$180$ Nm of passive resistance, making fast locomotion tasks (run, flip) substantially more costly while moderately degrading slower tasks (walk, stand).

\appsubsubsection{Quadruped}
The Quadruped is a three-dimensional four-legged robot with twelve actuated joints (yaw, lift, extend $\times$ $4$ legs) and four tasks: stand, walk, run, and jump.

\noindent\textbf{Tilted Gravity} (0 to 7.5 m/s² lateral component). We introduce a lateral gravitational component along the x-axis of the global frame while keeping the vertical component fixed at -9.81 m/s², producing an effective gravity vector tilted up to 37.4° from vertical. This creates a persistent lateral torque on the torso that the policy must continuously counteract to maintain uprightness, directly penalizing the upright reward component (r=r\_{\text{upright}} $\times$ r\_{\text{move}}) that governs all four tasks. The range  $[0,7.5]$ m/s² is chosen as the stable regime; values above $\approx 8$ m/s² induce numerical instability (\texttt{mjWARN\_BADQACC}) in a subset of episodes due to the robot falling and entering degenerate contact configurations. Stand and walk are affected through postural destabilization; run and jump additionally suffer degraded forward propulsion as lateral compensation diverts actuator effort.

\noindent\textbf{Ground Contact Stiffness} (solref timeconst: 0.01 to 0.25 s). We increase the MuJoCo \texttt{geom\_solref} timeconst of the floor geometry from its nominal value of
$0.01$ s to $0.25$ s. The timeconst parameter controls the rise time of contact constraint forces: larger values produce slower, softer contact responses, effectively making the ground more compliant and energy-absorbing, analogous to transitioning from rigid pavement to soft terrain such as mud or foam. The quadruped's gait relies on brief, stiff ground contacts to deliver sharp propulsive impulses; increased compliance reduces the effective push-off force per step, degrading locomotion efficiency across all tasks with a stronger effect on dynamic tasks (run, jump) than static ones (stand).

\noindent\textbf{Actuator Control Range} (command bound: $\pm 1.0$ to $\pm 0.3$). We clip the control signal sent to all twelve actuators symmetrically to $[-v, v]$ where v decreases from the nominal $1.0$ to $0.4$. This is mechanistically distinct from gear ratio scaling: the actuator hardware remains at full torque capacity, but the command channel is truncated so that the policy cannot request torques in the upper $[v,1]$ portion of its action space regardless of what the network outputs. At $v=0.3v$, the effective command range is reduced to $30\%$ of nominal, which for the quadruped's gear-1 actuators directly caps the maximum controllable torque at $30\%$ of the physical limit. This models command saturation or control signal degradation in the communication channel between the policy and the actuators.

\appsubsubsection{Cheetah}
The Cheetah is a planar six-joint runner (bthigh, bshin, bfoot, fthigh, fshin, ffoot) with nominal gear ratios of $120/90/60/90/60/30$ Nm/unit and four tasks: run, run\_backward, walk, and walk\_backward.

\noindent \textbf{Actuator Strength} ($0\%$ to $-30\%$ gear ratio reduction). We scale all six actuator gear ratios down from their nominal values by up to $30\%$, reducing the maximum joint torque to $70\%$ of nominal. The gear ratios are large ($120$ Nm/unit for bthigh), so even a $30\%$ reduction represents a substantial absolute torque loss of up to $36$ Nm on the primary propulsive joint. The range is deliberately conservative: exploratory testing revealed a sharp performance cliff at $-40\%$ reduction where the actuators become insufficient to sustain any meaningful locomotion, making deeper perturbations uninformative for robustness comparison. The $[0\%, 30\%]$ range corresponds to the gradual degradation regime where policy quality rather than physical feasibility determines performance.

\noindent \textbf{Joint Friction Loss} ($0$ to $25$ Nm per joint). We add a constant passive resistive torque opposing motion at all six actuated joints via the \texttt{dof\_frictionloss} parameter (nominal: $0$ Nm). Unlike viscous damping, which scales with joint velocity, frictionloss imposes a fixed load regardless of motion speed, effectively raising the torque threshold that must be exceeded before any useful joint motion occurs. At $25$ Nm per joint, the six joints collectively impose $150$ Nm of passive resistance per control cycle, substantial relative to the lighter actuators (ffoot: $30$ Nm/unit nominal gear). The effect is particularly pronounced on the Cheetah due to the high-frequency nature of its gait: fast locomotion tasks (run, run\_backward) require rapid joint reversals on every stride, paying the frictionloss penalty twice per cycle (opposing both flexion and extension). The range $[0,25]$ Nm is chosen to span the regime from negligible degradation to near-collapse of high-speed locomotion while preserving meaningful signal across all four tasks.

\noindent \textbf{Range of Motion} ($100\%$ to $60\%$ of nominal joint range retained). We shrink the mechanical range of motion of all six limb joints symmetrically around their center positions, reducing the allowed angular excursion to a specified fraction of the nominal range. For example, the bfoot joint with nominal range $[-4.01,0.87]$ rad is restricted to $[-2.52,0.37]$ rad at $60\%$ retention. This models physical range restriction arising from joint damage, swelling, or mechanical interference. The Cheetah's fast gait requires large angular excursions, particularly at bfoot and fshin, to generate the leg sweeping motion that drives forward velocity; restricting these excursions forces the policy into a kinematically constrained regime that was not present during expert demonstration. The range $[100\%,60\%]$ is chosen to avoid the sharp feasibility cliff below $60\%$ where joint constraints render normal locomotion physically impossible for all methods.

\begin{table}[t]
\centering
\small
\begin{tabular}{l l}
\toprule
\textbf{Hyperparameter} & \textbf{Value} \\
\midrule
Latent dimension $d$ & 50 \\
$F$ / $\psi$ dimensions & (1024, 1024) \\
$B$ / $\varphi$ dimensions & (256, 256, 256) \\
Preprocessor dimensions & (1024, 1024) \\
Std.\ deviation for policy smoothing $\sigma$ & 0.2 \\
Truncation level for policy smoothing & 0.3 \\
Learning steps & 1,000,000 \\
Batch size & 512 \\
Optimizer & Adam \citep{kingma2014adam} \\
Learning rate & 0.0001 \\
Discount $\gamma$ & 0.98 (0.99 for maze) \\
Activations (unless otherwise stated) & ReLU \\
Target network Polyak smoothing coefficient & 0.01 \\
$z$-inference labels & 10,000 \\
$z$ mixing ratio & 0.5 \\
\bottomrule
\end{tabular}
\caption{Hyperparameters for BFM methods.}
\label{tab:bfm_pretraining_hyperparameters}
\end{table}

\begin{table}[t]
\centering
\small
\begin{tabular}{l l}
\toprule
\textbf{Hyperparameter} & \textbf{Value} \\
\midrule
Optimization steps & $3{,}000$ \\
Batch size & $512$ \\
Learning rate & $0.001$ \\
Optimizer & Adam~\citep{kingma2014adam} \\
Loss & MSE \\
Policy smoothing std.\ deviation $\sigma$ & $0.1$ \\
\bottomrule
\end{tabular}
\caption{Hyperparameters for FB-IL.}
\label{tab:fbil_hyperparameters}
\end{table}

\begin{table}[t]
\centering
\small
\begin{tabular}{l l}
\toprule
\textbf{Hyperparameter} & \textbf{Value} \\
\midrule
Robustness parameter $\varepsilon_l$ & $0.8$ \\
Optimization steps & $5{,}000$ \\
Batch size & $512$ \\
Learning rate & $0.0005$ \\
Optimizer & Adam~\citep{kingma2014adam} \\
Policy smoothing std.\ deviation $\sigma$ & $0.1$ \\
(Task Vector $:$ Dual Update) Frequencies & $1:1$\\
\bottomrule
\end{tabular}
\caption{Hyperparameters for RBFM-Light.}
\label{tab:rbfm_light_hyperparameters}
\end{table}

\appsubsection{Evaluation Protocol}

Our pipeline consists of two stages: (i) pre-training of the BFM and (ii) downstream task inference.

\appsubsubsection{Pre-training Evaluation}
Following \citep{jeen2024zero}, during BFM pre-training, we assess performance by measuring the cumulative reward (score) achieved on each task. Training is run for $10^6$ steps, and evaluation is performed every $20,000$ steps. At each evaluation point, we execute $10$ rollouts per task and compute the interquartile mean (IQM) of the resulting scores to obtain a robust estimate of performance. We select the checkpoint corresponding to the highest IQM aggregated across all tasks. This checkpoint is used for all subsequent downstream experiments.

\appsubsubsection{Downstream Task Inference}
Using the selected pre-trained checkpoint, we train policies for downstream imitation learning tasks across $5$ random seeds. Performance is reported in terms of cumulative return, averaged across $100$ episodes for a given seed. Final results are presented with $95\%$ confidence intervals computed across seeds.

\appsubsection{Computational Resources}
\label{subsec:app_compute}
All experiments were conducted on a single NVIDIA Quadro P$5000$ GPU with $16$GB memory. Pretraining of BFM took approximately $14$ hours for $1M$ steps.

\appsection{Implementation Details}\label{sec:appendix_implementation_details}

\appsubsection{Forward-Backward Representations}
\appsubsubsection{Architecture}

Our implementation of the forward-backward framework is based on the design introduced by \citep{jeen2024zero}, with specific hyperparameter choices summarized in Table \ref{tab:bfm_pretraining_hyperparameters}.

\appsubsubsection{Forward Representation \texorpdfstring{$F(s,a,z)$}{F(s,a,z)}}

The forward network operates on preprocessed inputs. In particular, the state-action pair $(s,a)$ and the state-task pair $(s,z)$ are each passed through separate MLPs, which project them into a shared embedding space of dimension $512$. These two embeddings are then concatenated and fed into an additional MLP, which outputs a embedding vector of dimension $d$.

\appsubsubsection{Backward Representation \texorpdfstring{$B(s)$}{B(s)}}

The backward network is implemented as an MLP that takes the state $s$ as input and produces a $d$-dimensional embedding vector.

\appsubsubsection{Actor \texorpdfstring{$\pi_z(s)$}{pi\_z(s)}}

The policy network adopts a similar design to the forward model. The inputs $(s,a)$ and $(s,z)$ are first encoded independently using separate MLPs into $512$-dimensional embeddings. These embeddings are concatenated and passed through a final MLP that generates an action vector of dimension $a$, corresponding to the action space. A Tanh activation is applied at the output layer to ensure bounded actions. During training, Gaussian noise with standard deviation $\sigma$ is injected into the actions for smoothing, following \citep{fujimoto2019off}.

\appsubsubsection{Additional Details.}

To improve training stability, layer normalization \citep{ba2016layer} is applied, and Tanh activations are used in the first layer of each MLP.

\appsubsection{Task Sampling Distribution \texorpdfstring{$\mathcal{Z}$}{Z}}

FB representations require a strategy for generating the task vector $z$ during training. Following \citep{touati2022does}, we adopt a combination of two sampling procedures. First, $z$ is drawn uniformly from the surface of a hypersphere with radius $\sqrt{d}$ in $\mathbb{R}^d$. Second, $z$ is obtained by encoding states $s \sim \mathcal{D}$ through the backward model, i.e., $z = B(s)$. Due to the $\ell_2$ normalization applied in the backward network, these vectors also lie on the hypersphere, although their distribution is not uniform. At each training step, we sample $z$ using these two approaches with equal probability.

\appsubsection{Discussion on usage of \texorpdfstring{$f$}{f}-divergence in RBFM-Heavy} \label{subsec:app_f_div}

Lemma \ref{lemma:D_TV_mismatch_s_a_s_prime} provides a total variation (TV) bound on the mismatch between
occupancy measures under nominal and perturbed dynamics. We extend this to
general $f$-divergences by restricting to generators bounded by the TV
generator.

\begin{lemma}
Let $f:[0,\infty)\to\mathbb{R}$ satisfy $f(1)=0$ and
$f(t)\leq \alpha f_{\mathrm{TV}}(t)$ for all $t\geq 0$, where
$f_{\mathrm{TV}}(t)=\frac{1}{2}|t-1|$. Then, for any policy $\pi$ and
$T\in\mathcal{T}(\varepsilon')$,
\begin{equation}
D_f\!\left(\rho_T^\pi \,\|\, \rho_{T_o}^\pi\right)
\leq \alpha D_{\mathrm{TV}}\!\left(\rho_T^\pi, \rho_{T_o}^\pi\right)
\leq \alpha \frac{\varepsilon'}{1-\gamma}.
\end{equation}
\label{lemma:app_f_div_result}
\end{lemma}
\begin{proof}
    See Lemma $2$ in \citep{agrawal2025balance} for details. 
\end{proof}

Our formulation further requires $f$ to be differentiable with invertible
$f'$. Since $f_{\mathrm{TV}}$ is non-differentiable at $t=1$, we use the
smooth SoftTV generator:
\begin{equation}
f_{\mathrm{SoftTV}}(x)=\frac{1}{2}\log\bigl(\cosh(x-1)\bigr), \quad
(f'_{\mathrm{SoftTV}})^{-1}(y)=\tanh^{-1}(2y)+1.
\end{equation}

As $f_{\mathrm{SoftTV}}(x)\leq f_{\mathrm{TV}}(x)$ for all $x$, setting
$\alpha=1$ yields
\begin{equation}
D_{f_{\mathrm{SoftTV}}}\!\left(\rho_T^\pi \,\|\, \rho_{T_o}^\pi\right)
\leq \frac{\varepsilon'}{1-\gamma}.
\end{equation}

We use SoftTV in practice, though any generator satisfying the conditions in Lemma \ref{lemma:app_f_div_result} is admissible.

\appsubsection{BFM-based Hyperparameters}

\appsubsubsection{Pretraining}
We use the publicly available code at \citep{zheng2026breeze} for BFM and RBFM-based approaches. The hyperparameters used for pretraining BFM model are provided in Table \ref{tab:bfm_pretraining_hyperparameters}. Since, FB-IL, RBFM-Light, and RBFM-Heavy share the same pretrained BFM, these hyperparameters are common to all these approaches.

\appsubsubsection{Task-Inference}

\noindent\textbf{FB-IL.}
The hyperparameters used for inferring task vector in FB-IL are provided in Table \ref{tab:fbil_hyperparameters}.

\noindent\textbf{RBFM-Light.}
The hyperparameters used for inferring task vector in RBFM-Light are provided in Table \ref{tab:rbfm_light_hyperparameters}.

\noindent\textbf{RBFM-Heavy.}
\noindent\textbf{$Q$-network architecture.} The $Q$-function in RBFM-Heavy is parameterized as a two-layer MLP with hidden dimension $64$ and ReLU activations, operating on the frozen forward model features $\bar{F}(s, a, z_w)$ evaluated at a fixed warm-start task embedding $z_w$. Specifically, given the pretrained forward model $F$, we compute $\bar{F}(s, a, z_w) \in \mathbb{R}^d$ and pass it through the MLP to produce a scalar $Q$-value estimate. Critically, a key motivation for building on BFMs is their ability to infer task-specific policies within minutes; using a large  $Q-$network over raw observations and actions would render inference computationally prohibitive, defeating this purpose entirely. By operating on the pretrained FB model instead, the $Q$-network remains lightweight while still leveraging the rich representational structure already encoded during pretraining. The hyperparameters used for task vector inference in RBFM-Heavy are provided in Table~\ref{tab:rbfm_heavy_hyperparameters}.

\begin{table}[t]
\centering
\small
\begin{tabular}{l l}
\toprule
\textbf{Hyperparameter} & \textbf{Value} \\
\midrule
Robustness parameter $\varepsilon$ & 0.8 \\
Optimization steps & $5{,}000$ \\
Batch size & $512$ \\
$\gamma$ & $0.99$ \\
Learning rate & $0.0003$ \\
Optimizer & Adam~\citep{kingma2014adam} \\
Policy smoothing std.\ deviation $\sigma$ & $0.1$ \\
$Q$-network architecture & MLP \\
$Q$-network hidden dimension & 64 \\
$Q$-network hidden layers & 2 \\
$Q$-network activation & ReLU \\
(Task Vector $:$ Dual Update) Frequencies & $1:1$\\
Dual variables learning rate & $0.0003$ \\
\bottomrule
\end{tabular}
\caption{Hyperparameters for RBFM-Heavy.}
\label{tab:rbfm_heavy_hyperparameters}
\end{table}

\appsubsection{Single Task Robust Offline IL baselines}

We compare against the following Robust Offline IL baselines in our experiments.

\appsubsubsection{DRBC}
We implement DRBC \citep{panaganti2023distributionally} in PyTorch using a policy network with two hidden layers of 256 units each and Tanh activations. Models are trained for $2 \times 10^6$ gradient steps with a batch size of 256, optimized using Adam. The learning rate follows the schedule in \citep{panaganti2023distributionally}, with an initial value of $1 \times 10^{-4}$ for all environments except HalfCheetah, where $1 \times 10^{-5}$ is used. The learning rate is decayed exponentially every $10000$ steps with a factor between 0.9 and 0.95. The robustness parameter $\varepsilon_l$ is set to $\{0.2, 0.3, 0.5, 0.6\}$ for Cheetah, Walker, and Quadruped, respectively (see Table \ref{tab:drbc_hyperparameters}).

\begin{table}[t]
\centering
\small
\resizebox{\textwidth}{!}{
\begin{tabular}{l c c c c c c c c}
\toprule
\textbf{Environment} & \textbf{Robustness param $\varepsilon_l$} & \textbf{Policy layers} & \textbf{Activation} & \textbf{Max steps} & \textbf{Learning rate} & \textbf{LR decay} & \textbf{Decay rate} & \textbf{Decay freq.} \\
\midrule
Cheetah & 0.3 & (256, 256) & Tanh & 2M & $1 \times 10^{-5}$ & True & 0.9  & 10k \\
Walker & 0.5 & (256, 256) & Tanh & 2M & $1 \times 10^{-4}$ & True & 0.95 & 10k \\
Quadruped & 0.6 & (256, 256) & Tanh & 2M & $1 \times 10^{-4}$ & True & 0.9  & 10k \\
\bottomrule
\end{tabular}
}
\caption{Hyperparameter configuration for DRBC}
\label{tab:drbc_hyperparameters}
\end{table}

\appsubsubsection{BE-DROIL}
BE-DROIL \citep{agrawal2025balance} method is also implemented in PyTorch. The policy is modeled as a TanhGaussian network with two fully connected layers of $256$ units, producing a Gaussian distribution whose outputs are squashed via a Tanh transformation to respect action bounds. Both the policy and Q-function are trained using Adam with a learning rate of $5 \times 10^{-5}$ and a batch size of 512. The Q-function is parameterized as a two-layer MLP with 256 hidden units and ReLU activations, and we use a discount factor of $\gamma = 0.99$. The policy objective minimizes the mean squared error between actions sampled from the learned policy and those from the expert i.e. $
\mathcal{L}_{\pi}(s) = \mathrm{MSE}(a, a_D), \quad a \sim \pi(\cdot \mid s), \; a_D \sim \pi_D(\cdot \mid s)$, where $\pi$ and $\pi_D$ denote the learned and expert policies, respectively. Training is performed for $10^6$ gradient steps, alternating between a policy update and a joint update of the Q-function and $\tau$ network. See Table \ref{tab:be_droil_hyperparameters} for additional details.

\begin{table}[t]
\centering
\small
\resizebox{\textwidth}{!}{
\begin{tabular}{l l c c c c c c c}
\toprule
\textbf{Component} & \textbf{Architecture} & \textbf{Hidden units} & \textbf{Activation} & \textbf{Batch size} & \textbf{Learning rate} & \textbf{Steps} & $\gamma$ & \textbf{Update ratio (policy : Q, $\tau$)} \\
\midrule
Policy $\pi$   & TanhGaussian MLP & (256, 256) & ReLU & 512 & $5 \times 10^{-5}$ & 1M & 0.99 & 1:1 \\
Q-function $Q$ & MLP             & (256, 256) & ReLU & 512 & $5 \times 10^{-5}$ & 1M & 0.99 & 1:1 \\
\bottomrule
\end{tabular}
}
\caption{Hyperparameter configuration for BE-DROIL.}
\label{tab:be_droil_hyperparameters}
\end{table}

\appsection{Algorithm}
Pseudo-code for pre-training of BFM is provided in Algorithm \ref{alg:bfm_pretraining_alg}. This mimics the Algorithm $1$ of \citep{jeen2024zero}, with conservative part removed. Pseudo-code to extract the task vector for our RBFM-Light method is provided in Algorithm \ref{alg:rbfm_light_alg}. Finally, Algorithm \ref{alg:rbfm_heavy_alg} provides the pseudo-code for our RBFM-Heavy algorithm.

\begin{algorithm}[!htbp]
\caption{Unsupervised Pre-training of Forward-Backward Representations}
\label{alg:bfm_pretraining_alg}
\begin{algorithmic}[1]
\REQUIRE Exploratory Dataset $\mathcal{D}$ of trajectories
\STATE Initialize $F_{\theta_F}$, $B_{\theta_B}$, and policy $\pi_z$
\STATE Set number of steps $N$, $z$-distribution $\mathcal{Z}$, Polyak factor $\nu$, batch size $b$

\FOR{$n = 1$ to $N$}
    \STATE Sample mini-batch $\{(s_i, a_i, s_{i+1})\}_{i=1}^b \sim \mathcal{D}$
    \STATE Sample $\{z_i\}_{i=1}^b \sim \mathcal{Z}$

    \STATE \textbf{// FB Update}
    \STATE Sample $a_{i+1} \sim \pi(s_{i+1}, z_i)$
    \STATE Update $F, B$ using $\{(s_i, a_i, s_{i+1}, a_{i+1}, z_i)\}$ using Eq. \ref{eqn:bfm_pretraining})
    
    \STATE \textbf{// Actor Update}
    \STATE Sample $a_i \sim \pi(s_i, z_i)$
    \STATE Update actor to maximize Eq. \eqref{eqn:actor_pretraining}, i.e.,
    \[
        \mathbb{E}\big[F(s_i, a_i, z_i)^\top z_i\big]
    \]

    \STATE \textbf{// Target network updates (Polyak averaging)}
    \STATE $\theta_F^{-} \leftarrow (1 - \nu) \theta_F^{-} +\nu \theta_F$
    \STATE $\theta_B^{-} \leftarrow (1 - \nu)\theta_B^{-} + \nu\theta_B$
\ENDFOR
\end{algorithmic}
\end{algorithm}

\begin{algorithm}[!htbp]
\caption{RBFM-Light}
\label{alg:rbfm_light_alg}
\begin{algorithmic}[1]
\REQUIRE Expert dataset $\mathcal{D}_e = \{(s_j, a_j)\}_{j=1}^N$, uncertainty radius $\epsilon_l$
\STATE Warm-start  $z \leftarrow z_w$ as defined in Implementation details in Section \ref{sec:experiments} 
\FOR{iteration $k = 1, \dots, K$}
\STATE Sample mini-batch $\{(s_i, a_i)\}_{i=1}^b \sim \mathcal{D}_e$
\STATE \textbf{//RBFM-Light objective}
\begin{align*}
L &:= \max_{i \in [b]} \, \mathcal{L}(\pi_z(\cdot \mid s_i), a_i), \\[4pt]
\mathcal{L}_{\text{RBFM-Light}}(z)
&= \min_{\lambda \in [0, (1+\epsilon_l)L]}
\Bigg[
\frac{1}{b} \sum_{i=1}^{b}
\big(\mathcal{L}(\pi_z(\cdot \mid s_i), a_i) - \lambda \big)_+ \\
&\qquad\qquad\qquad\qquad\qquad\quad + \epsilon_l
\big(\max_{i \in [b]} \mathcal{L}(\pi_z(\cdot \mid s_i), a_i) - \lambda \big)_+
+ \lambda
\Bigg].
\end{align*}

\STATE \textbf{// Dual update}
\STATE $\lambda \leftarrow \arg\min_{{\lambda \in [0, (1+\epsilon_l)L]}} \; \mathcal{L}_{\text{RBFM-Light}}(z, \lambda)$
\STATE \textbf{// Policy update}
\STATE $z \leftarrow \arg\min_z \; \mathcal{L}_{\text{RBFM-Light}}(z, \lambda)$
\ENDFOR
\STATE \textbf{return} $z$
\end{algorithmic}
\end{algorithm}

\begin{algorithm}[!htbp]
\caption{RBFM-Heavy}
\label{alg:rbfm_heavy_alg}
\begin{algorithmic}[1]
\REQUIRE Expert Dataset $\mathcal{D}_e = \{(s_j, a_j, s_j', a_j')\}_{j=1}^N$, initial state distribution $\mu$, 
discount $\gamma$, robustness parameters $\varepsilon$

\STATE Initialize $Q_\theta$, scalar $\tau \ge 0$, and warm-start  $z \leftarrow z_w$ as defined in Implementation details in Section \ref{sec:experiments} 

\FOR{iteration $k = 1, \dots, K$}

    \STATE Sample mini-batch $\{(s_i,a_i,s_i',a_i')\}_{i=1}^{b} \sim \mathcal{D}_e$
    \STATE Sample $\{(s_0^{(i)}, a_0^{(i)})\}_{i=1}^{b} \sim \mathcal{D}_e$ (See Implementation details in Experiments section \ref{sec:experiments})
    \STATE \textbf{// Step 1: Compute optimal density ratio $w^\star$}
    \STATE Compute $\{c_{Q, \pi_z}(s_i,a_i,s_i')\}_{i=1}^{b}$ as defined in Eq. \ref{eqn:final_lagrangian_constrained_bc}

    \STATE Compute $\{w^\star_{Q,\tau,\pi_z}(s,a,s')\}_{i=1}^{b}$ via closed-form in Proposition \ref{proposition:optimal_importance_weight}

    \STATE \textbf{// Step 2: Update critic $(Q,\tau)$}
    \STATE Estimate gradient of:
    \[
    \begin{aligned}
    \mathcal{L}_{\theta, \tau} = 
    &\frac{(1-\gamma)}{b}\sum_{i=1}^{b} Q_{\theta}(s_0^{i}, a_0^{i})  + \epsilon \tau \\
    &-\frac{\varepsilon}{b}\left[\sum_{i=1}^{b}\left(f(w^\star_{Q_{\theta},\tau,\pi_z}(s_i,a_i,s_i') + w^\star_{Q_{\theta},\tau,\pi_z}(s_i,a_i,s_i') c_{Q_{\theta},\pi_z}(s_i,a_i,s_i')\right)\right]
    \end{aligned}
    \]
    \STATE Update $\theta \leftarrow \theta - \eta_Q \nabla_\theta \mathcal{L}_{Q_{\theta},\tau}$
    \STATE Update $\tau \leftarrow \max(0, \tau - \eta_\tau \nabla_\tau \mathcal{L}_{Q_{\theta},\tau})$

    \STATE \textbf{// Step 3: Policy update (actor)}
    \STATE Estimate:
    \[
    \mathcal{L}_{\pi_z} = \sum_{i=1}^{b} \frac{w^\star_{Q_{\theta},\tau,\pi_z}(s_i,a_i,s_i')L_{\pi_z}(s_i)}{b} 
    \]
    \STATE Update $z \leftarrow z - \eta_\pi \nabla_z \mathcal{L}_{\pi_z}$

\ENDFOR

\STATE \textbf{return} $(Q_\theta, \tau, z)$

\end{algorithmic}
\end{algorithm}

\appsection{Extended Results}

\begin{figure}[t]
\centering
\setlength{\tabcolsep}{4pt}
\renewcommand{\arraystretch}{0.0}

\begin{tabular}{>{\centering\arraybackslash}m{0.5cm}
                >{\centering\arraybackslash}m{0.30\textwidth}
                >{\centering\arraybackslash}m{0.30\textwidth}
                >{\centering\arraybackslash}m{0.30\textwidth}}

&
\textbf{\hspace{0.2cm}Gravity} &
\textbf{\hspace{0.2cm}Mass} &
\textbf{\hspace{0.4cm}Joint Friction Loss} \\[2pt]

\cellcolor{runblue}\rotatebox{90}{\small\textbf{Run}} &
\cellcolor{runblue}\includegraphics[width=\linewidth]{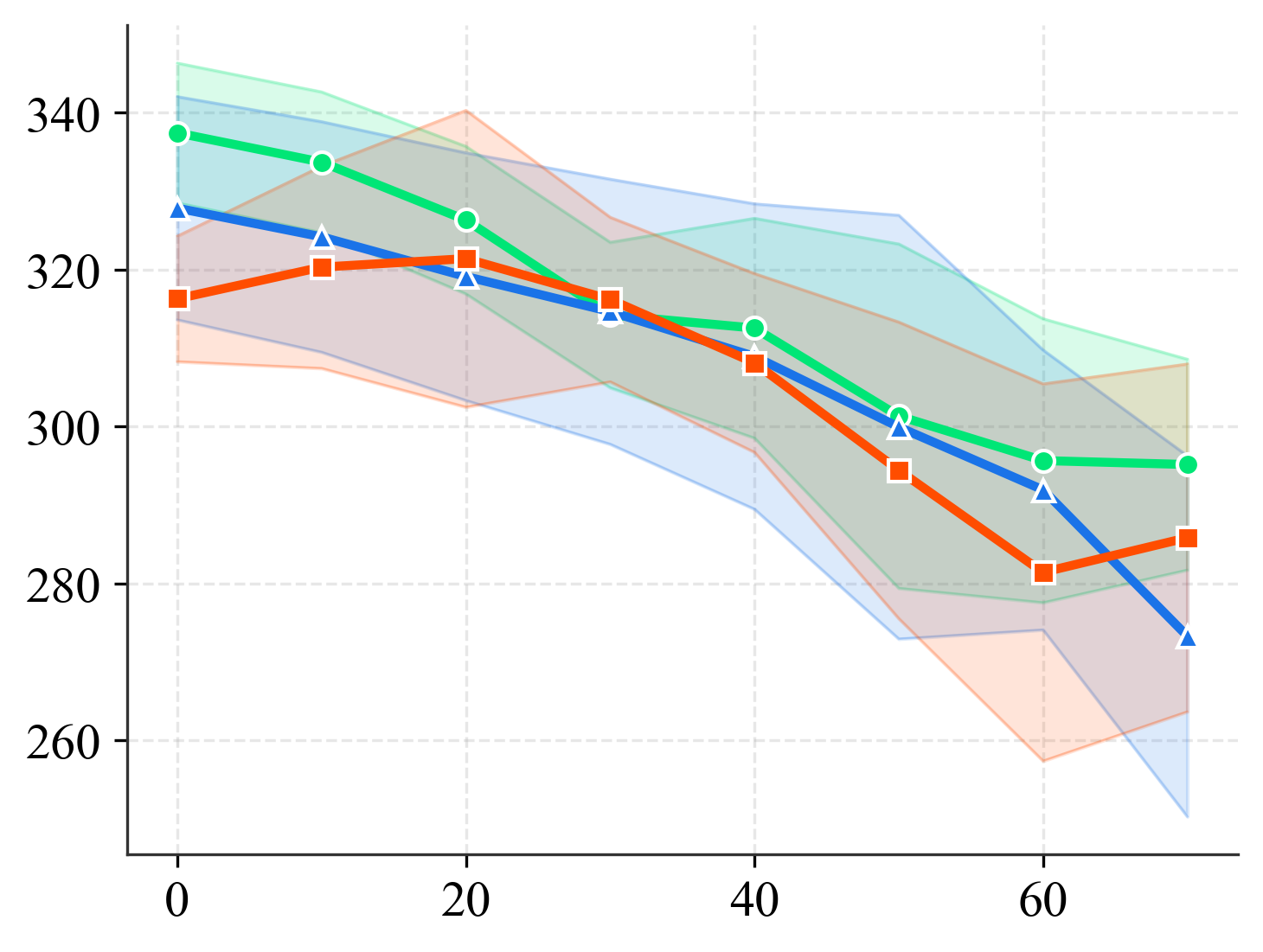} &
\cellcolor{runblue}\includegraphics[width=\linewidth]{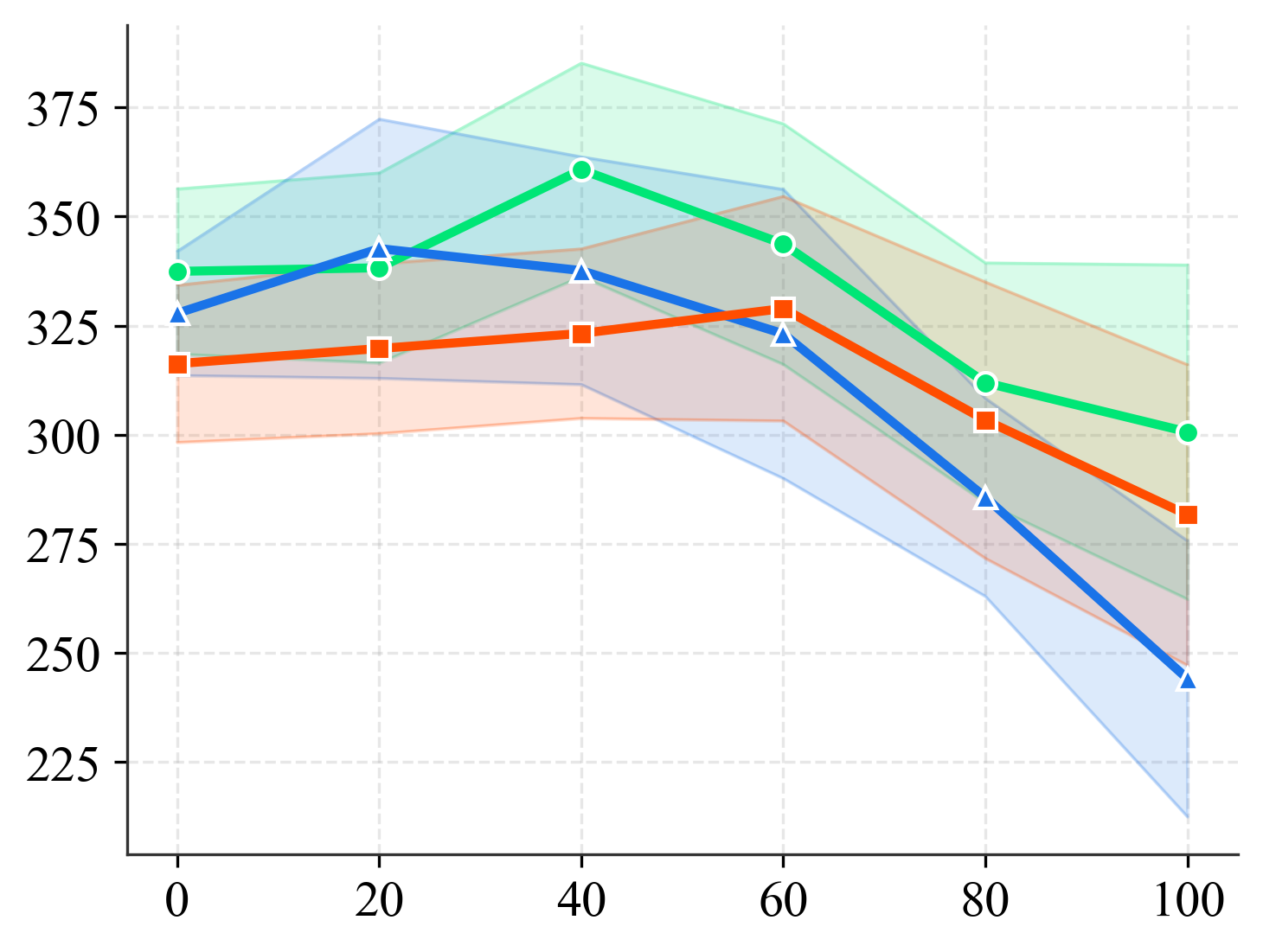} &
\cellcolor{runblue}\includegraphics[width=\linewidth]{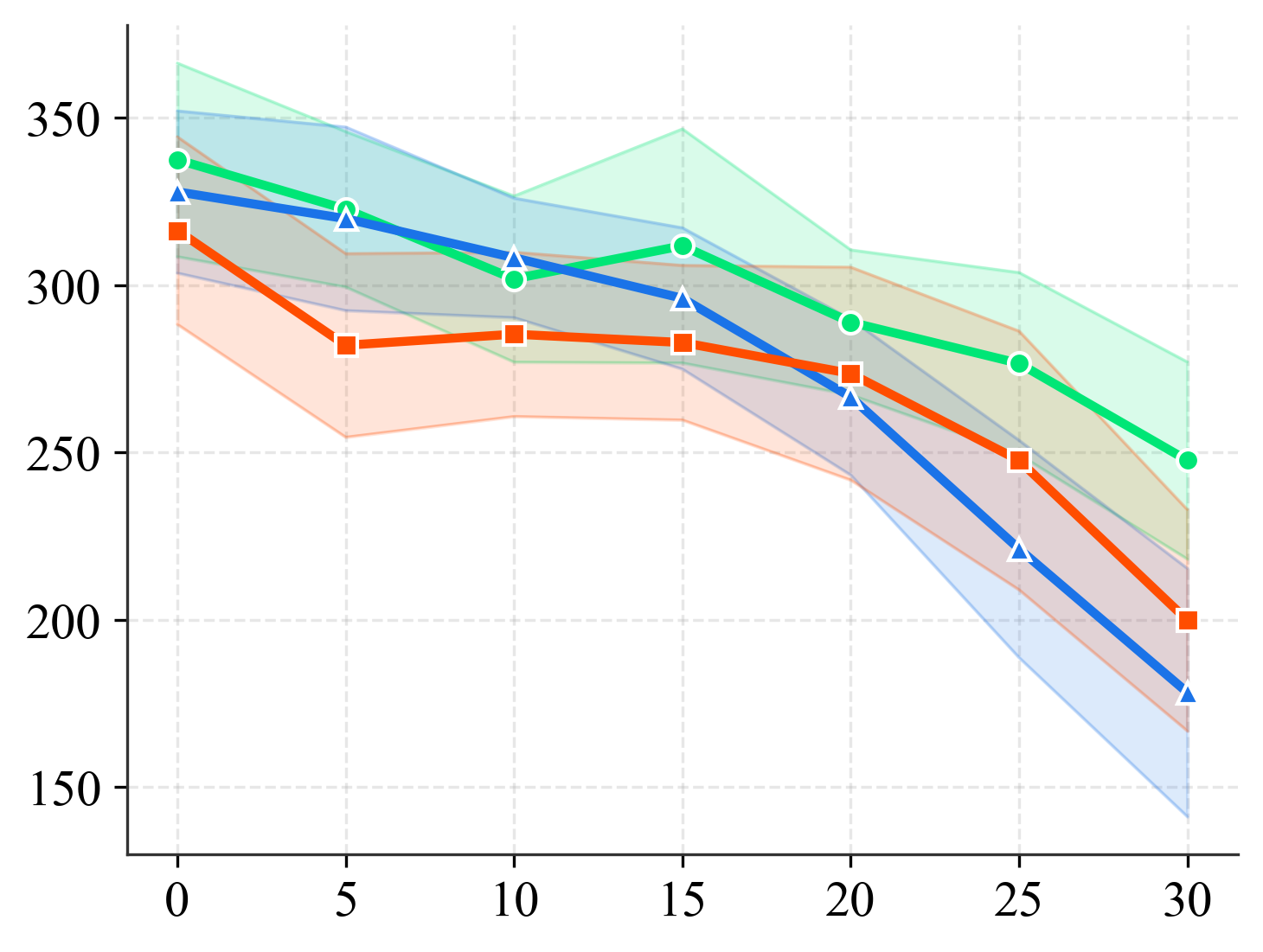} \\[1pt]

\cellcolor{walkgreen}\rotatebox{90}{\small\textbf{Walk}} &
\cellcolor{walkgreen}\includegraphics[width=\linewidth]{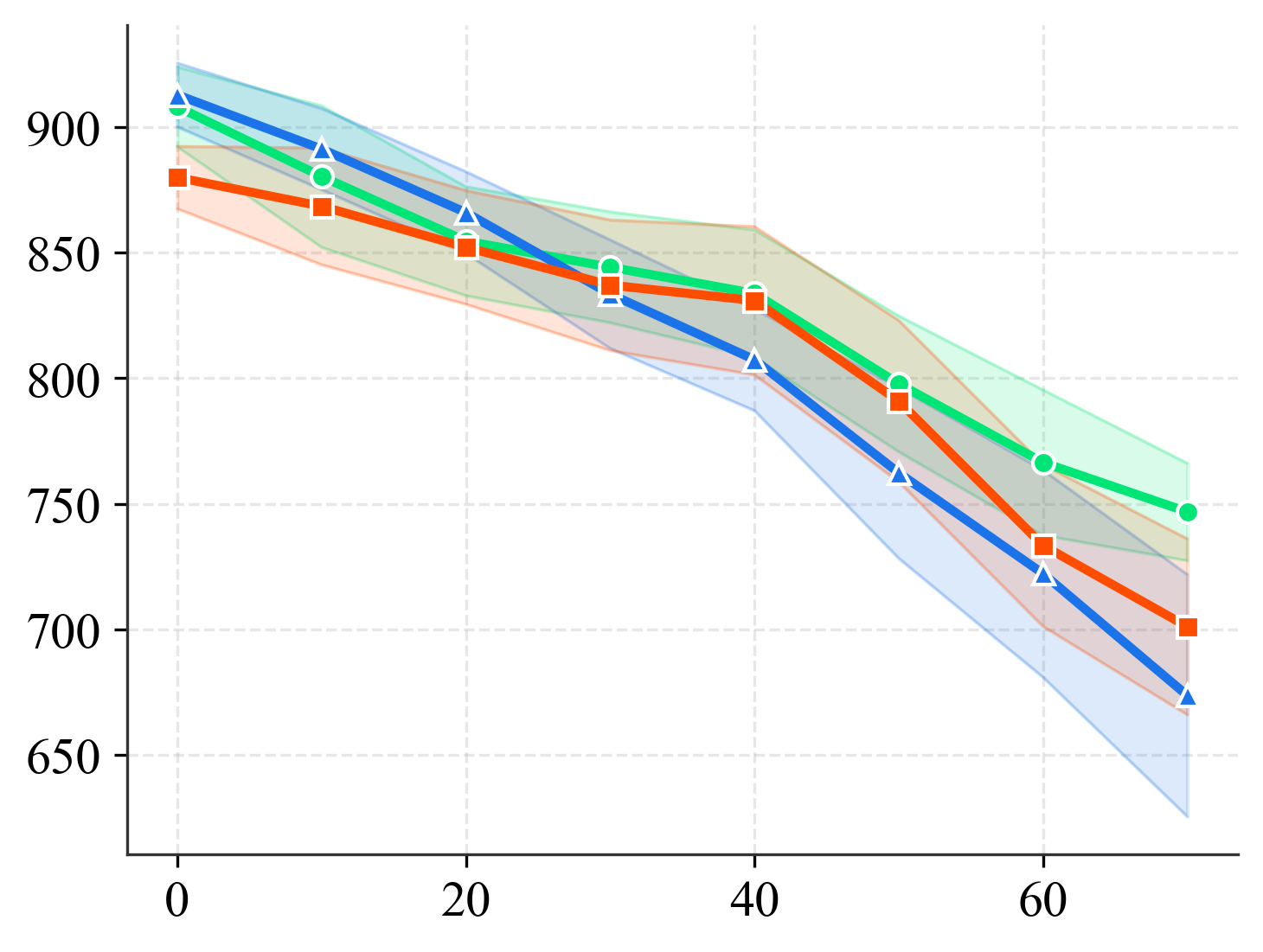} &
\cellcolor{walkgreen}\includegraphics[width=\linewidth]{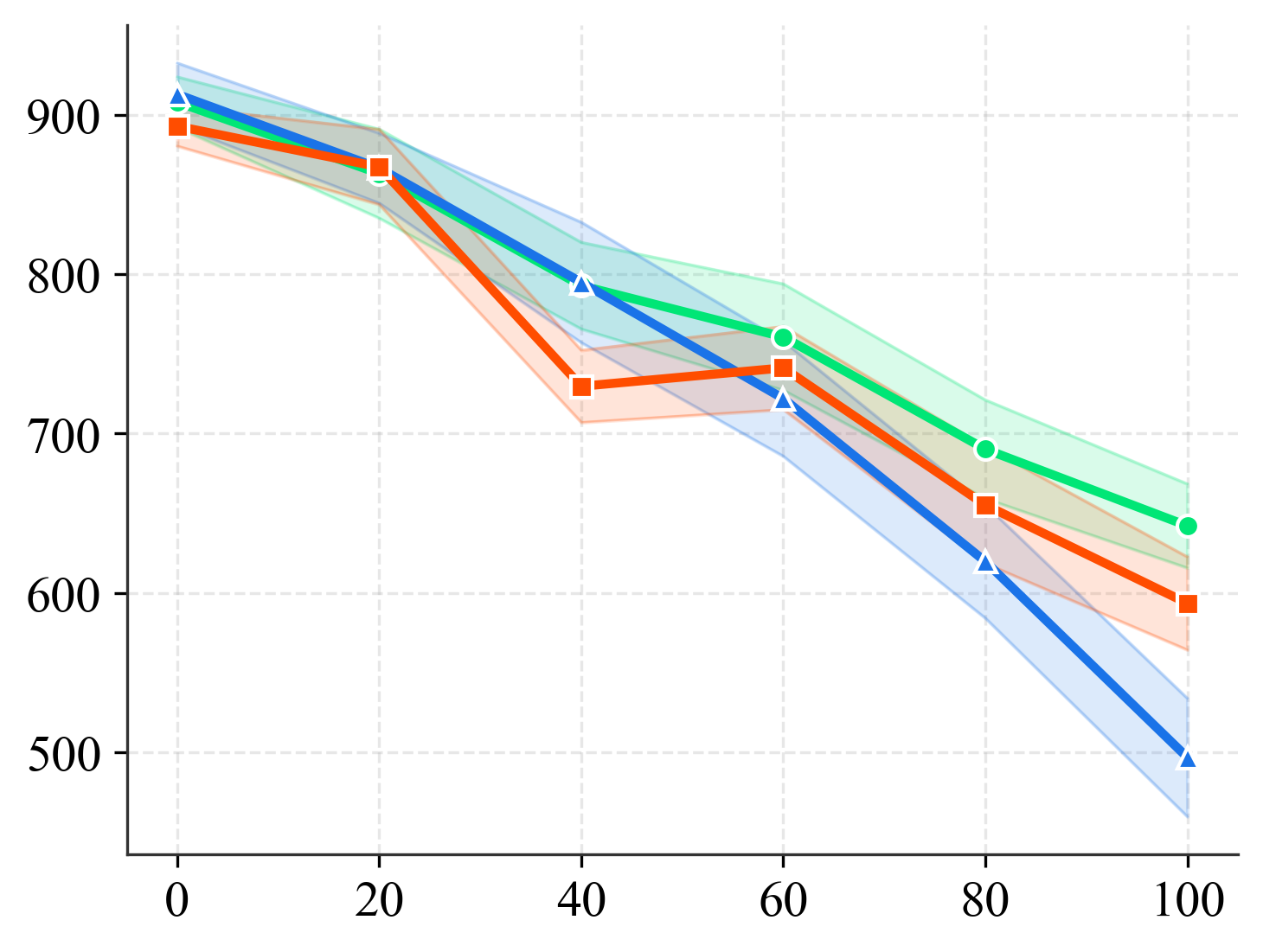} &
\cellcolor{walkgreen}\includegraphics[width=\linewidth]{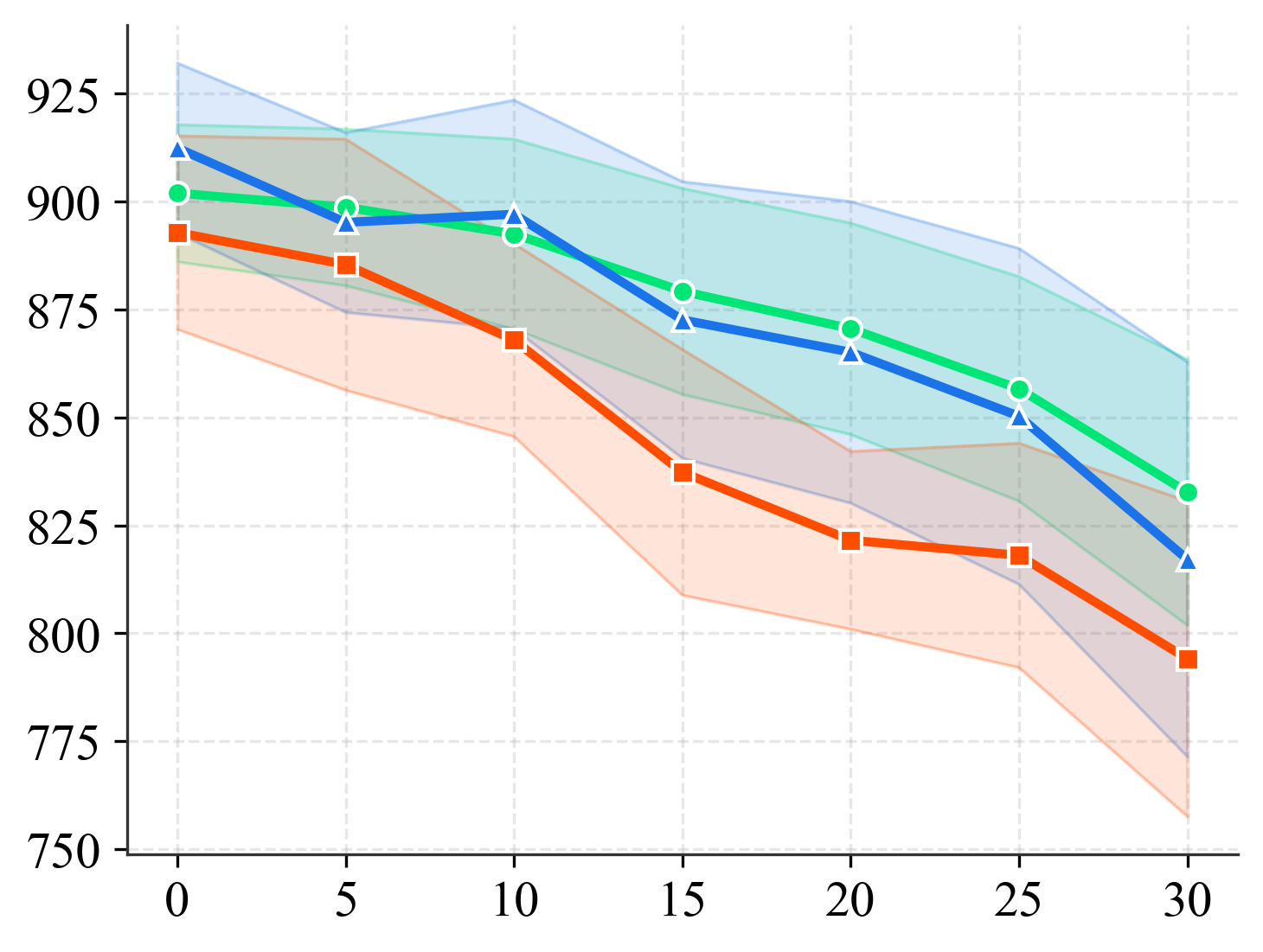} \\[1pt]

\cellcolor{fliporange}\rotatebox{90}{\small\textbf{Flip}} &
\cellcolor{fliporange}\includegraphics[width=\linewidth]{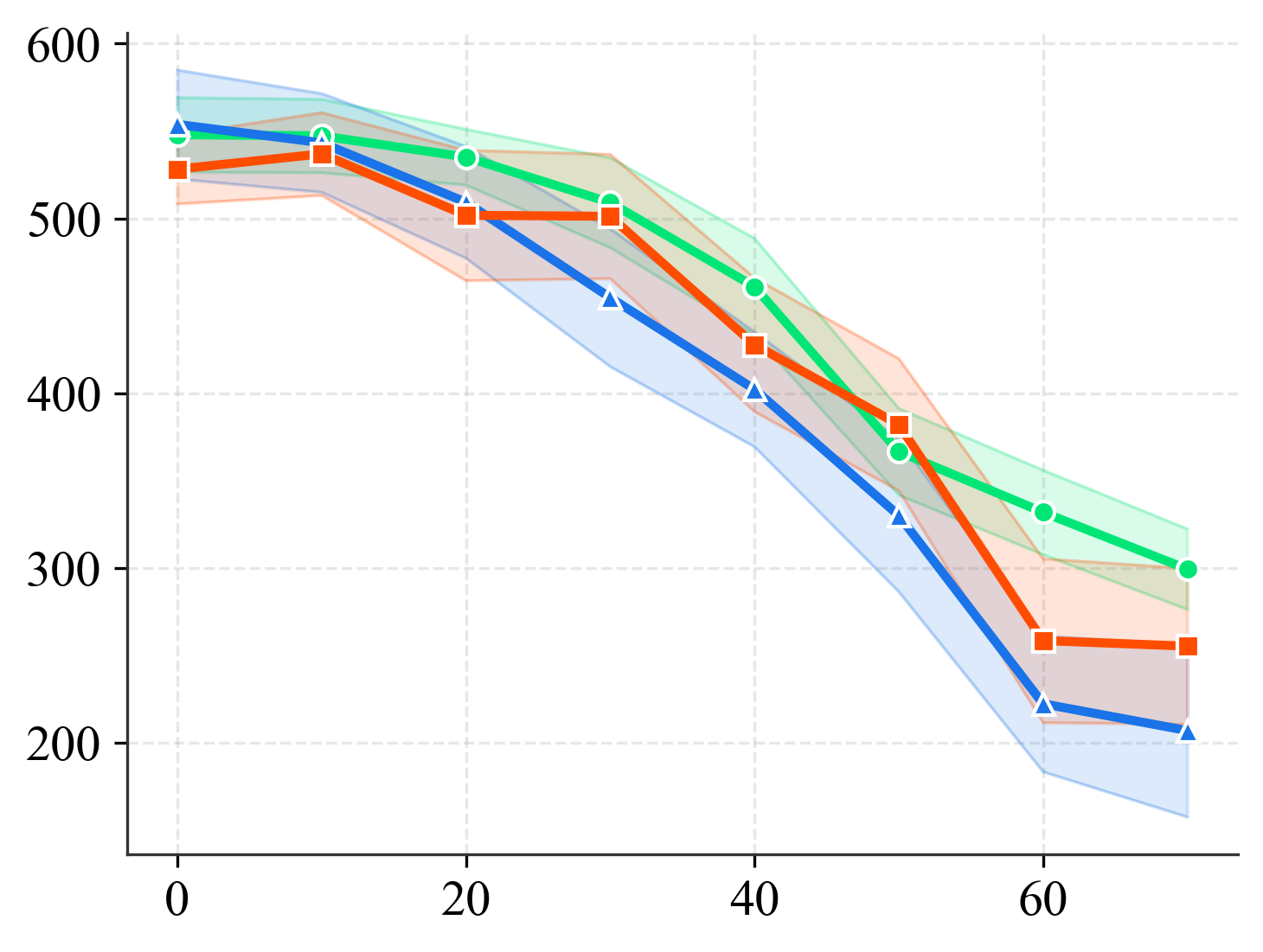} &
\cellcolor{fliporange}\includegraphics[width=\linewidth]{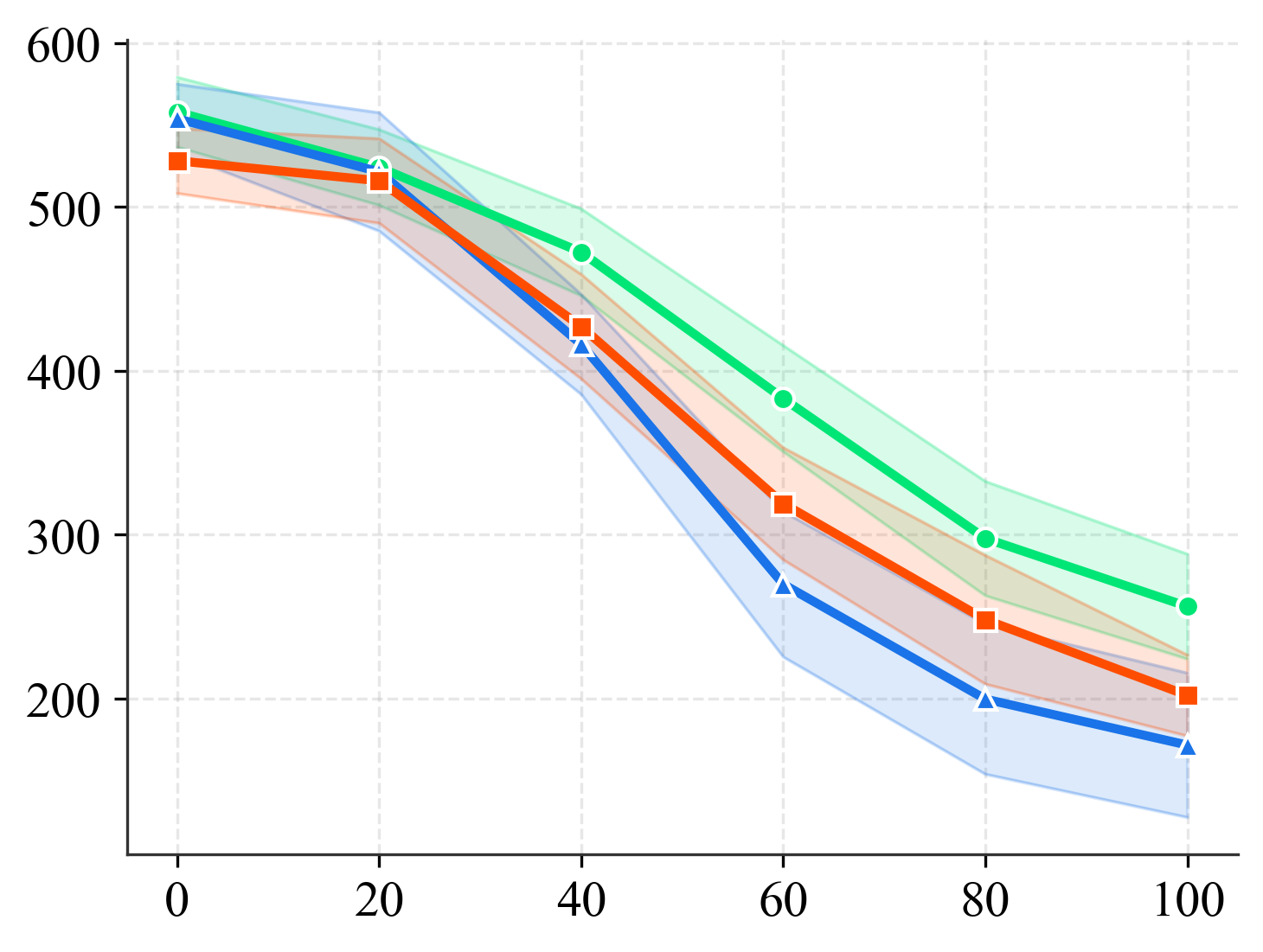} &
\cellcolor{fliporange}\includegraphics[width=\linewidth]{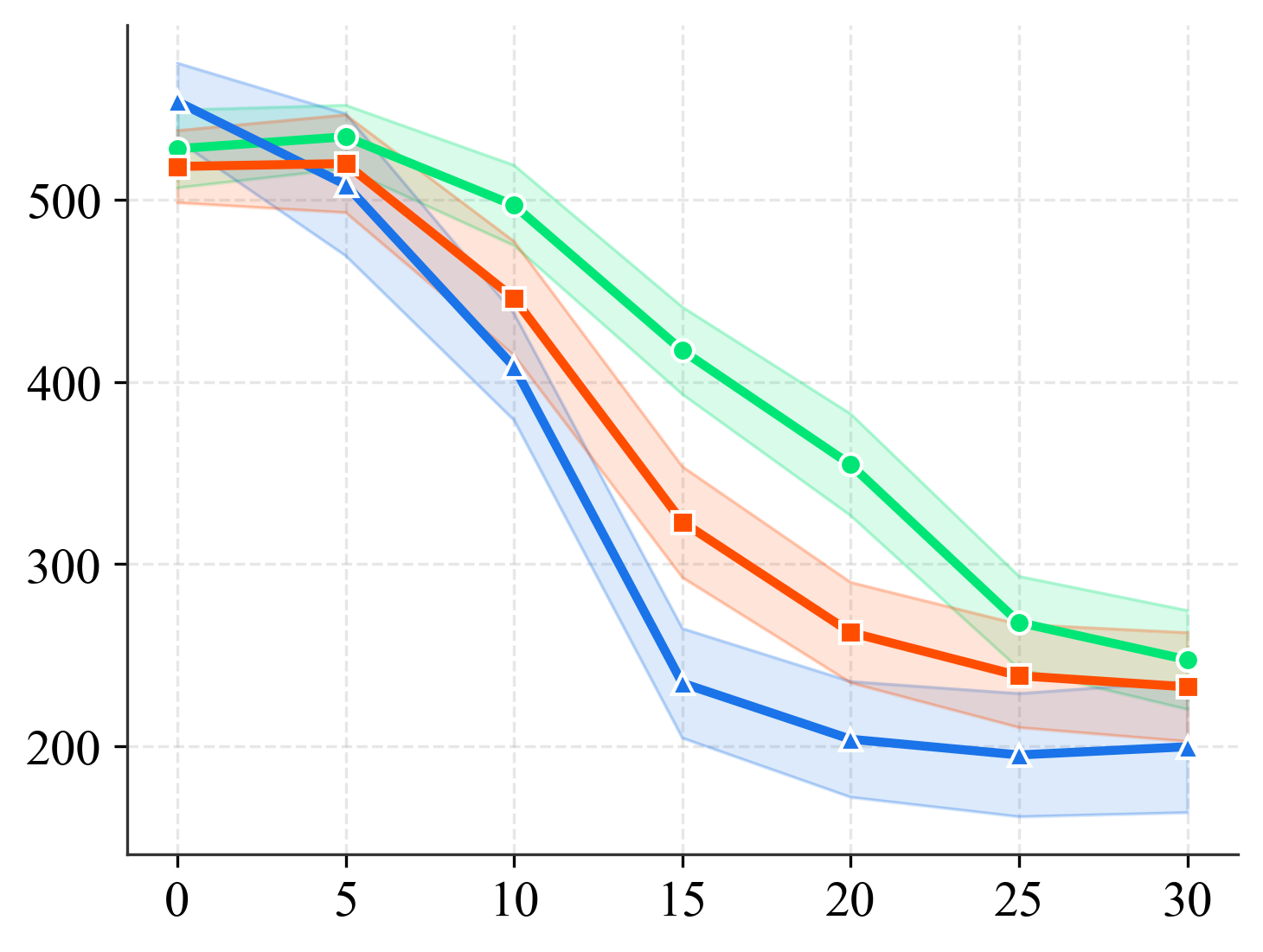} \\[1pt]

\cellcolor{standpurple}\rotatebox{90}{\small\textbf{Stand}} &
\cellcolor{standpurple}\includegraphics[width=\linewidth]{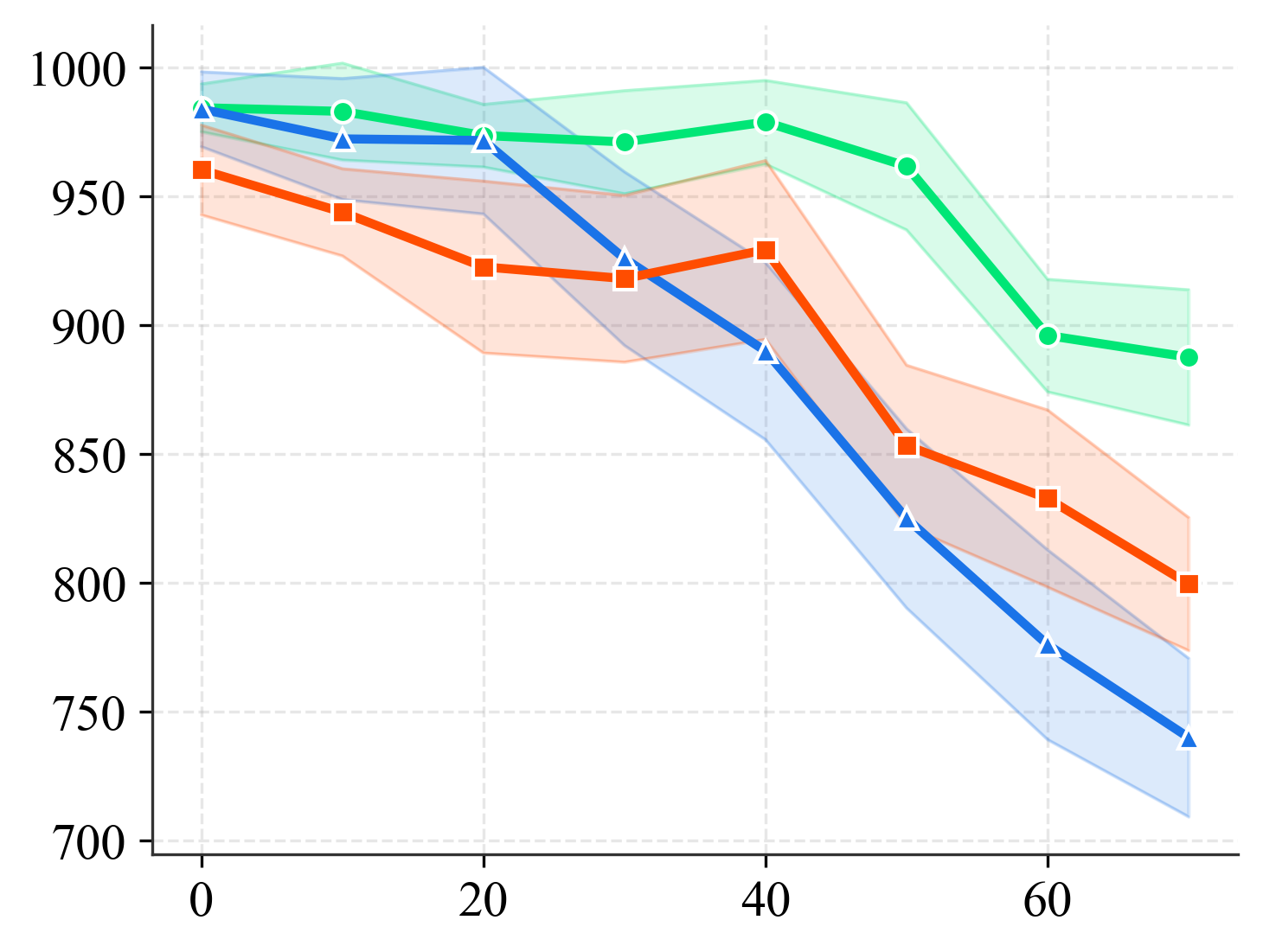} &
\cellcolor{standpurple}\includegraphics[width=\linewidth]{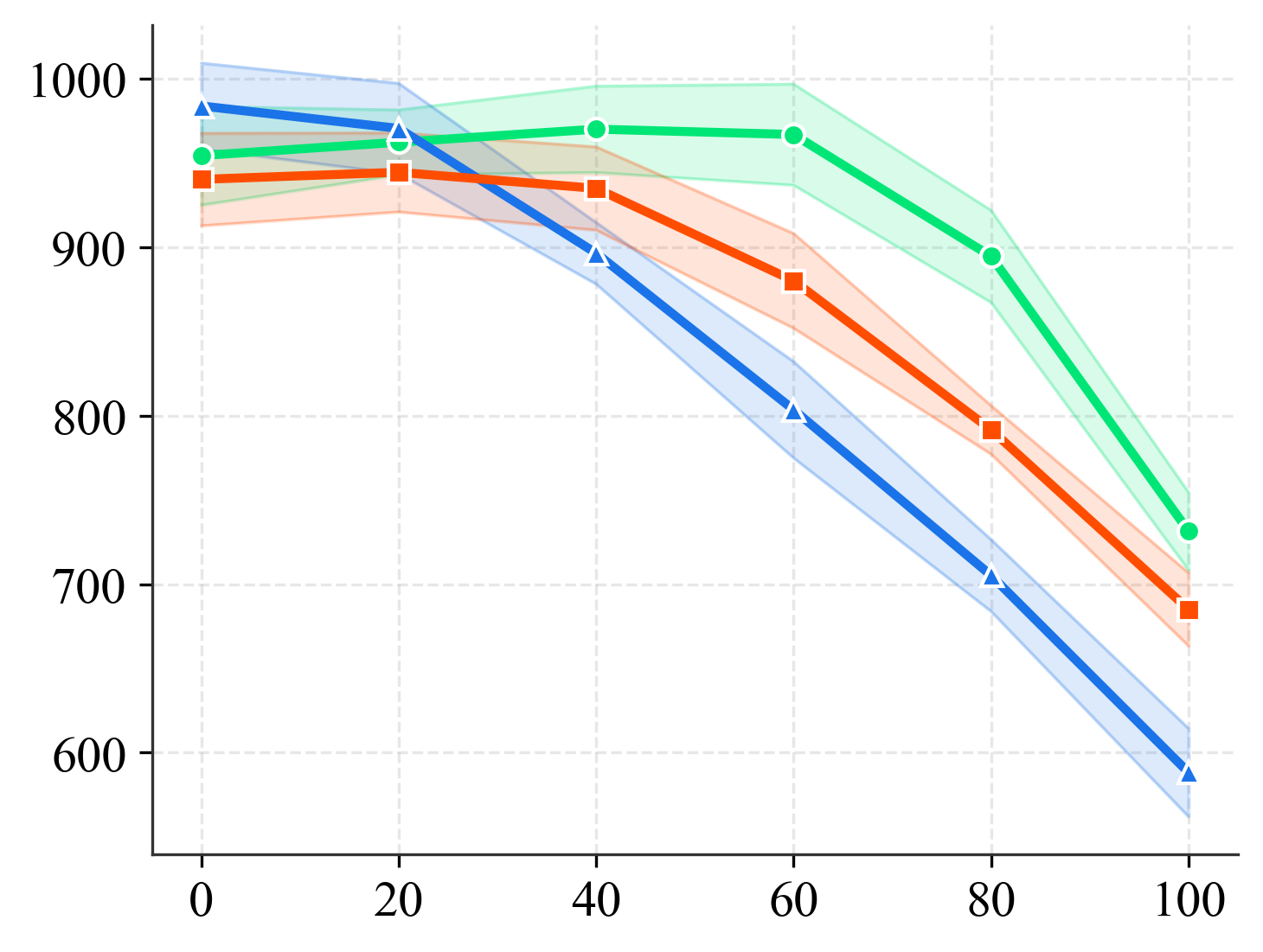} &
\cellcolor{standpurple}\includegraphics[width=\linewidth]{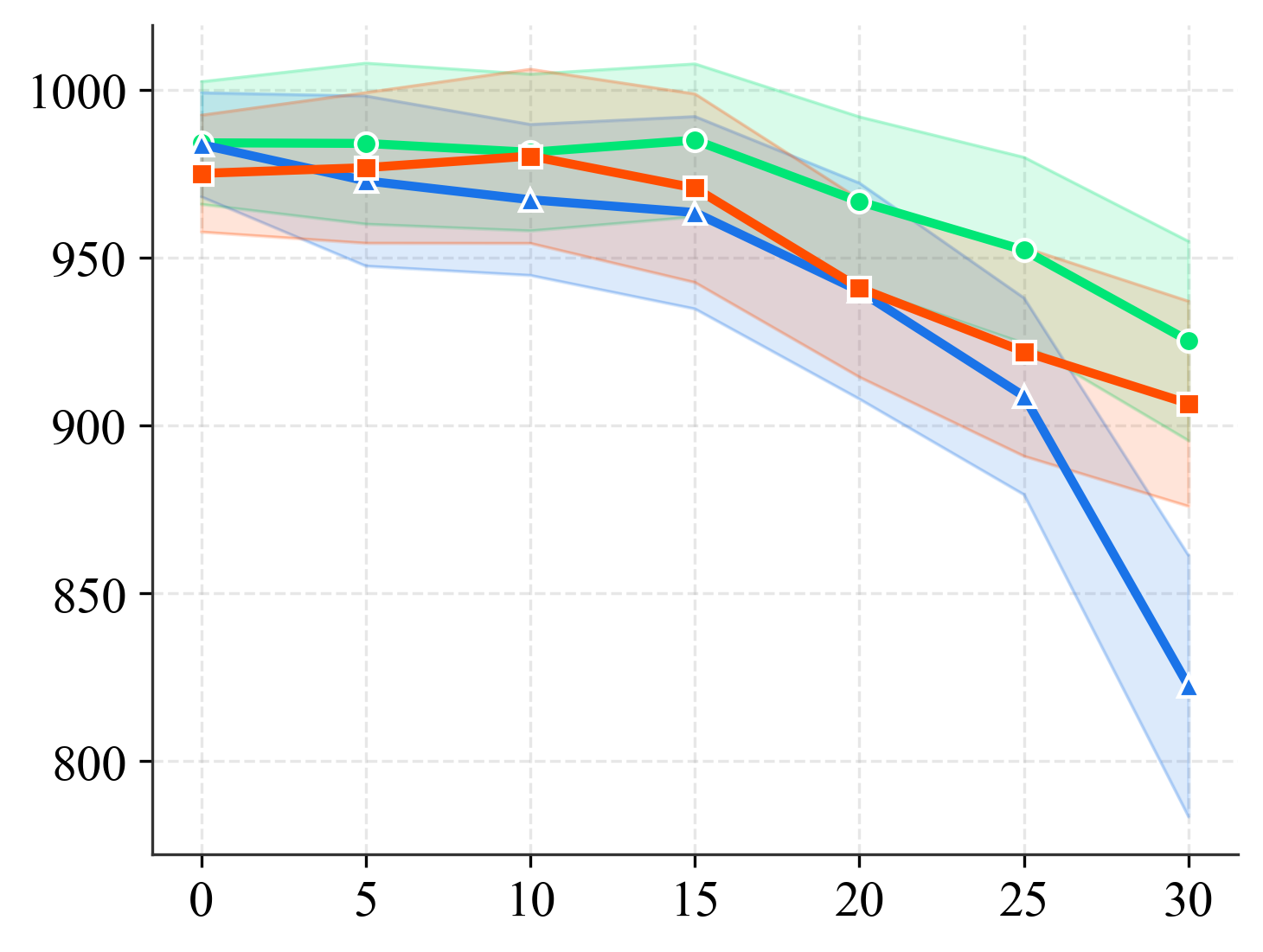} \\

\end{tabular}

\vspace{0.2cm}
\begin{minipage}{0.95\textwidth}
\centering
\footnotesize
\legendcircle{rbfmheavy}{RBFM-Heavy}\hspace{3.0em}
\legendtriangle{fbil}{FB-IL}\hspace{3.0em}
\legendsquare{rbfmlight}{RBFM-Light}
\end{minipage}

\caption{Walker performance with $95\%$ confidence interval across four tasks (rows) and three perturbation types (columns): gravity and body mass increase the physical load on the robot (\% change from nominal); joint friction loss adds passive resistive torque at each joint, simulating mechanical wear (absolute Nm per joint). Pretrained on RND data.}
\label{fig:walker}
\end{figure}

\begin{figure}[t]
\centering
\setlength{\tabcolsep}{4pt}
\renewcommand{\arraystretch}{0.0}

\begin{tabular}{>{\centering\arraybackslash}m{0.5cm}
                >{\centering\arraybackslash}m{0.30\textwidth}
                >{\centering\arraybackslash}m{0.30\textwidth}
                >{\centering\arraybackslash}m{0.30\textwidth}}

&
\textbf{\hspace{0.2cm}Range of Motion} &
\textbf{\hspace{0.2cm}Actuator Strength} &
\textbf{\hspace{0.4cm}Joint Friction Loss} \\[2pt]

\cellcolor{runblue}\rotatebox{90}{\small\textbf{Run}} &
\cellcolor{runblue}\includegraphics[width=\linewidth]{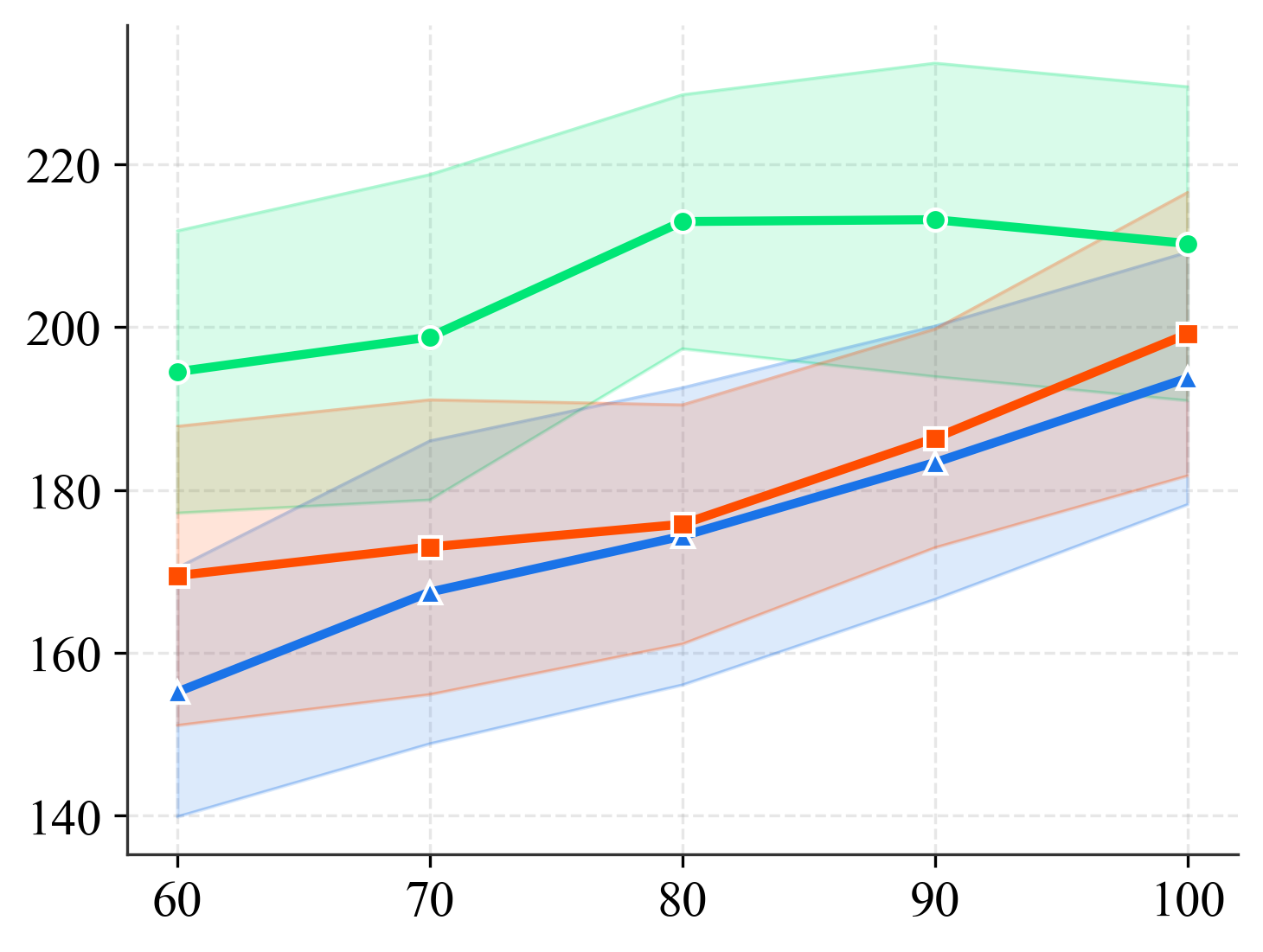} &
\cellcolor{runblue}\includegraphics[width=\linewidth]{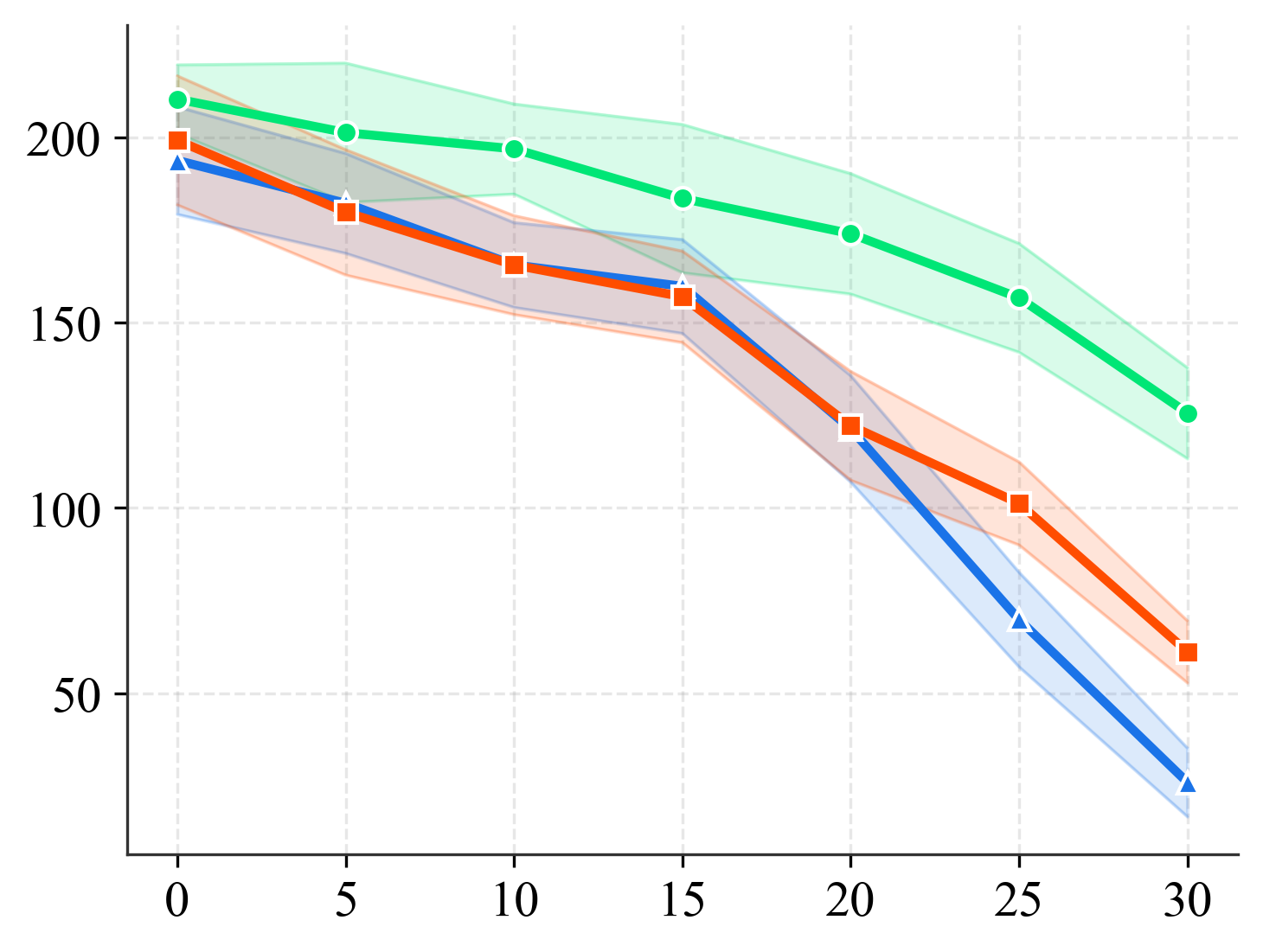} &
\cellcolor{runblue}\includegraphics[width=\linewidth]{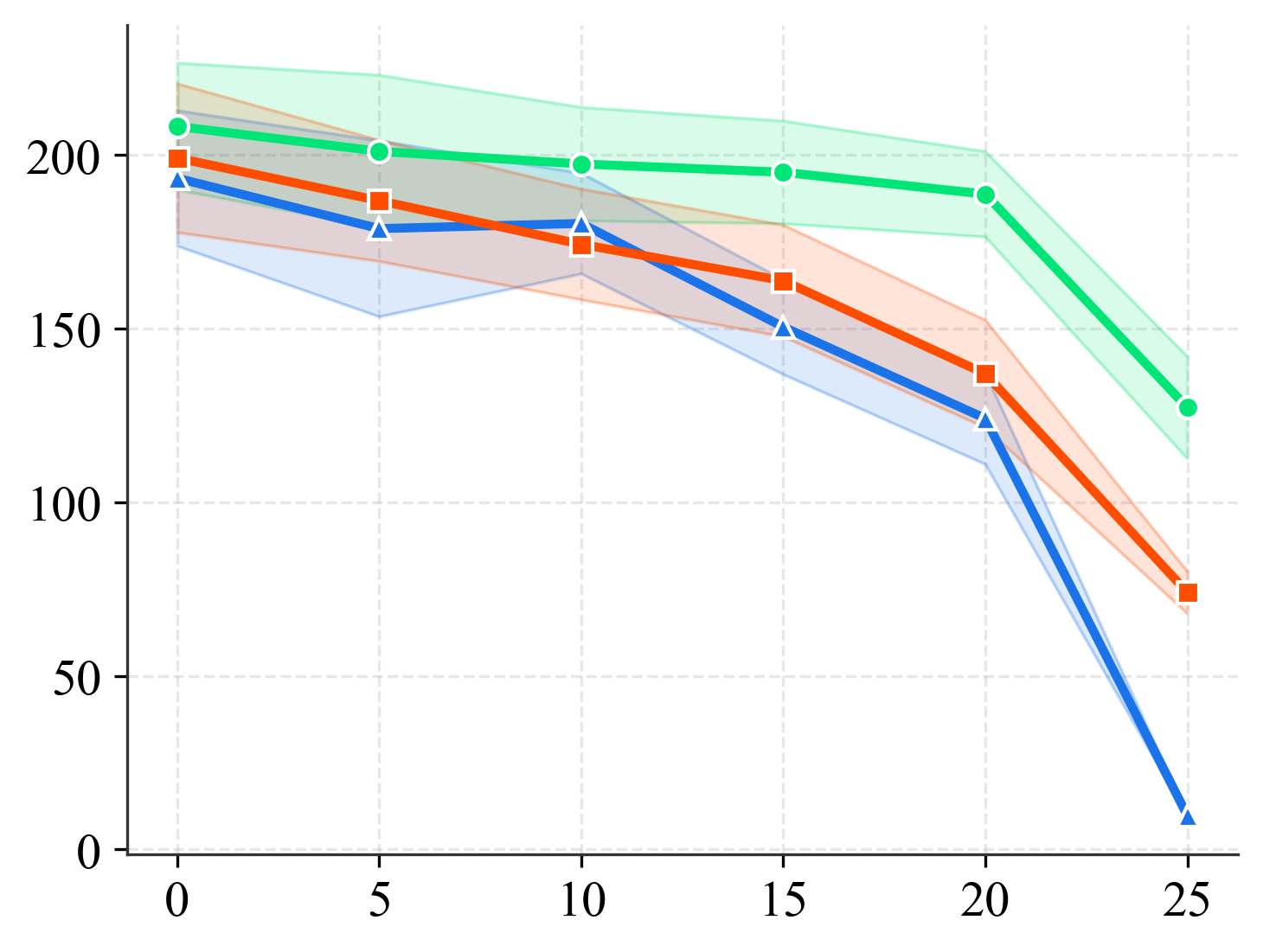} \\[1pt]

\cellcolor{walkgreen}\rotatebox{90}{\small\textbf{Run Backward}} &
\cellcolor{walkgreen}\includegraphics[width=\linewidth]{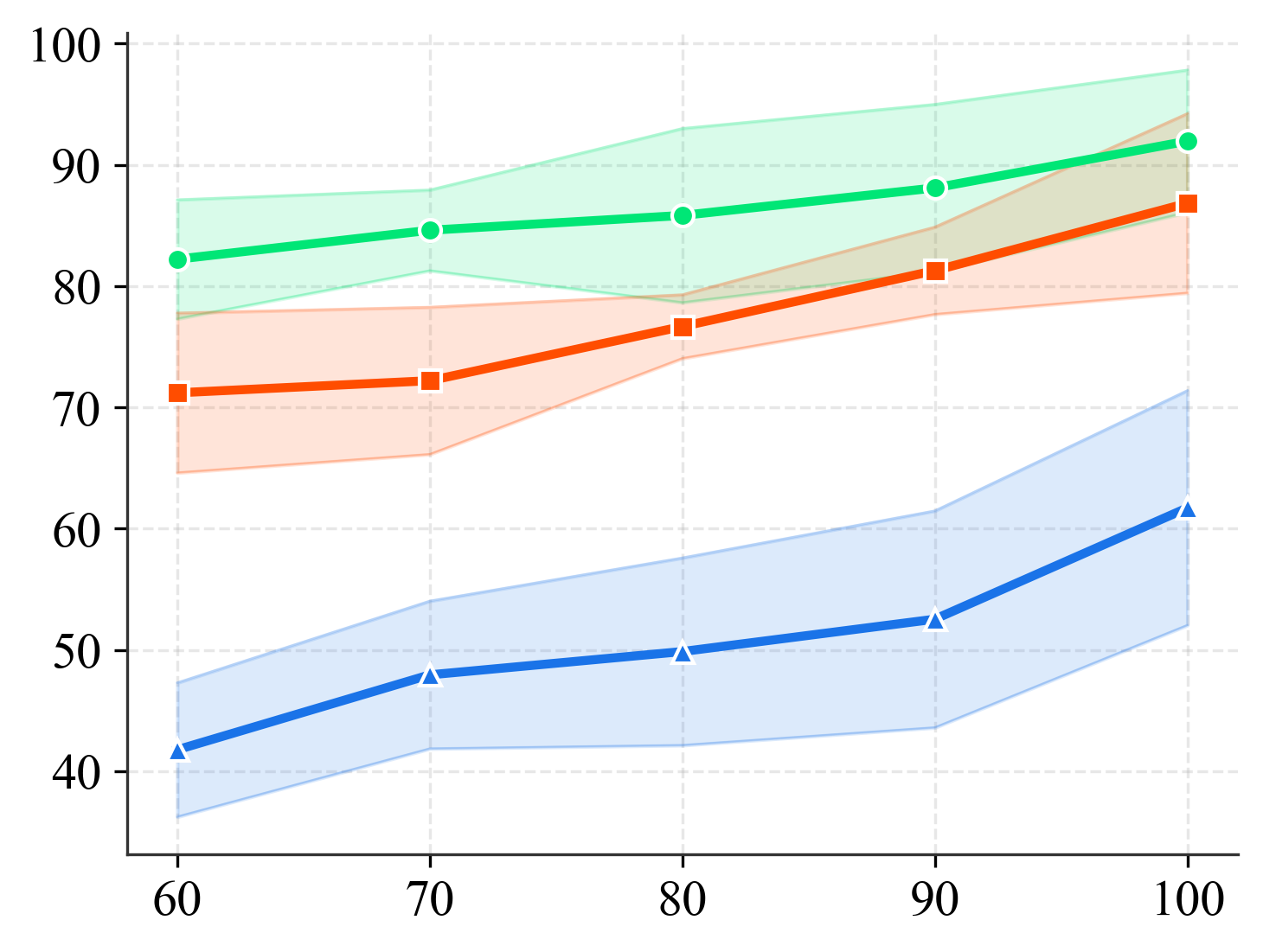} &
\cellcolor{walkgreen}\includegraphics[width=\linewidth]{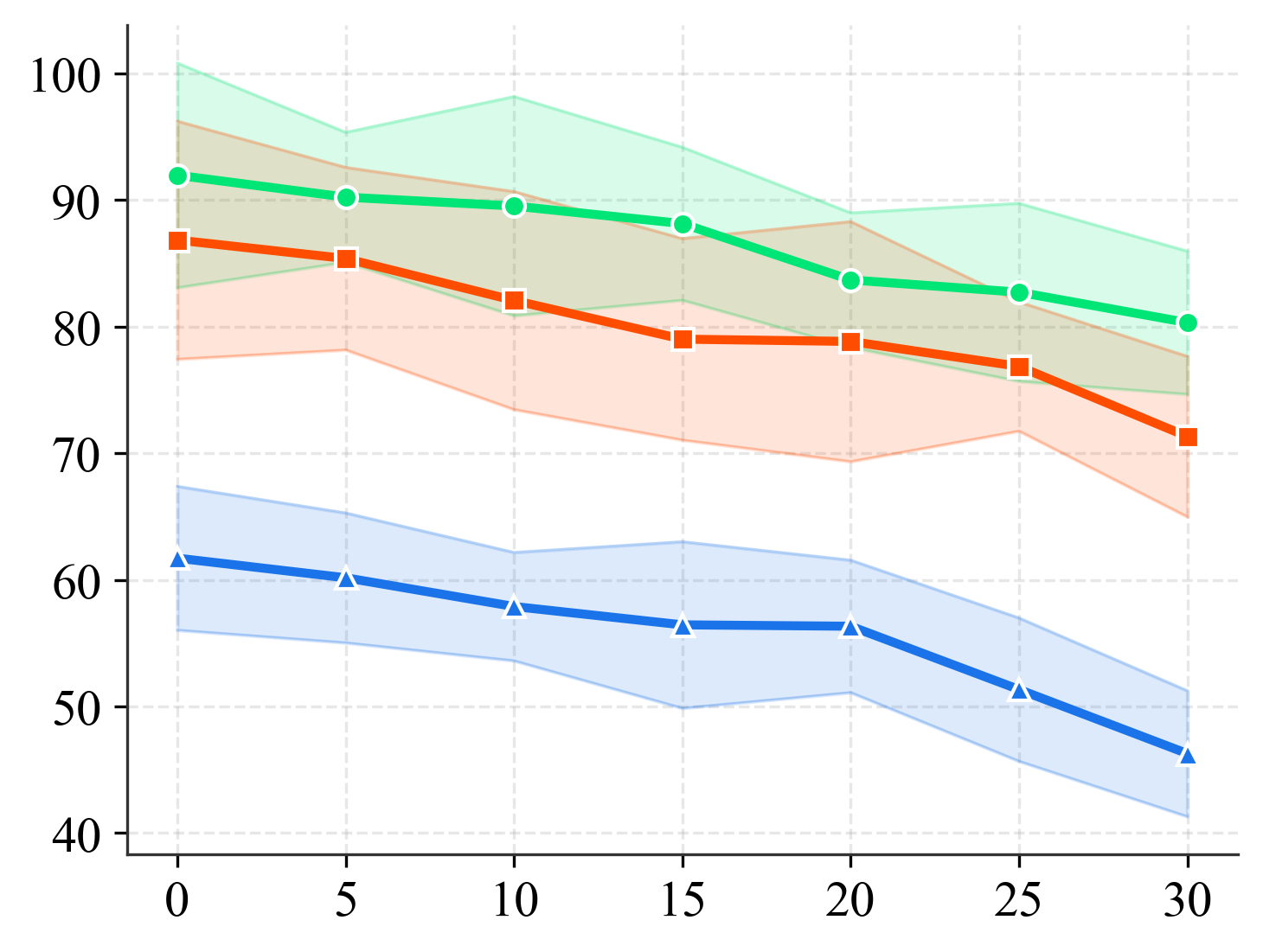} &
\cellcolor{walkgreen}\includegraphics[width=\linewidth]{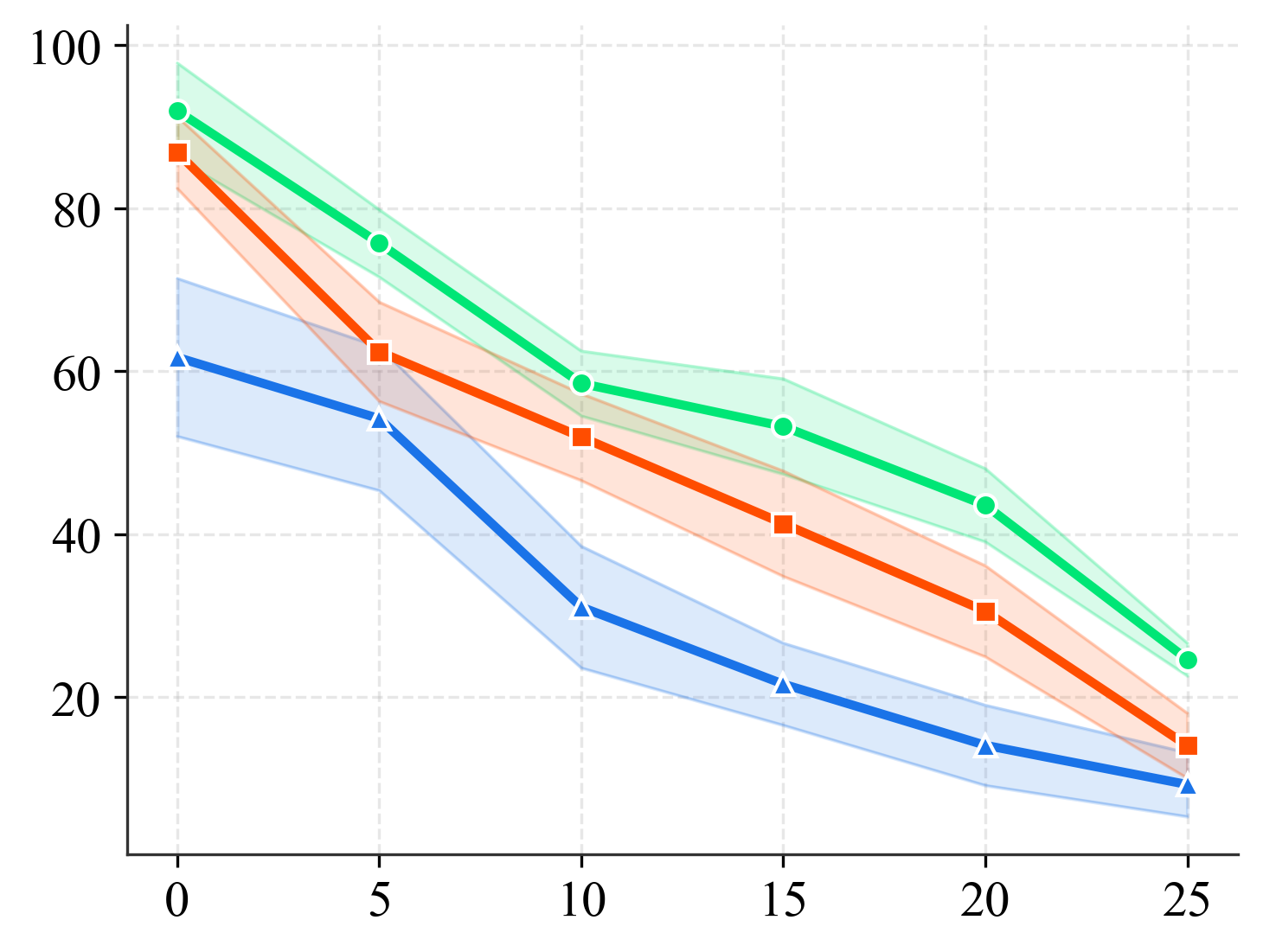} \\[1pt]

\cellcolor{fliporange}\rotatebox{90}{\small\textbf{Walk}} &
\cellcolor{fliporange}\includegraphics[width=\linewidth]{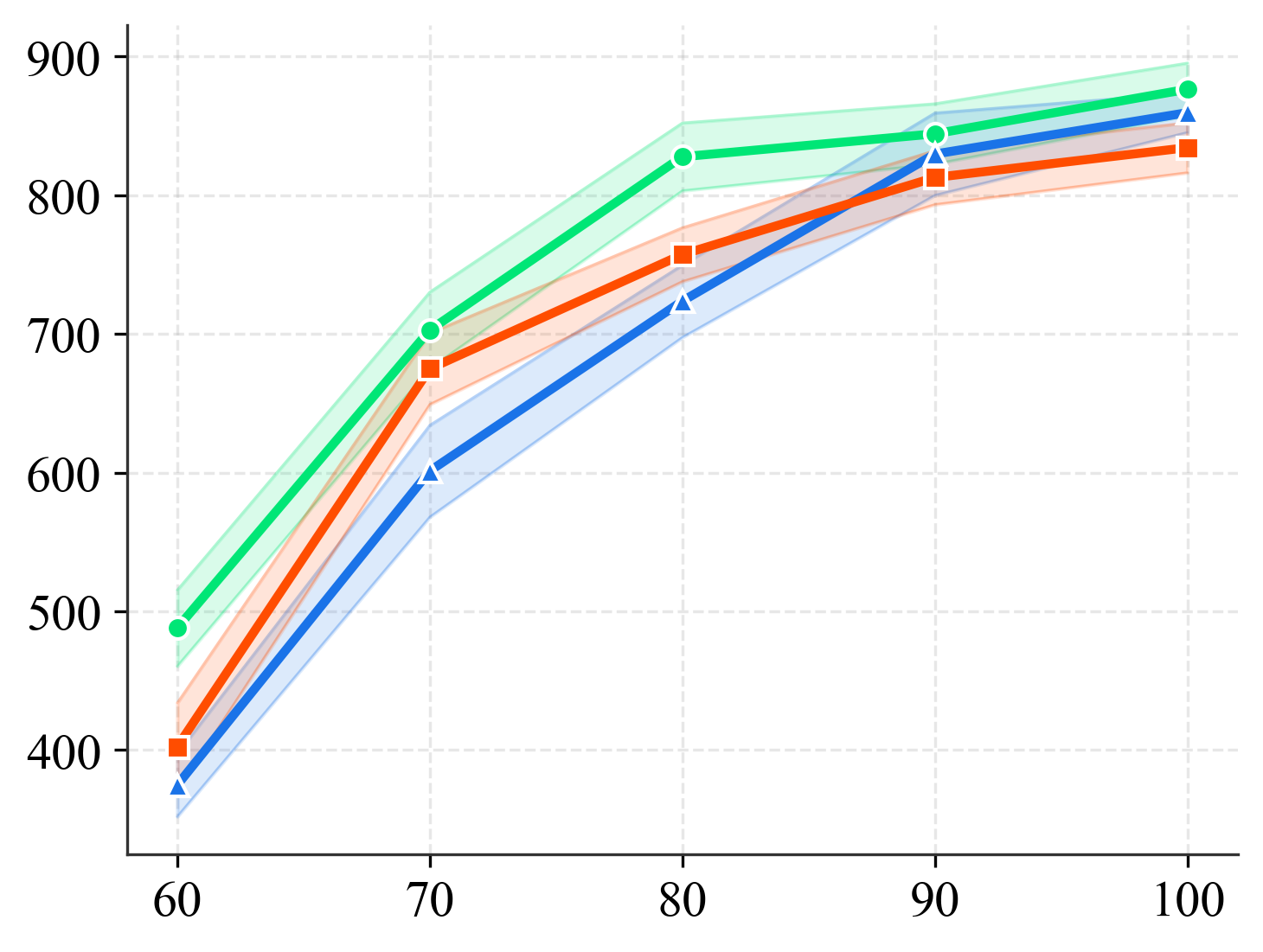} &
\cellcolor{fliporange}\includegraphics[width=\linewidth]{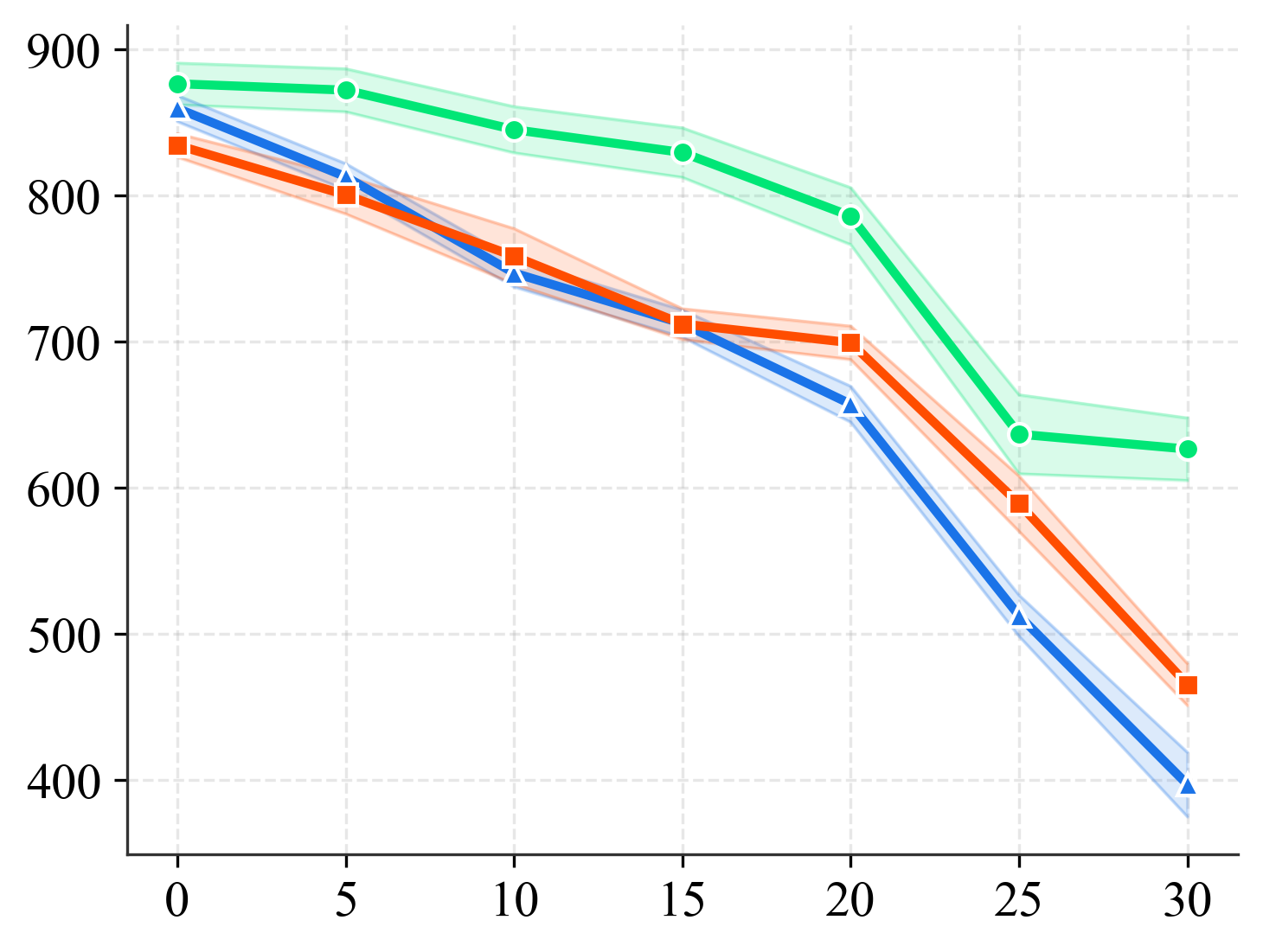} &
\cellcolor{fliporange}\includegraphics[width=\linewidth]{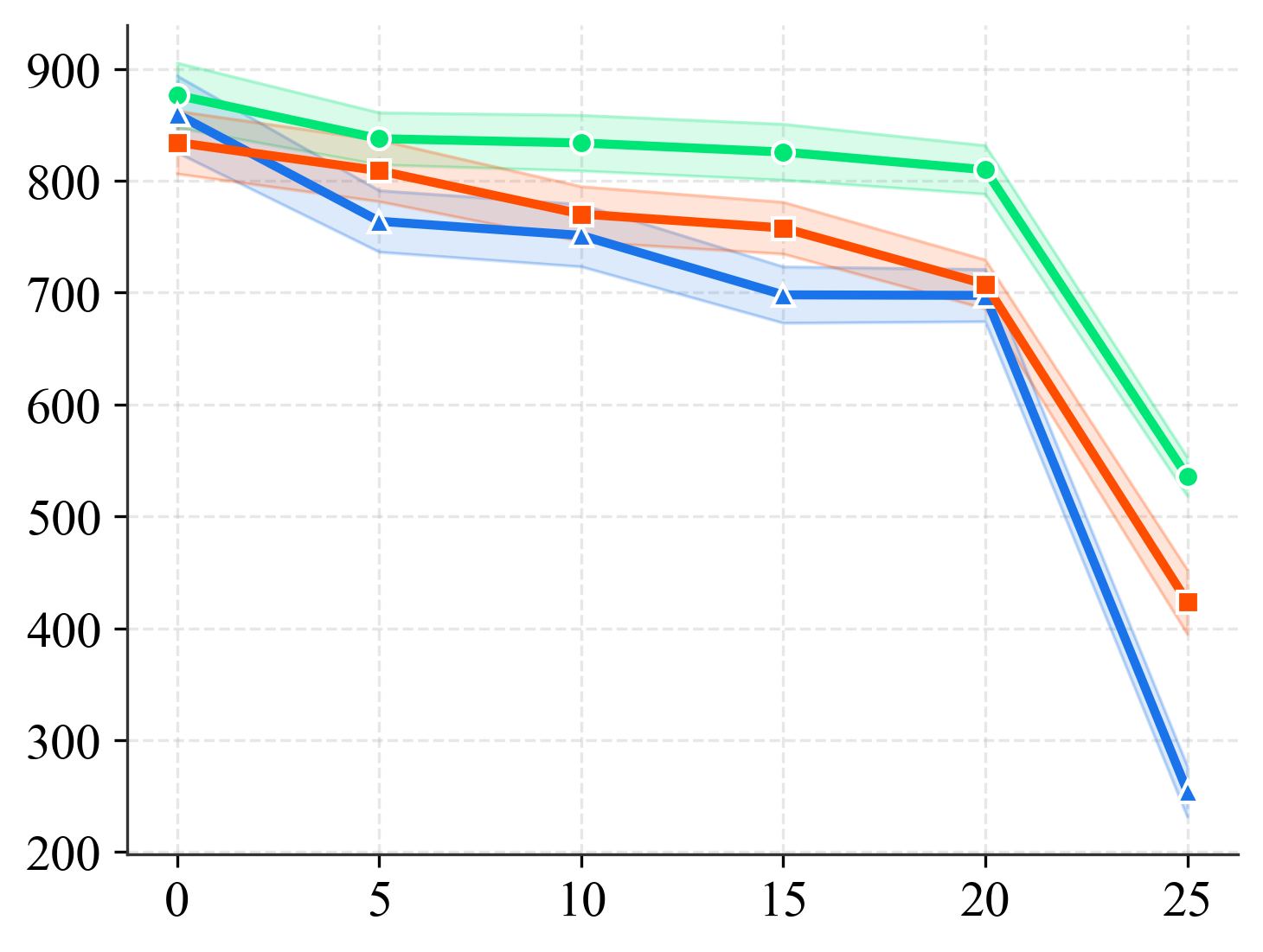} \\[1pt]

\cellcolor{standpurple}\rotatebox{90}{\small\textbf{Walk Backward}} &
\cellcolor{standpurple}\includegraphics[width=\linewidth]{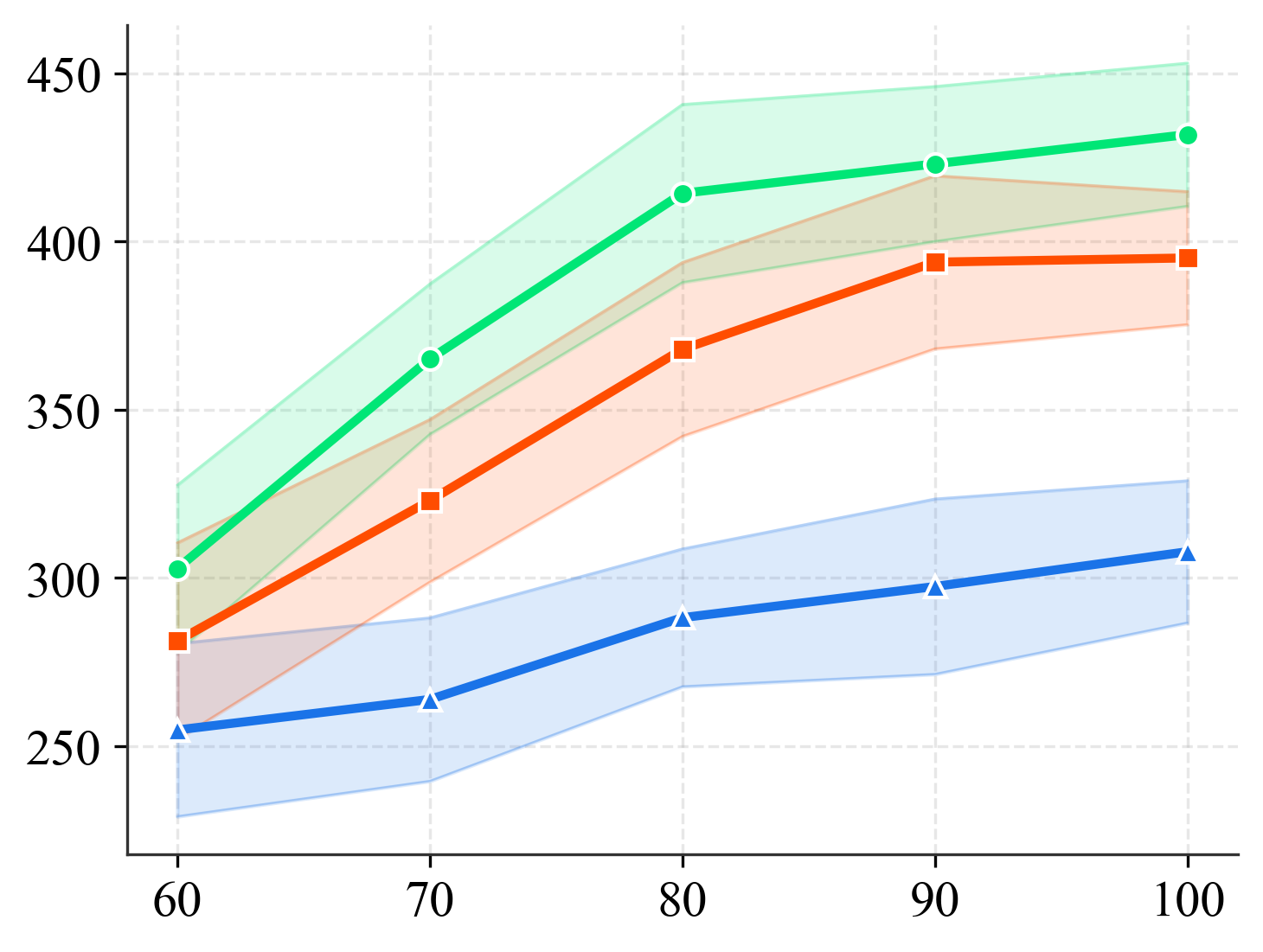} &
\cellcolor{standpurple}\includegraphics[width=\linewidth]{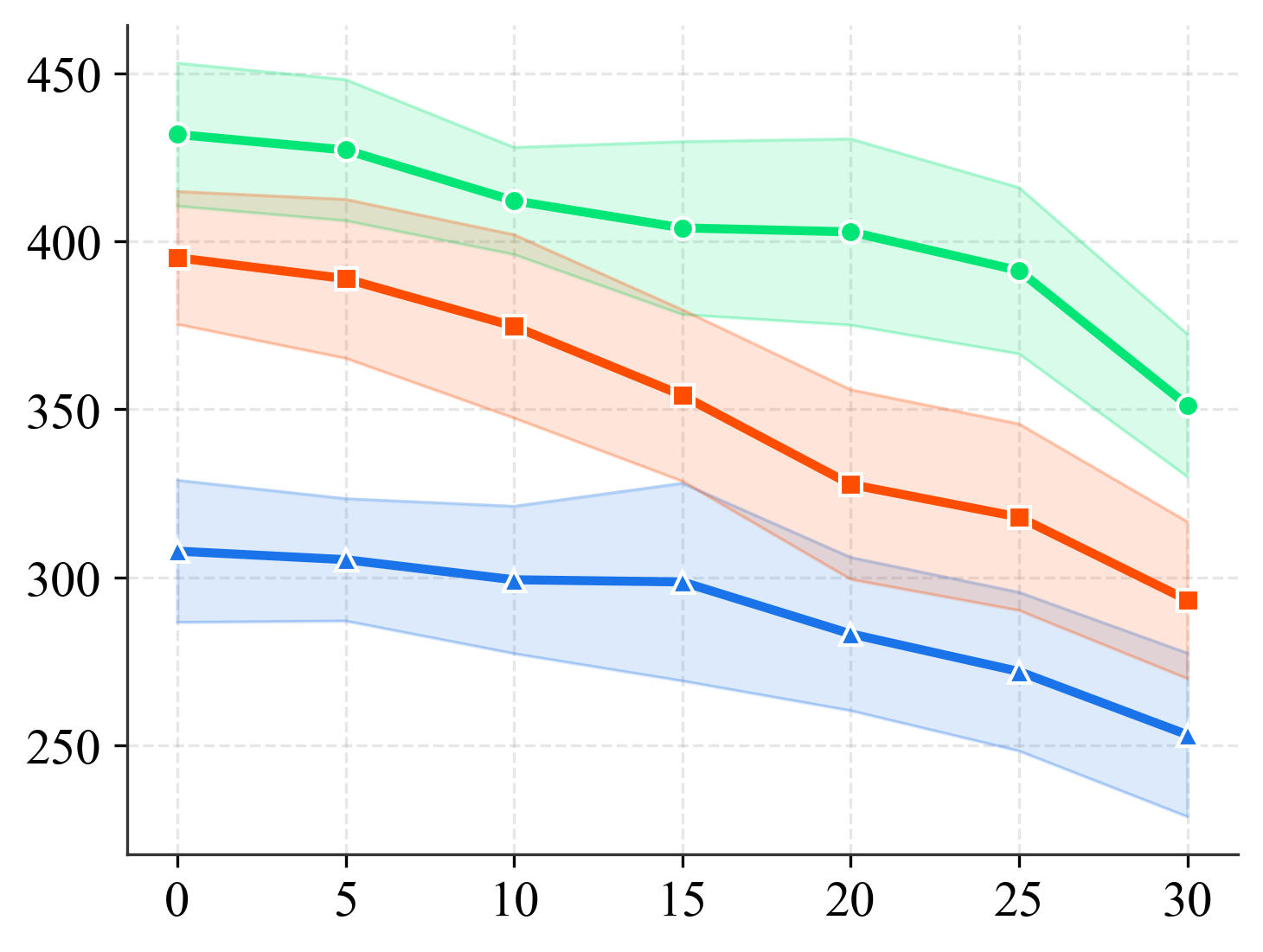} &
\cellcolor{standpurple}\includegraphics[width=\linewidth]{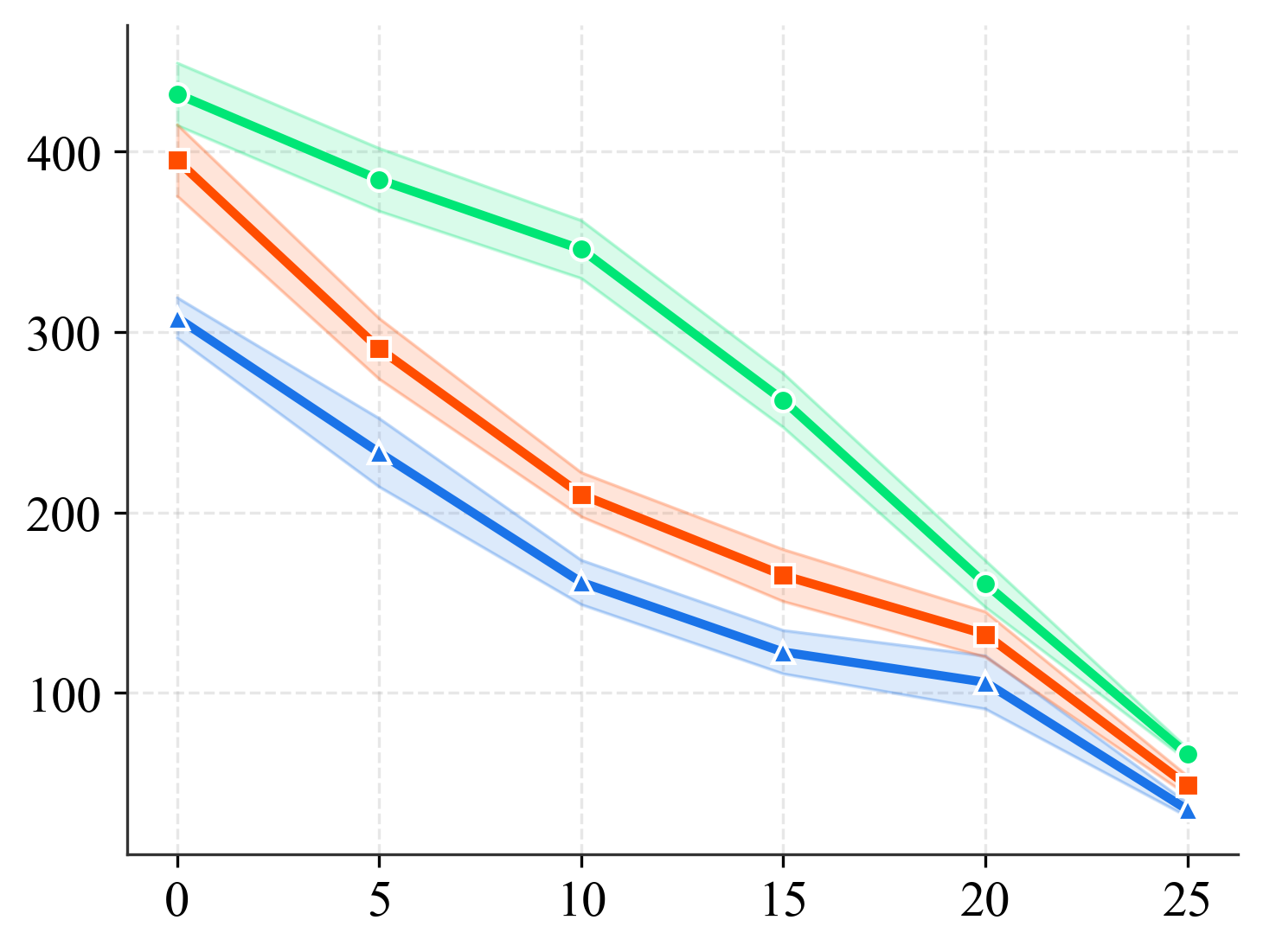} \\

\end{tabular}

\vspace{0.2cm}
\begin{minipage}{0.95\textwidth}
\centering
\footnotesize
\legendcircle{rbfmheavy}{RBFM-Heavy}\hspace{3.0em}
\legendtriangle{fbil}{FB-IL}\hspace{3.0em}
\legendsquare{rbfmlight}{RBFM-Light}
\end{minipage}

\caption{Cheetah performance with $95\%$ confidence interval across four tasks (rows) and three perturbation types (columns): range of motion restricts how far each joint can move, simulating mechanical damage (\% of nominal retained); actuator strength reduces peak joint torque, simulating motor degradation (\% reduction); joint friction loss adds passive resistive torque at each joint, simulating mechanical wear (absolute Nm per joint). Pretrained on RND data.}
\label{fig:cheetah}
\end{figure}

\appsubsection{Further discussion of (Q1)} \label{subsec:app_q1_discussion}
\noindent \textbf{Walker.} On the walker, Figure \ref{fig:walker} shows that the robustness advantage of RBFM variants over FB-IL is most pronounced under gravity and body mass perturbations, with the performance gap reaching its maximum on the flip task, a dynamically demanding task requiring full-body coordinated rotation where perturbation-induced transition distribution shifts are largest in magnitude. Under joint friction loss, all three methods degrade monotonically with increasing resistive torque, but RBFM-Heavy maintains the highest return throughout the sweep, with RBFM-Light consistently intermediate. Notably, FB-IL exhibits the highest variance at severe perturbation levels, indicating that its inferred $z$ becomes behaviorally unstable when the transition dynamics deviate significantly from the nominal expert distribution. The walker results most cleanly validate the core hypothesis: DRO-based task inference produces embeddings that remain behaviorally viable under transition distributions overlapping with the worst-case support identified during inference, whereas the FB-IL embedding is brittle outside the nominal support.

\noindent \textbf{Quadruped.} The quadruped presents a more challenging robustness benchmark due to its redundant four-leg morphology, which enables passive mechanical compensation for many single-parameter perturbations. Despite this, Figure \ref{fig:quadruped} shows that under tilted gravity, which introduces a persistent lateral torque component that directly penalizes the upright reward term and ground contact stiffness, which degrades push-off efficiency by softening contact impulse delivery; RBFM-Heavy maintains the highest return across all four tasks, followed by RBFM-Light, with FB-IL degrading most steeply. The advantage is most pronounced on dynamic tasks (run, jump) where perturbation-induced distributional shifts are larger, and relatively attenuated on static tasks (stand, walk) where the quadruped's passive stability partially compensates for dynamics mismatch. Under actuator ctrlrange clipping, the performance advantage of RBFM-Heavy is most visible at intermediate perturbation levels; beyond a critical clipping threshold all methods collapse, suggesting that DRO-based robustness operates within the recoverable dynamics regime but cannot compensate for command truncation that renders the task fundamentally infeasible under any task embedding.

\noindent \textbf{Cheetah.} On the cheetah, Figure \ref{fig:cheetah} shows that the three-way performance separation is most sharply differentiated across all four tasks and all three perturbation types. Under joint friction loss — which imposes a fixed passive resistive torque at every joint on every control cycle, compounding across the six actuated degrees of freedom, FB-IL collapses to near-zero return at moderate perturbation levels while RBFM-Heavy maintains substantial locomotion performance, with RBFM-Light intermediate. Under actuator strength reduction, the performance separation is consistent from the first perturbation level, indicating that even marginal reductions in gear ratio expose the brittleness of the FB-IL task embedding. Under range of motion restriction, the only kinematic perturbation across all experiments, constraining the feasible configuration space rather than the force generation or dissipation properties, RBFM-Heavy maintains the highest return throughout, while FB-IL exhibits notably higher variance, suggesting greater sensitivity to the hard feasibility constraints imposed by reduced joint limits. The consistently sharp three-way separation on the cheetah, relative to the walker and quadruped, likely reflects the high sensitivity of fast asymmetric locomotion to dynamics mismatch: the cheetah's gait involves large joint excursions and brief ground contact phases, concentrating the nominal transition distribution in a narrow dynamical regime that is more susceptible to perturbation-induced shifts falling outside the expert support.


\appsubsection{Q4: Different Sources of Pretraining Data} \label{subsec:app_different_pre_source}
Figures \ref{fig:walker_aps} and \ref{fig:walker_proto} compares the performance of RBFM-Heavy and RBFM-Light against FB-IL when pretrained on Walker-APS and Walker-Proto dataset respectively. We observe similar trend as that when Walker-RND was used for pre-training. We observe RBFM-Heavy to have the strongest robust performance, followed by RBFM-Light, and finally FB-IL.

\begin{figure}[!htbp]
\centering
\setlength{\tabcolsep}{4pt}
\renewcommand{\arraystretch}{0.0}

\begin{tabular}{>{\centering\arraybackslash}m{0.5cm}
                >{\centering\arraybackslash}m{0.30\textwidth}
                >{\centering\arraybackslash}m{0.30\textwidth}
                >{\centering\arraybackslash}m{0.30\textwidth}}

&
\textbf{\hspace{0.2cm}Gravity} &
\textbf{\hspace{0.2cm}Mass} &
\textbf{\hspace{0.4cm}Joint Friction Loss} \\[2pt]

\cellcolor{runblue}\rotatebox{90}{\small\textbf{Run}} &
\cellcolor{runblue}\includegraphics[width=\linewidth]{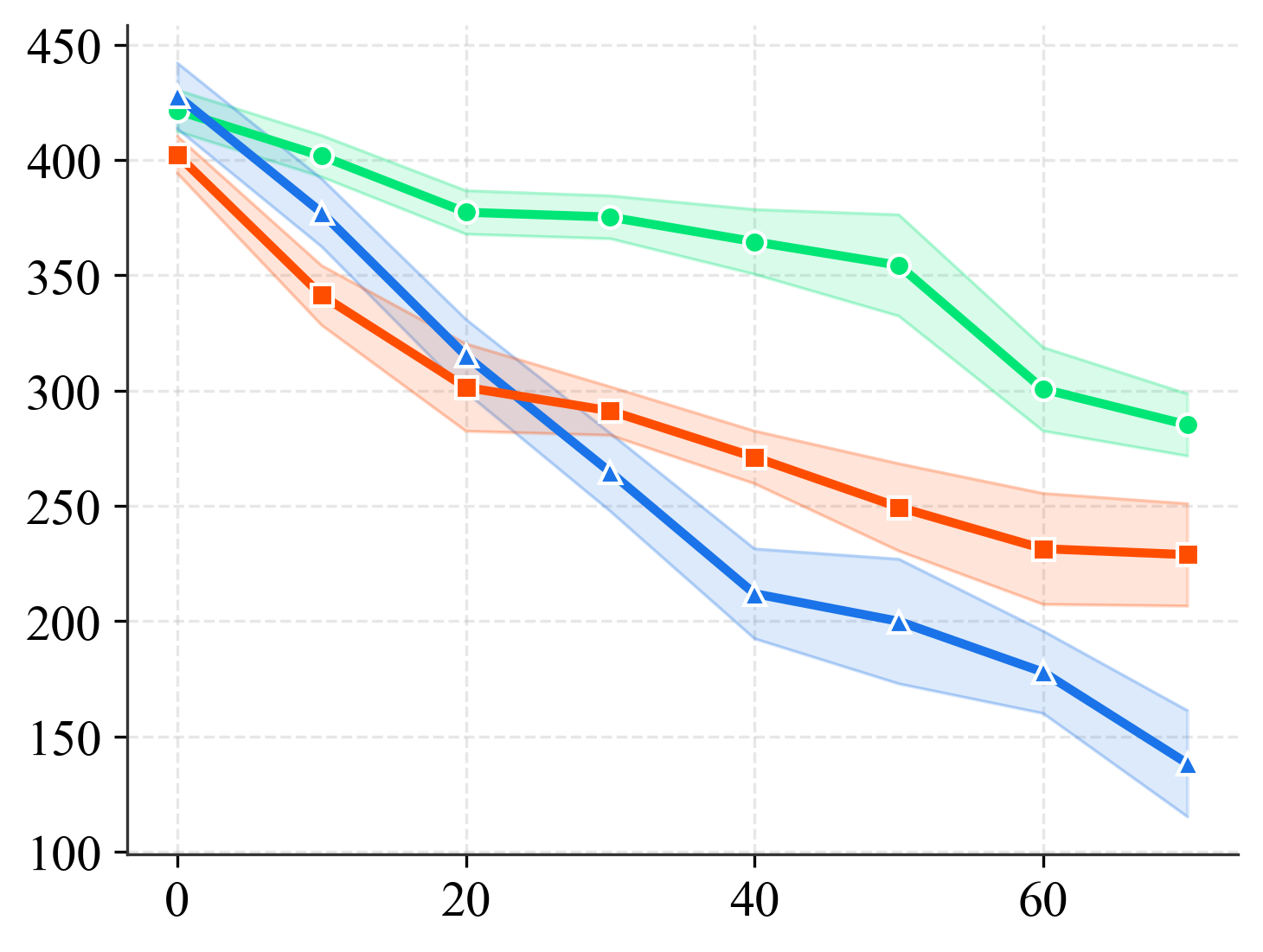} &
\cellcolor{runblue}\includegraphics[width=\linewidth]{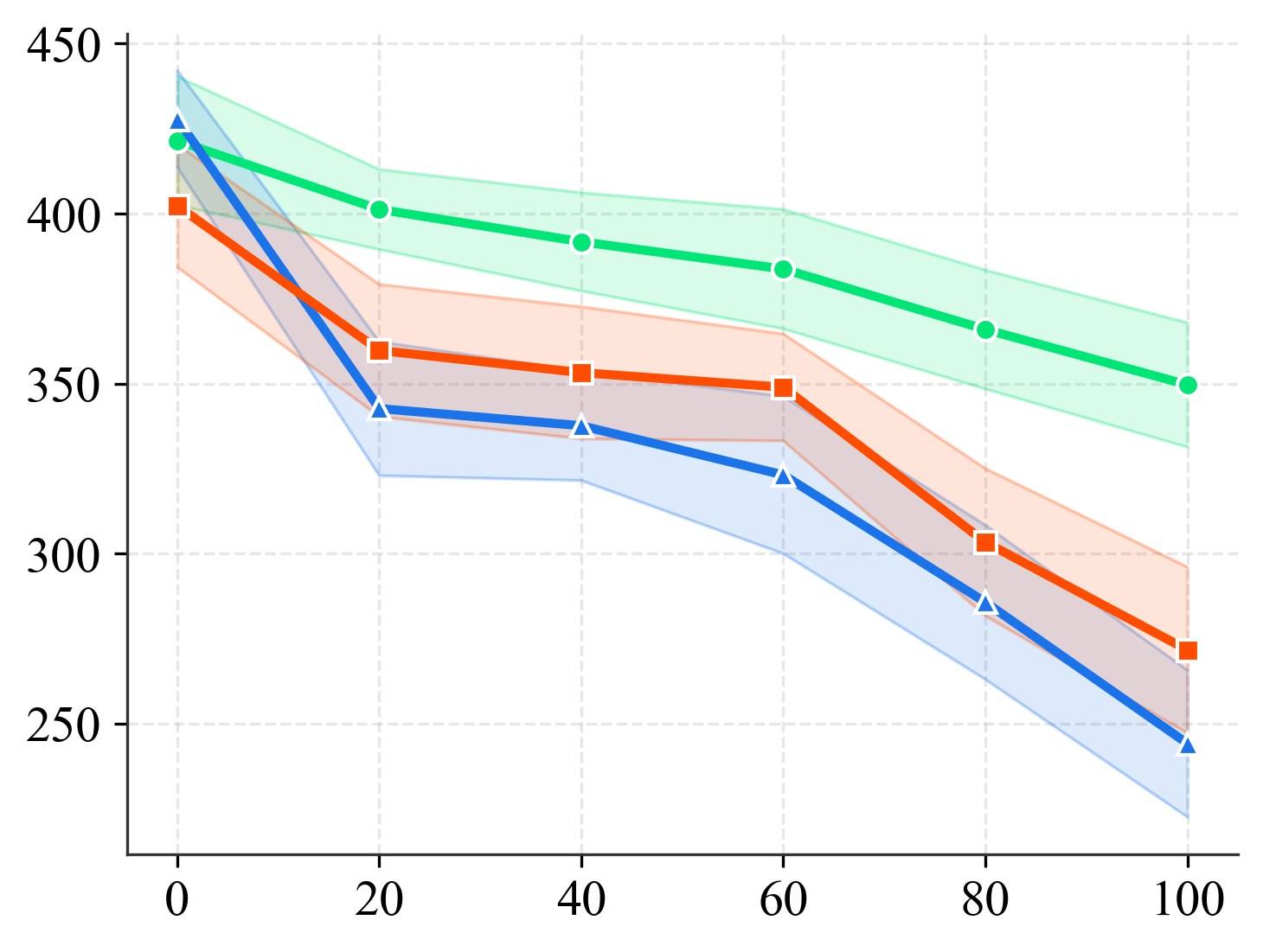} &
\cellcolor{runblue}\includegraphics[width=\linewidth]{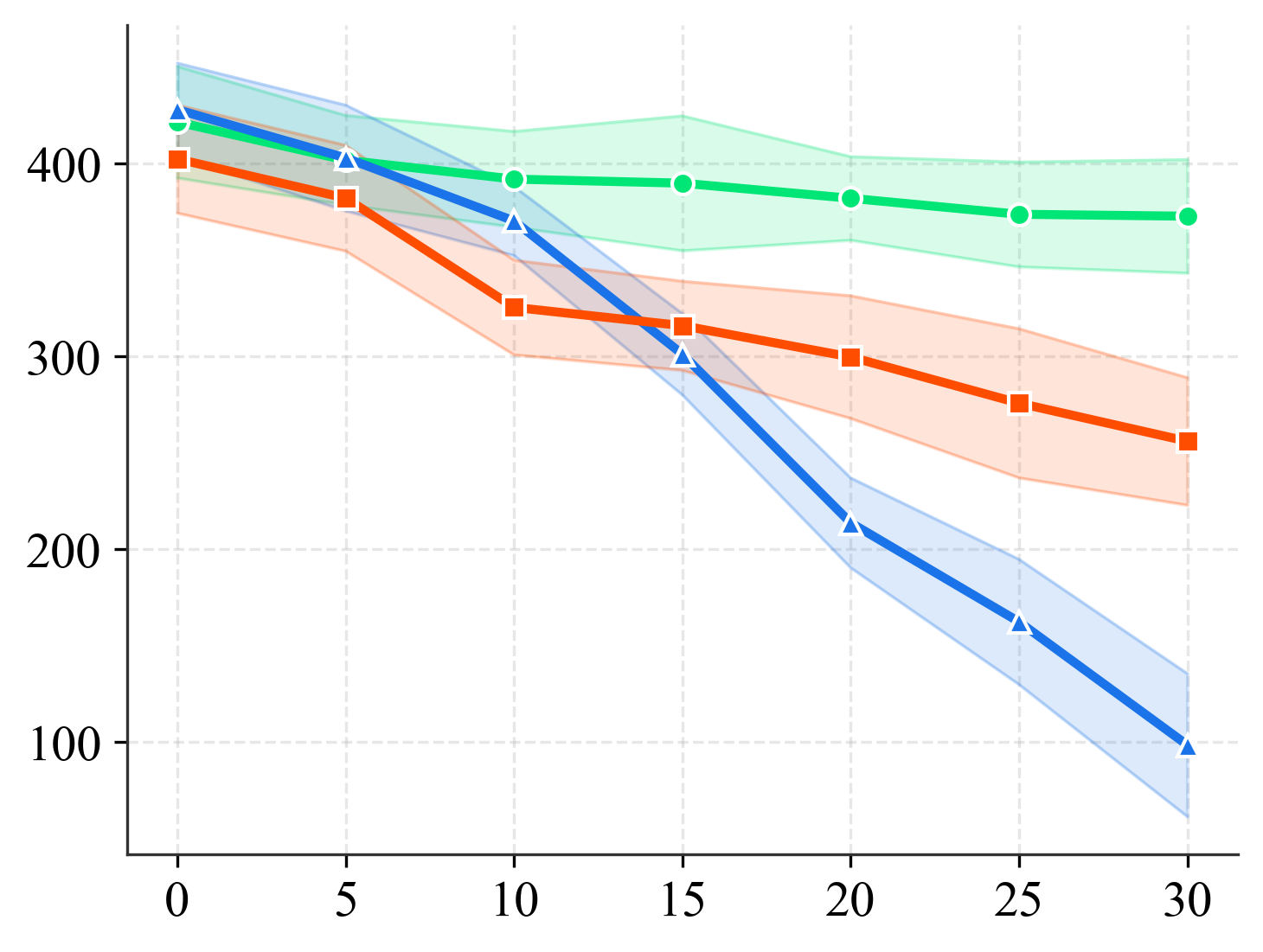} \\[1pt]

\cellcolor{walkgreen}\rotatebox{90}{\small\textbf{Walk}} &
\cellcolor{walkgreen}\includegraphics[width=\linewidth]{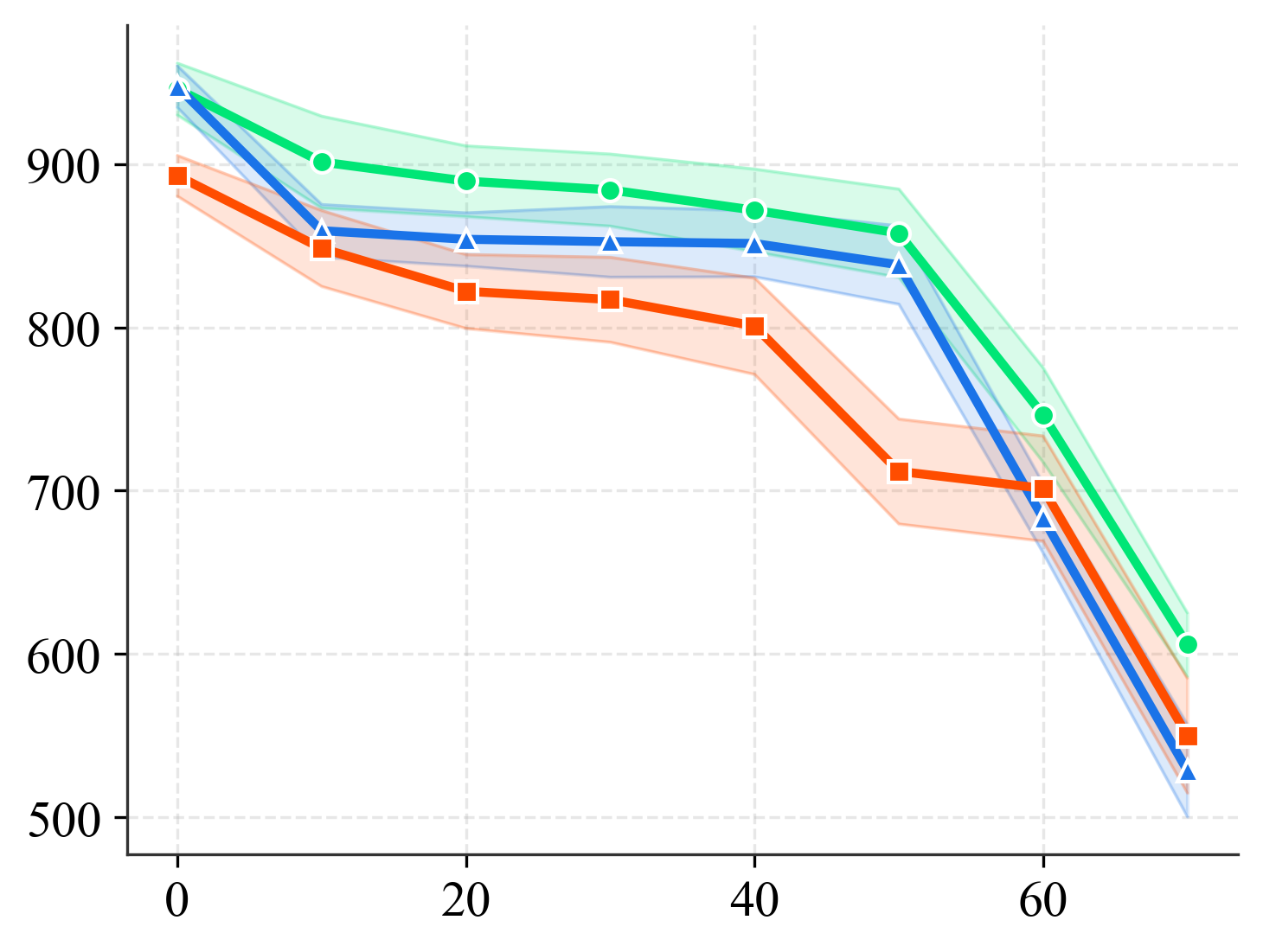} &
\cellcolor{walkgreen}\includegraphics[width=\linewidth]{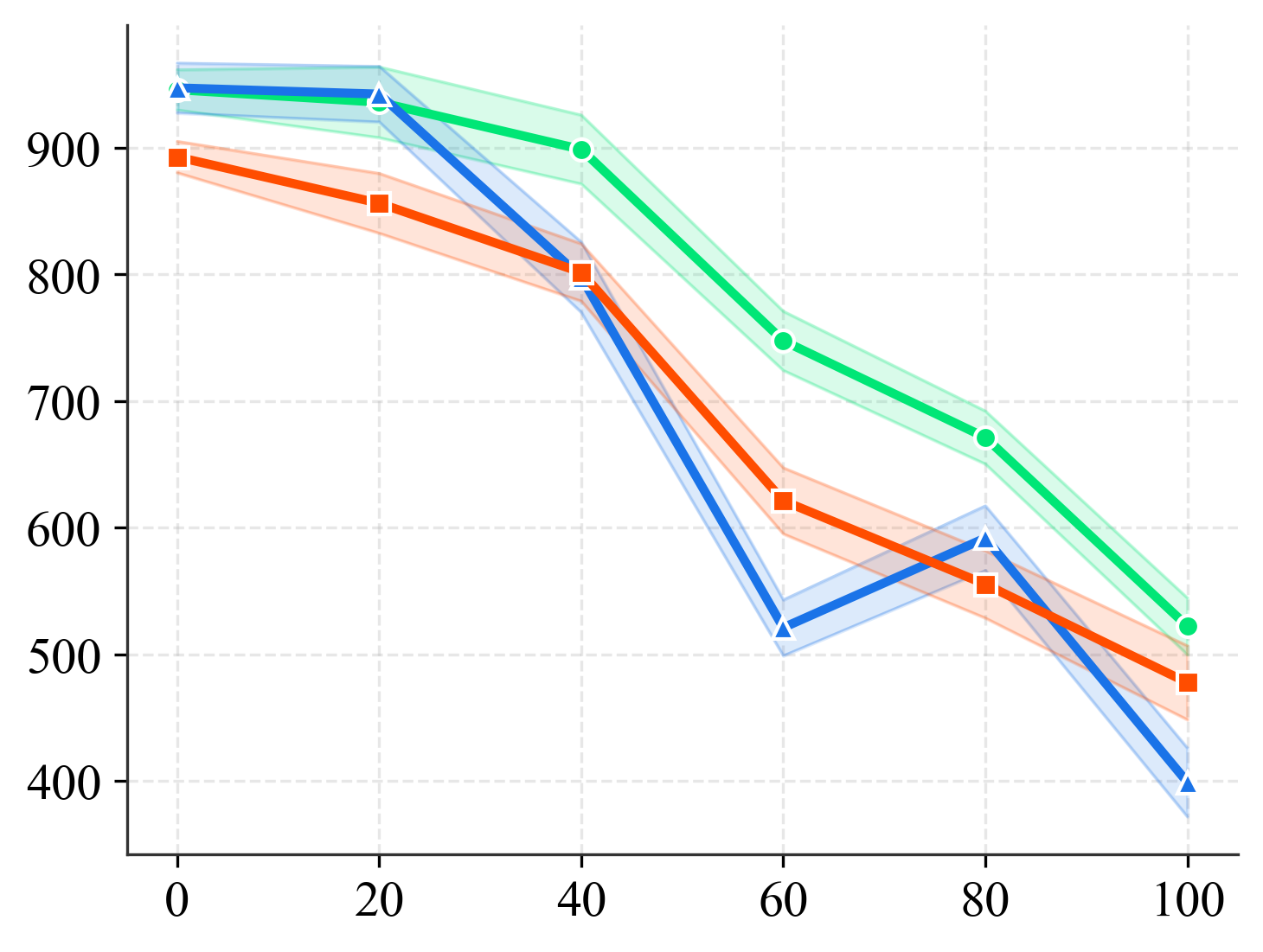} &
\cellcolor{walkgreen}\includegraphics[width=\linewidth]{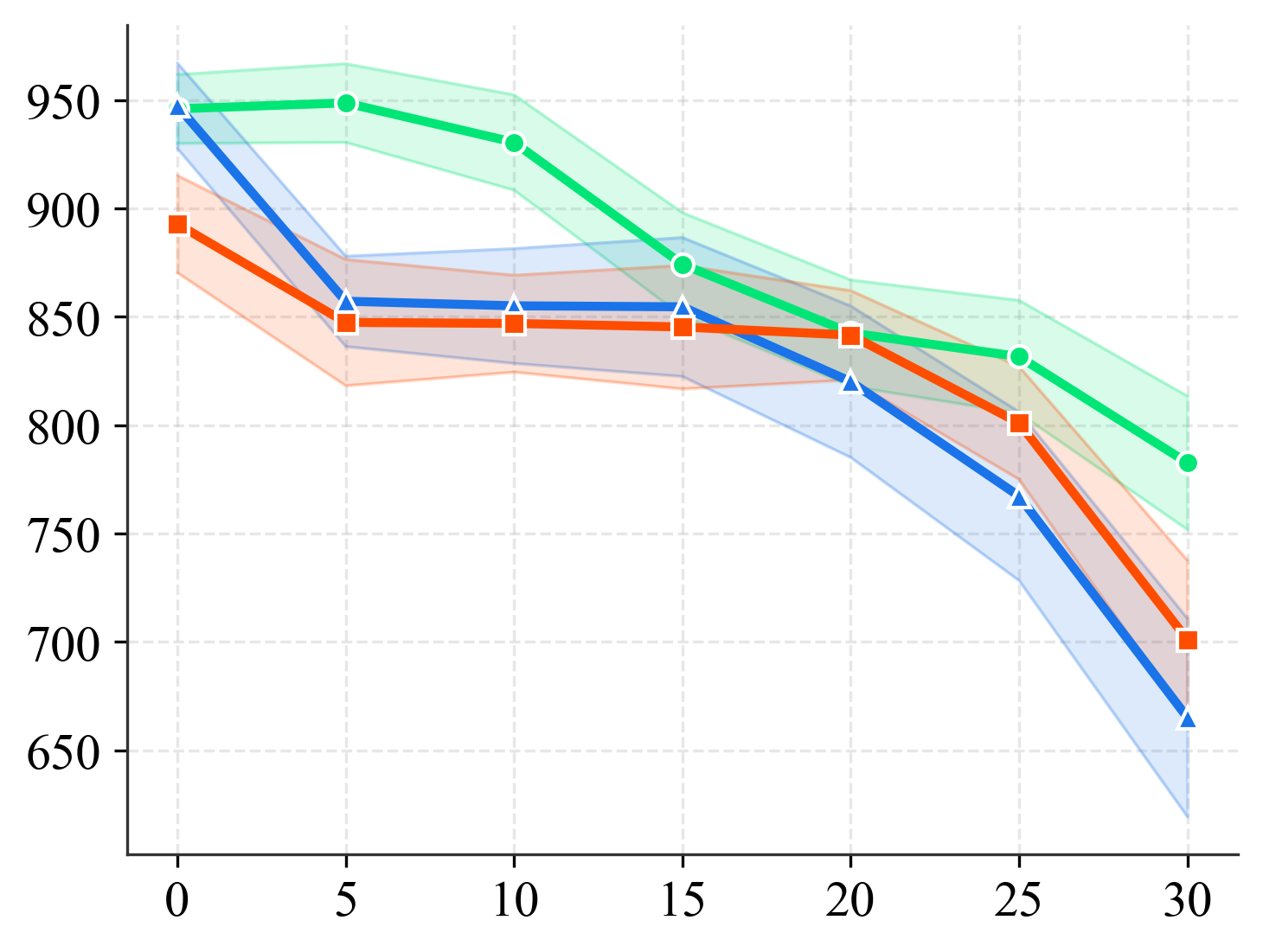} \\[1pt]

\cellcolor{fliporange}\rotatebox{90}{\small\textbf{Flip}} &
\cellcolor{fliporange}\includegraphics[width=\linewidth]{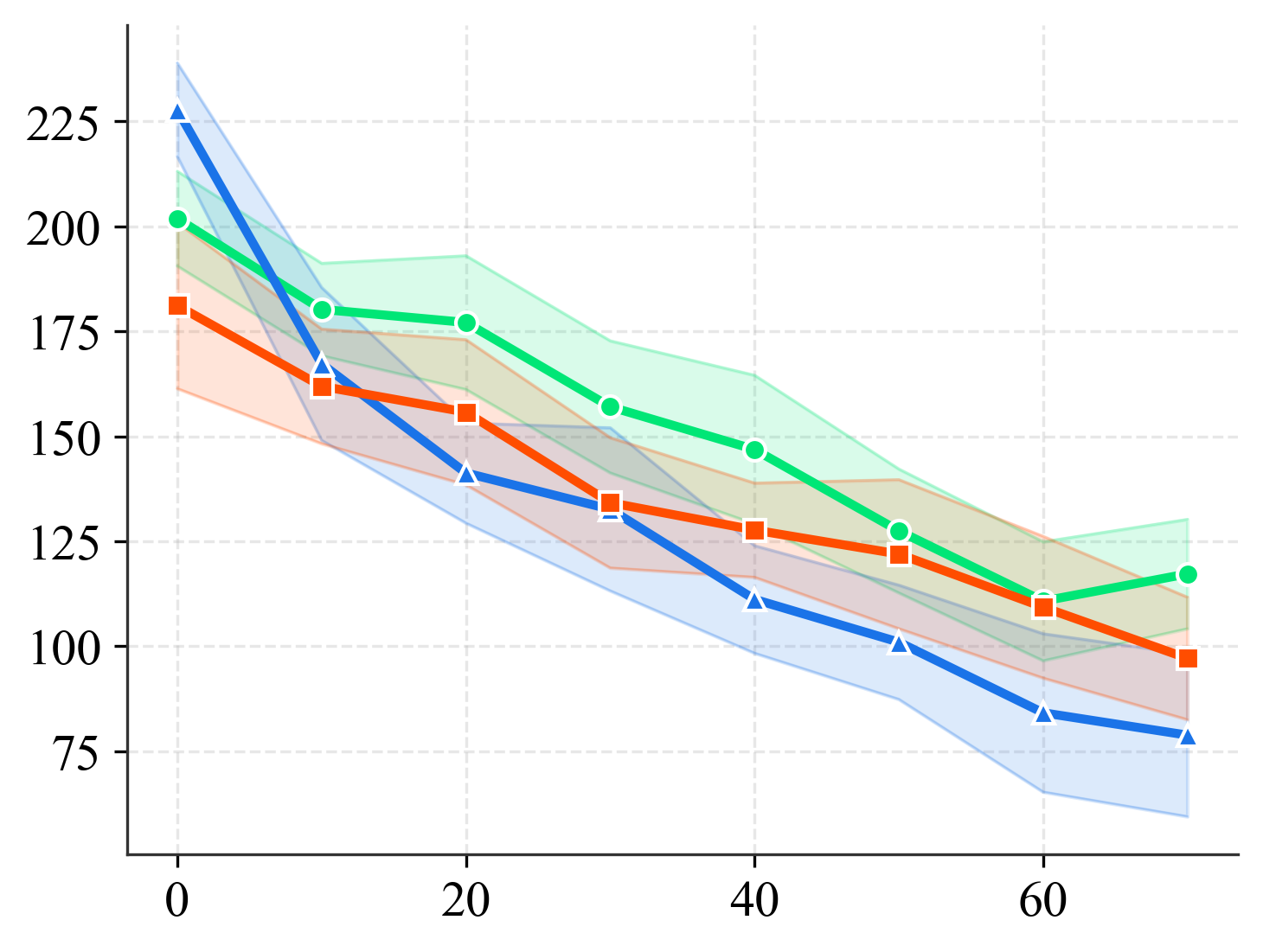} &
\cellcolor{fliporange}\includegraphics[width=\linewidth]{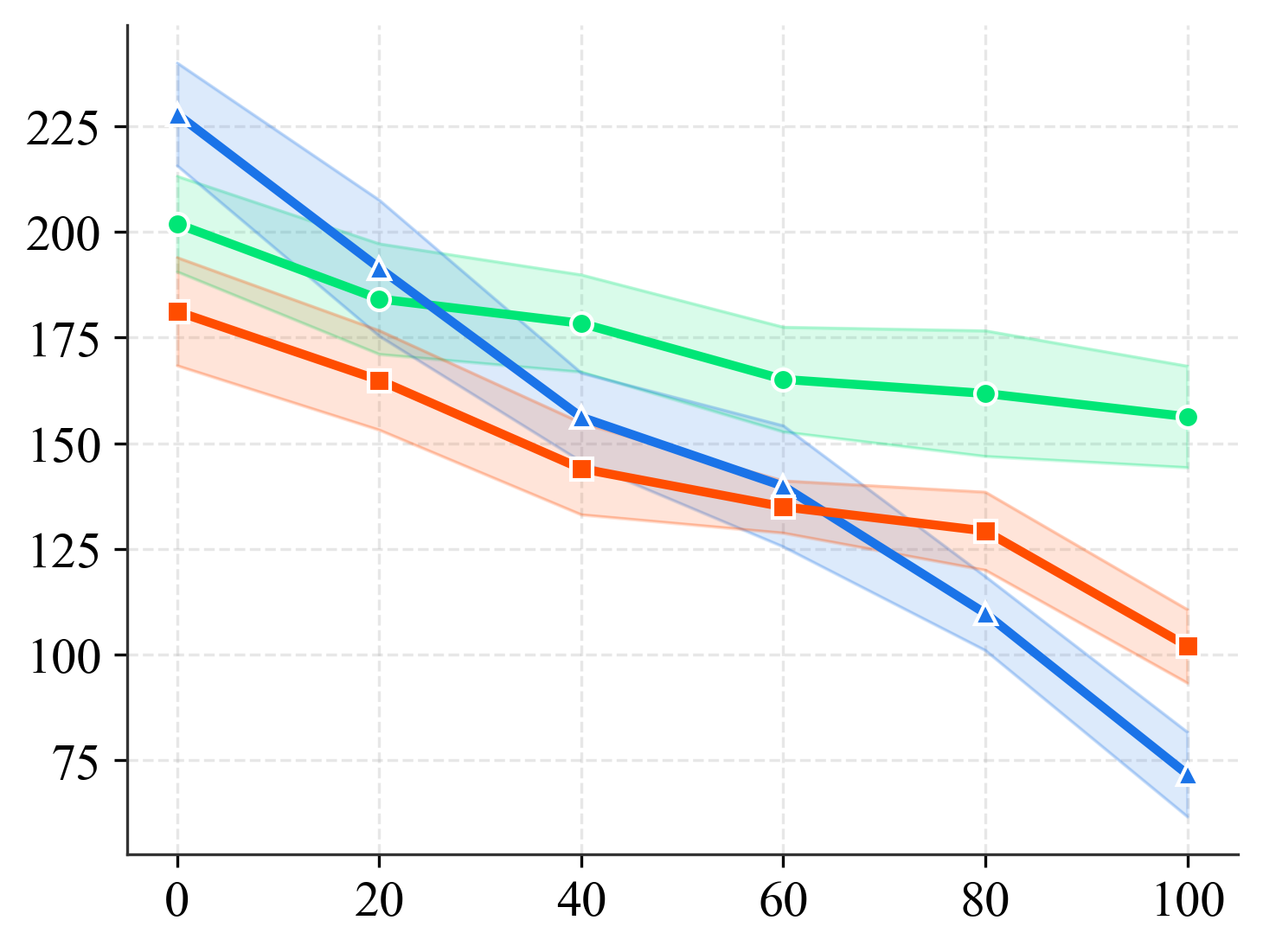} &
\cellcolor{fliporange}\includegraphics[width=\linewidth]{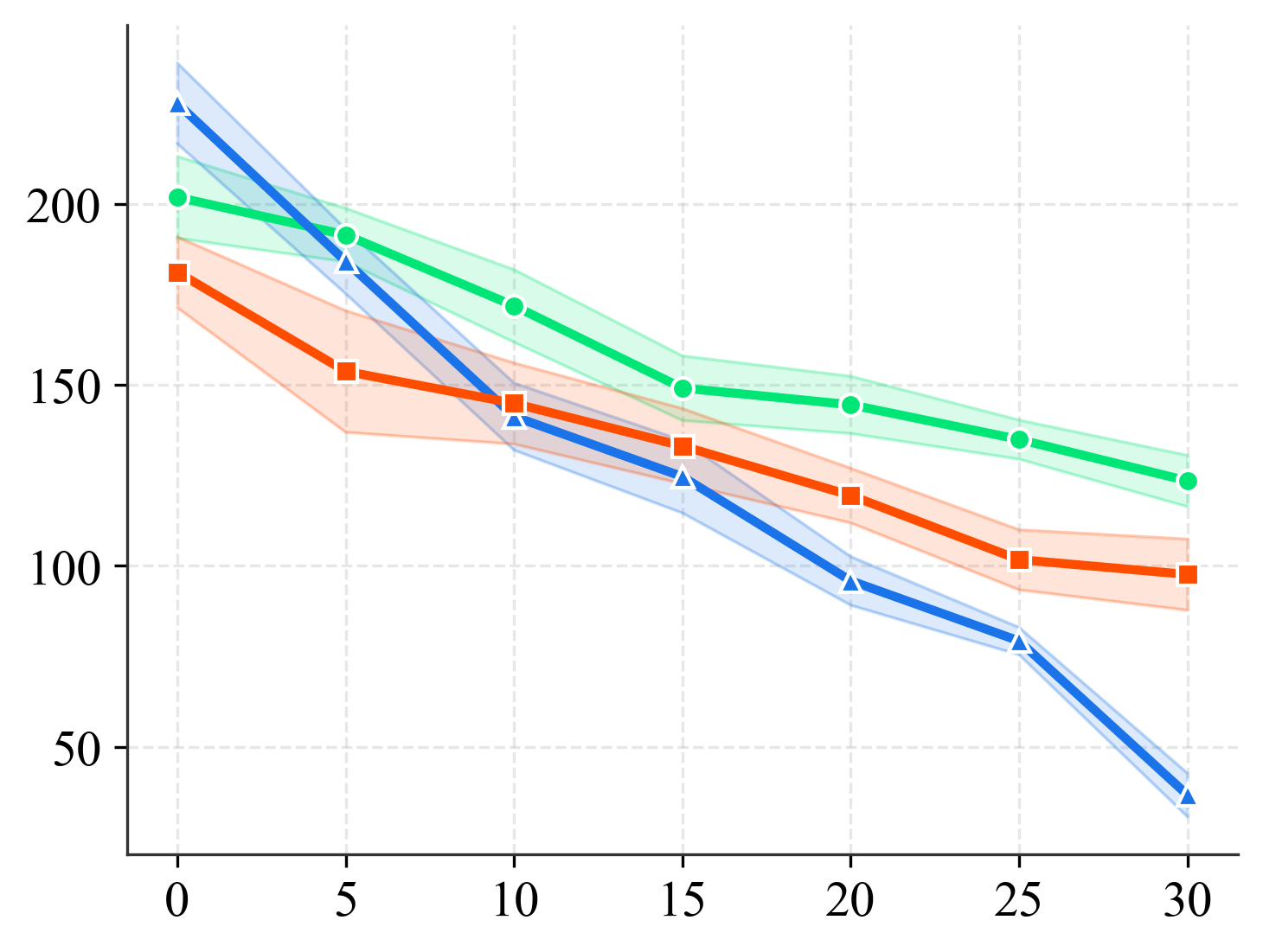} \\[1pt]

\cellcolor{standpurple}\rotatebox{90}{\small\textbf{Stand}} &
\cellcolor{standpurple}\includegraphics[width=\linewidth]{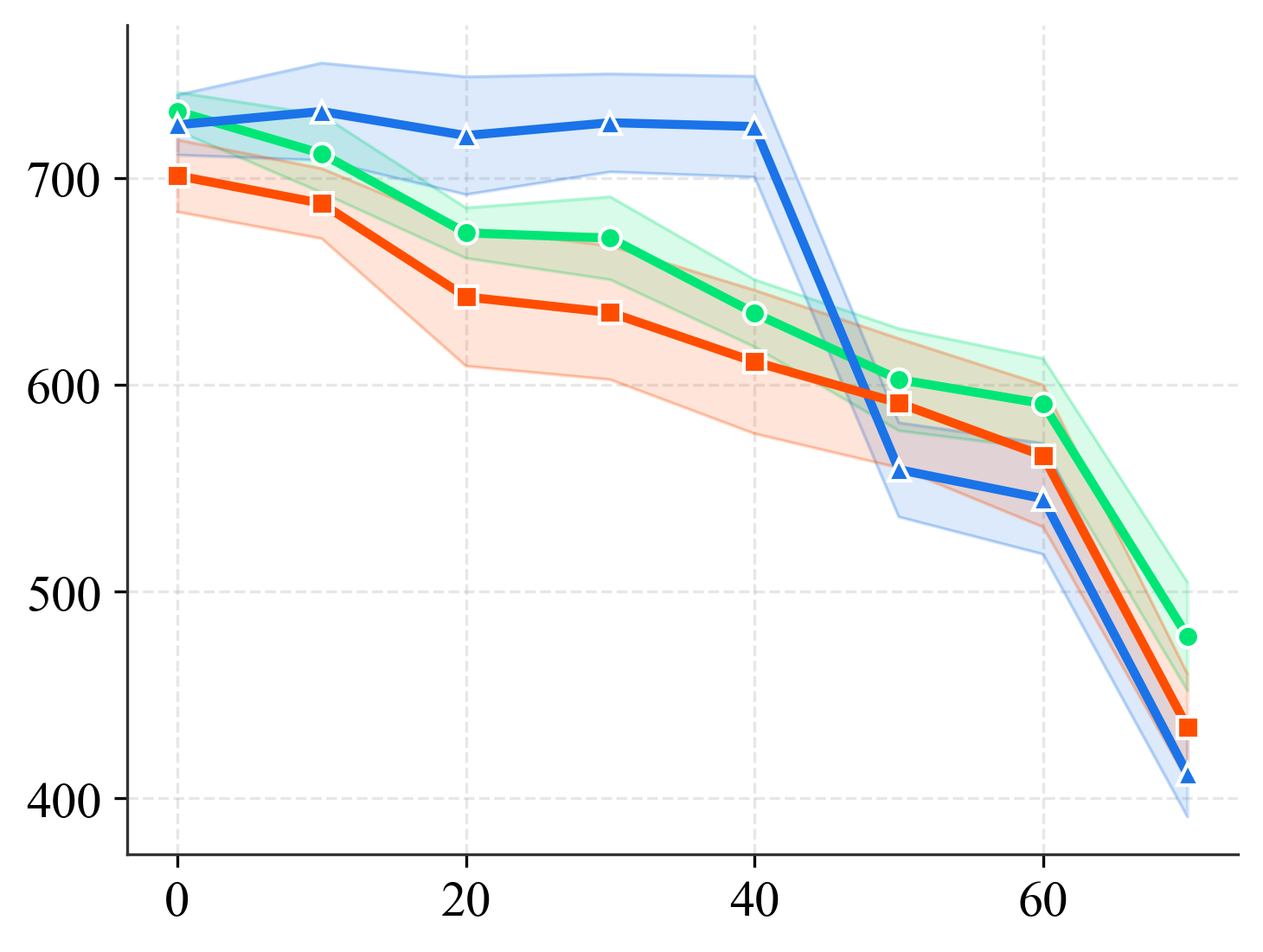} &
\cellcolor{standpurple}\includegraphics[width=\linewidth]{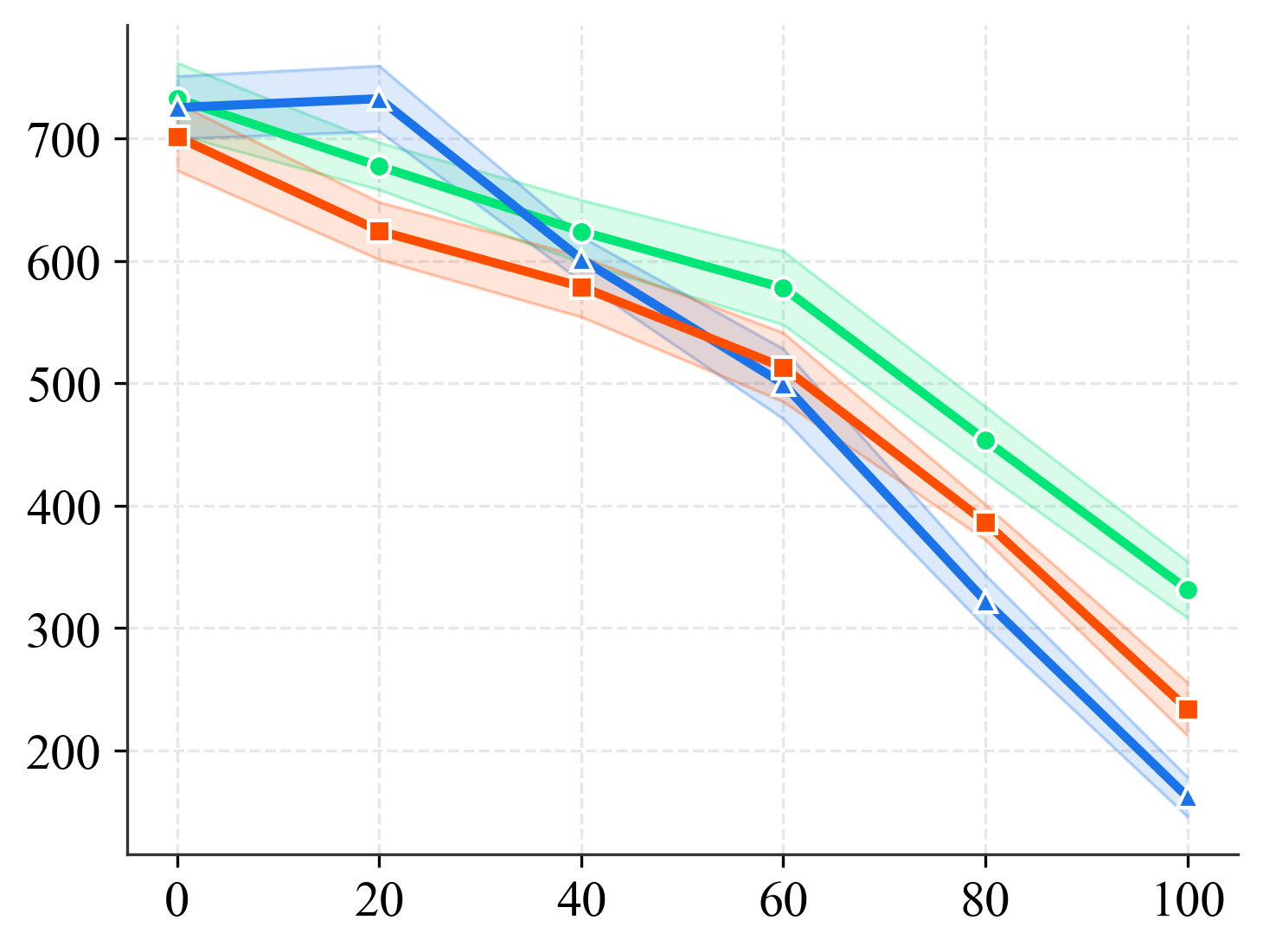} &
\cellcolor{standpurple}\includegraphics[width=\linewidth]{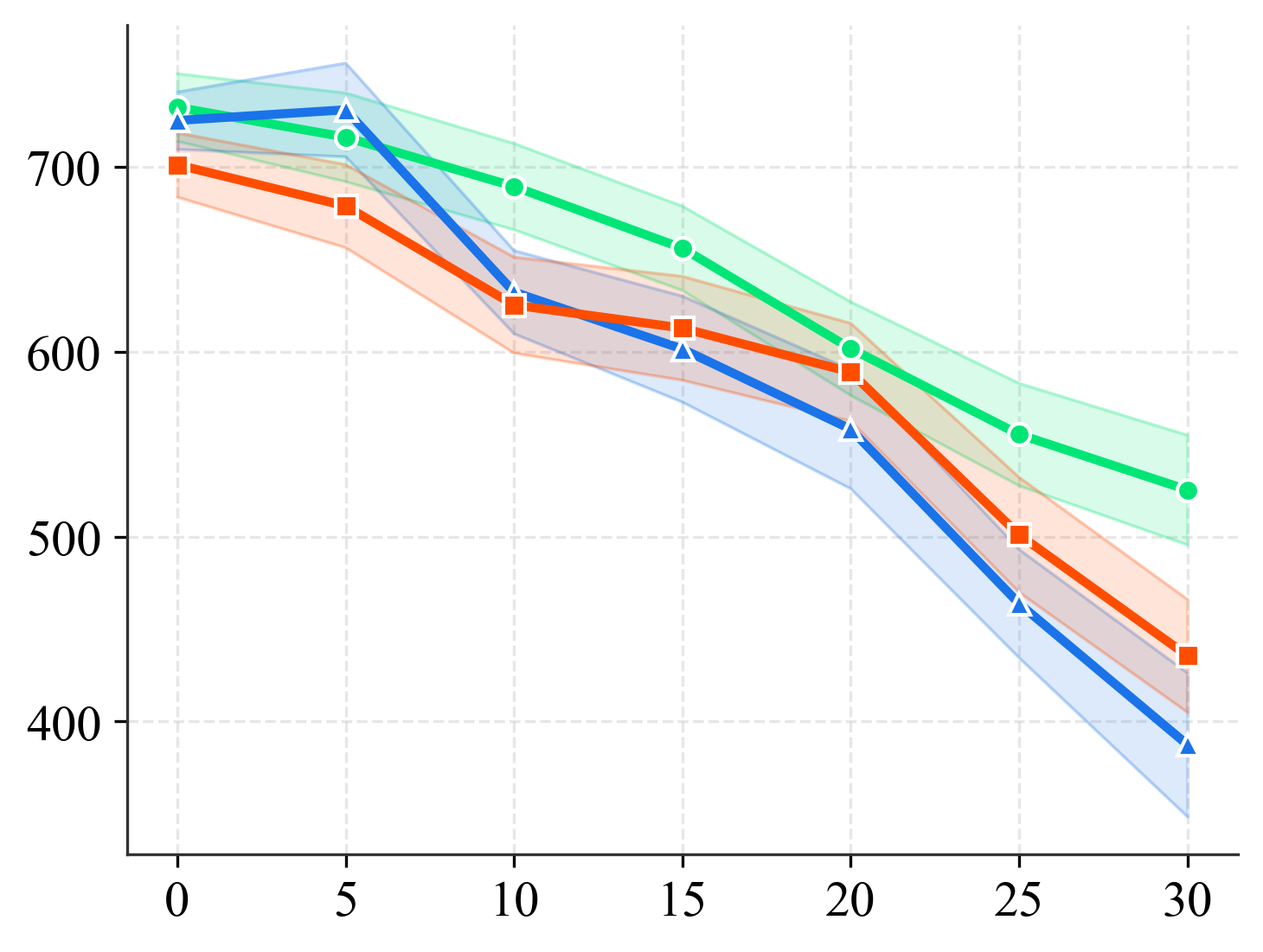} \\

\end{tabular}

\vspace{0.2cm}
\begin{minipage}{0.95\textwidth}
\centering
\footnotesize
\legendcircle{rbfmheavy}{RBFM-Heavy}\hspace{3.0em}
\legendtriangle{fbil}{FB-IL}\hspace{3.0em}
\legendsquare{rbfmlight}{RBFM-Light}
\end{minipage}

\caption{Walker performance with $95\%$ confidence interval across four tasks (rows) and three perturbation types (columns): gravity and body mass increase the physical load on the robot (\% change from nominal); joint friction loss adds passive resistive torque at each joint, simulating mechanical wear (absolute Nm per joint). Pretrained on APS data.}
\label{fig:walker_aps}
\end{figure}

\begin{figure}[!htbp]
\centering
\setlength{\tabcolsep}{4pt}
\renewcommand{\arraystretch}{0.0}

\begin{tabular}{>{\centering\arraybackslash}m{0.5cm}
                >{\centering\arraybackslash}m{0.30\textwidth}
                >{\centering\arraybackslash}m{0.30\textwidth}
                >{\centering\arraybackslash}m{0.30\textwidth}}

&
\textbf{\hspace{0.2cm}Gravity} &
\textbf{\hspace{0.2cm}Mass} &
\textbf{\hspace{0.4cm}Joint Friction Loss} \\[2pt]

\cellcolor{runblue}\rotatebox{90}{\small\textbf{Run}} &
\cellcolor{runblue}\includegraphics[width=\linewidth]{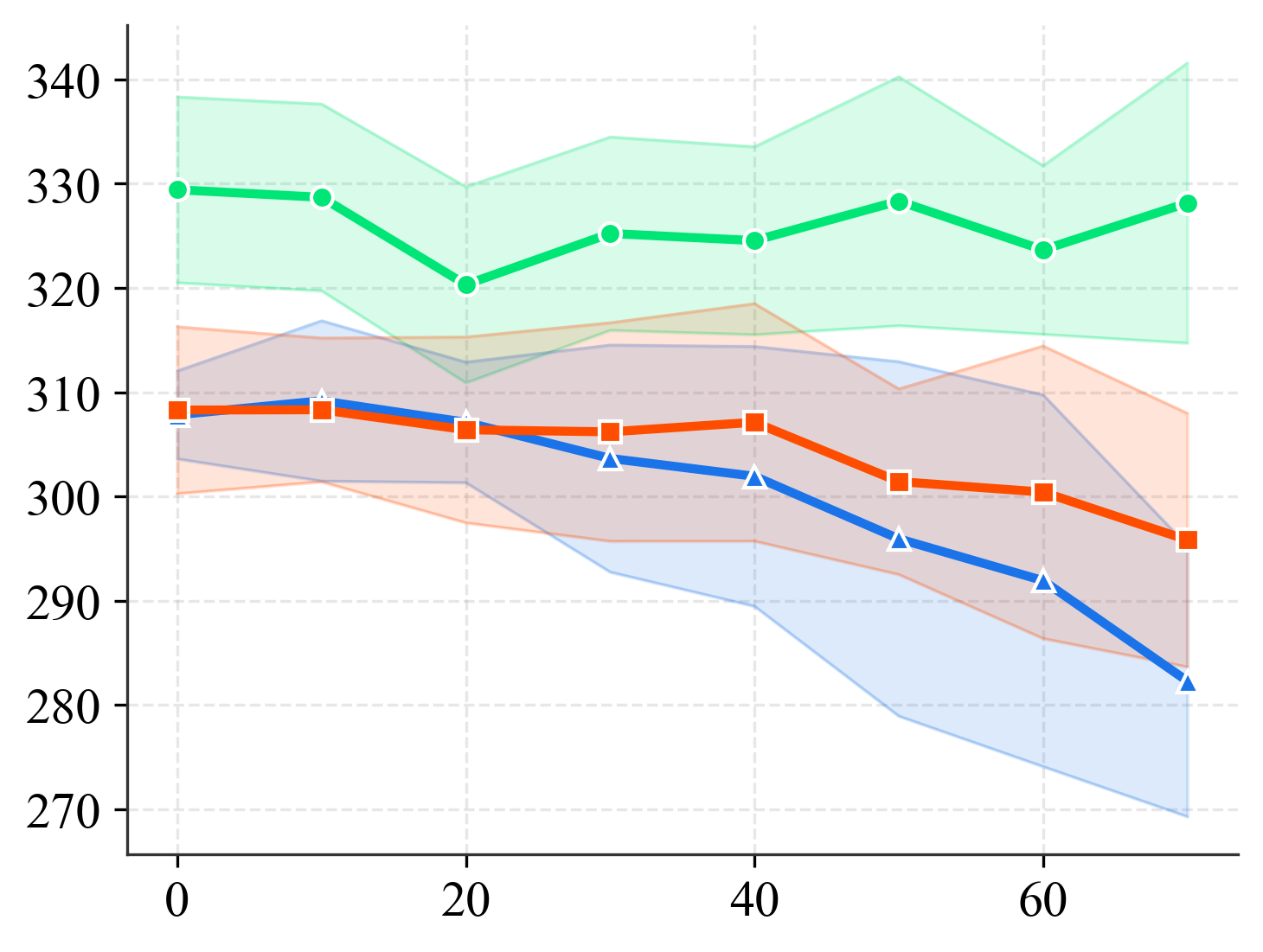} &
\cellcolor{runblue}\includegraphics[width=\linewidth]{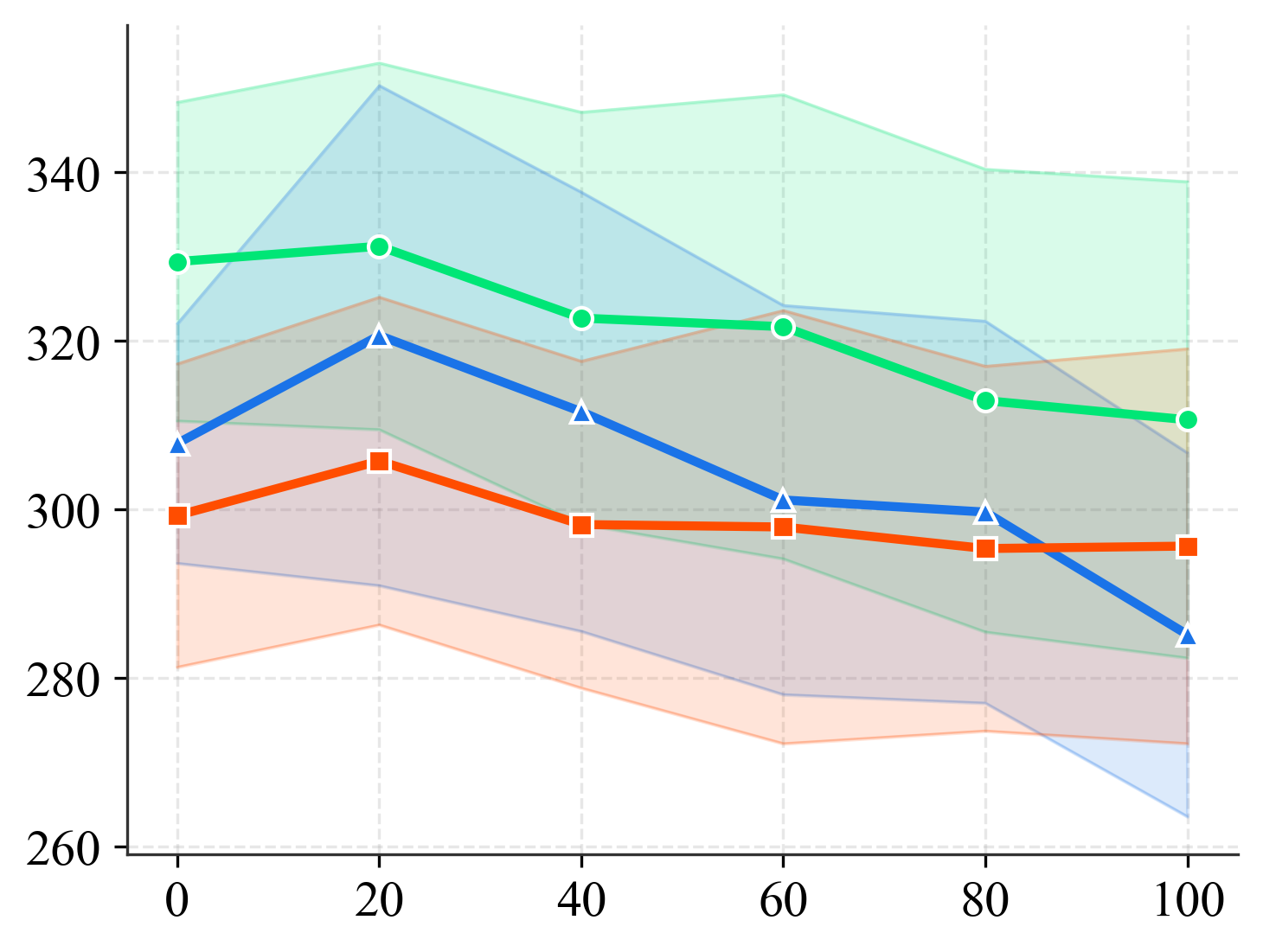} &
\cellcolor{runblue}\includegraphics[width=\linewidth]{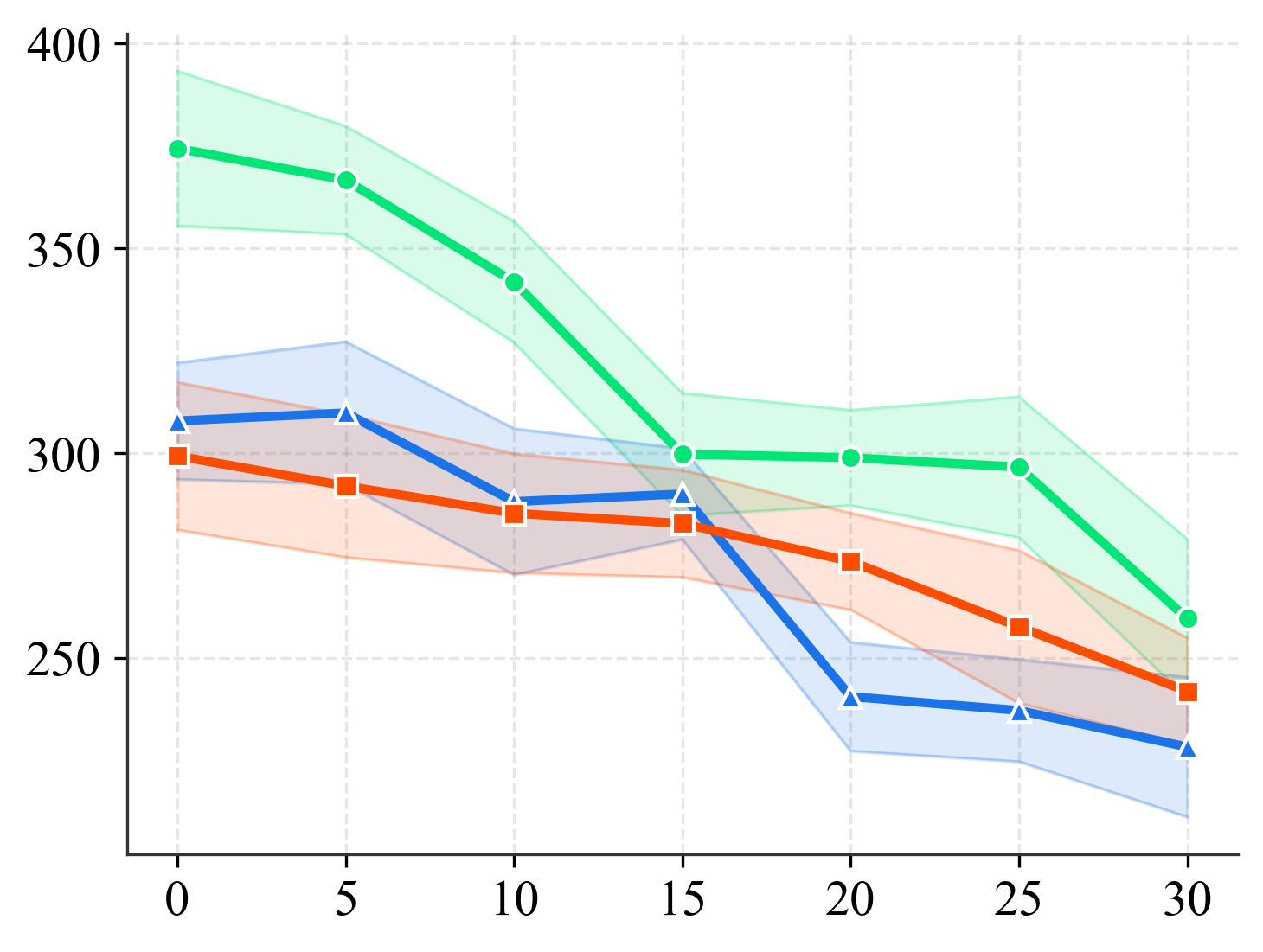} \\[1pt]

\cellcolor{walkgreen}\rotatebox{90}{\small\textbf{Walk}} &
\cellcolor{walkgreen}\includegraphics[width=\linewidth]{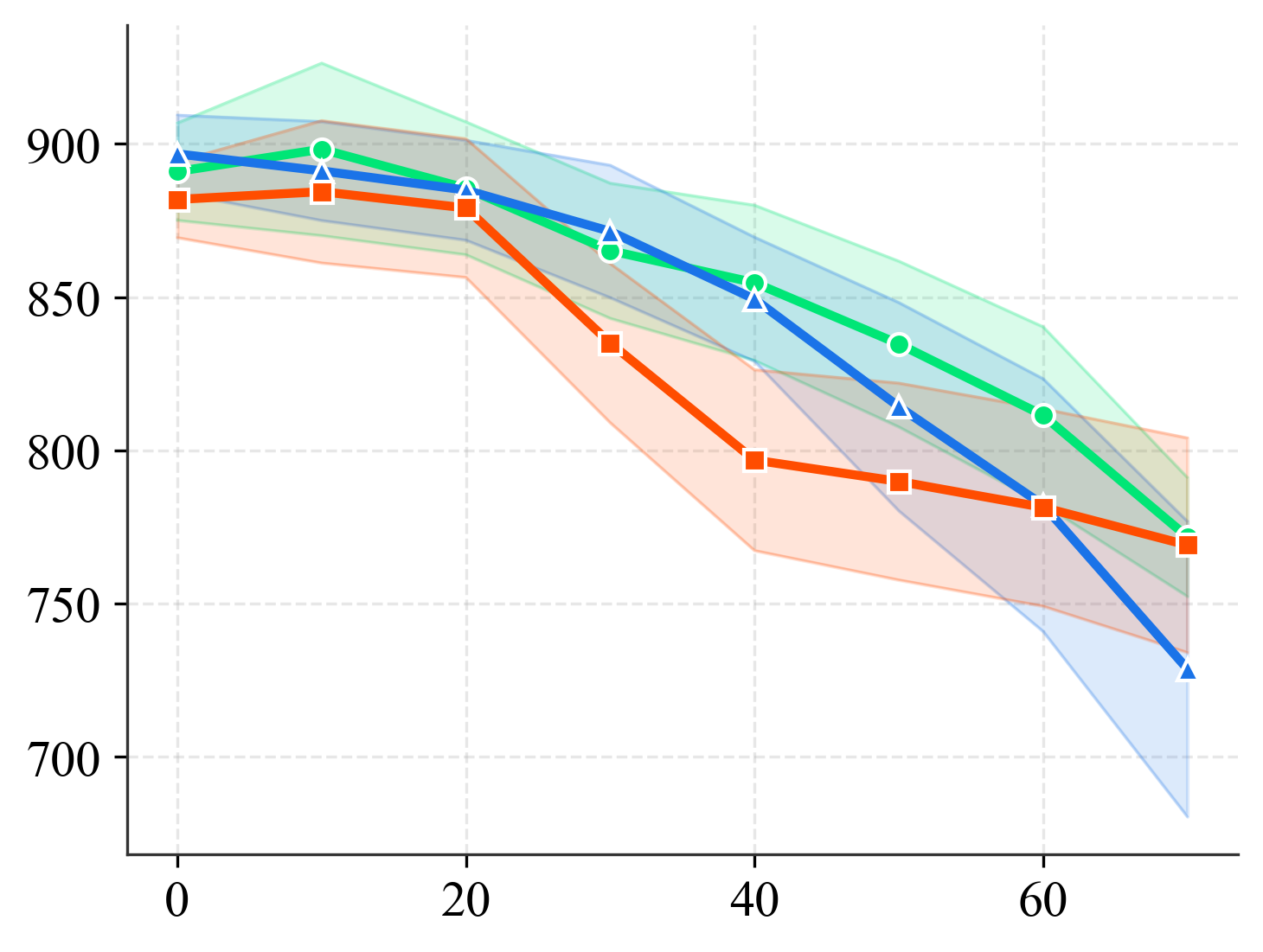} &
\cellcolor{walkgreen}\includegraphics[width=\linewidth]{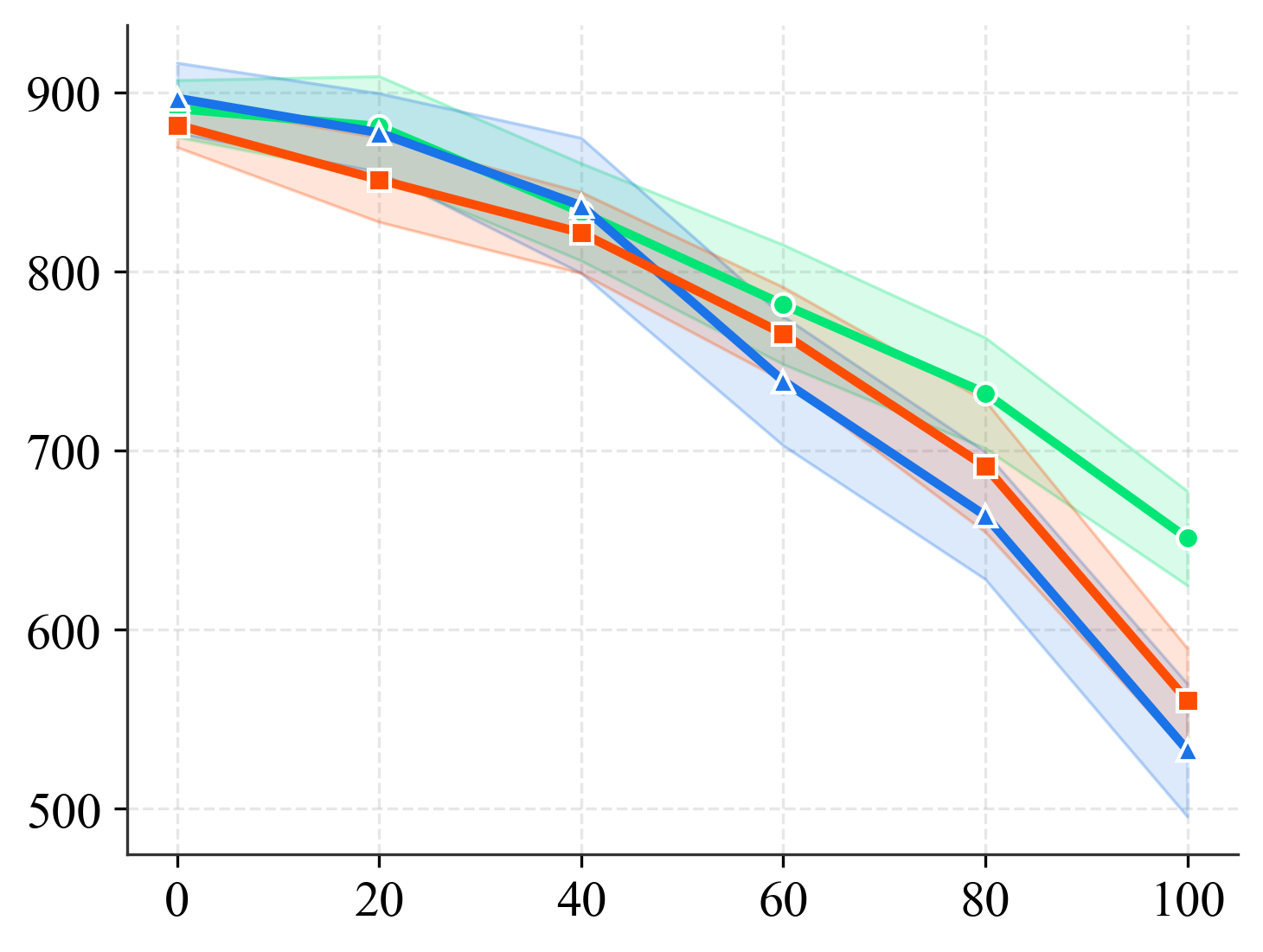} &
\cellcolor{walkgreen}\includegraphics[width=\linewidth]{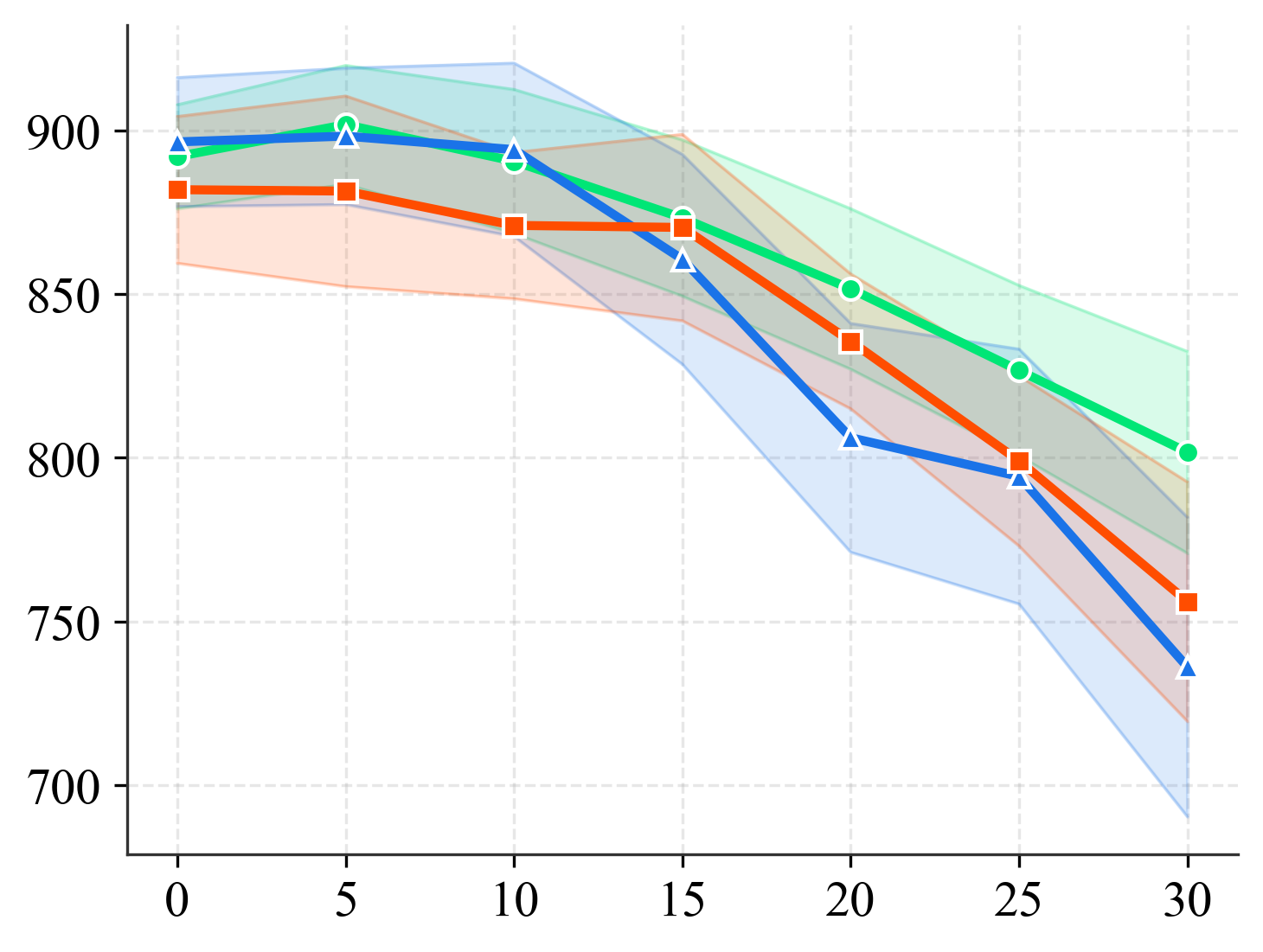} \\[1pt]

\cellcolor{fliporange}\rotatebox{90}{\small\textbf{Flip}} &
\cellcolor{fliporange}\includegraphics[width=\linewidth]{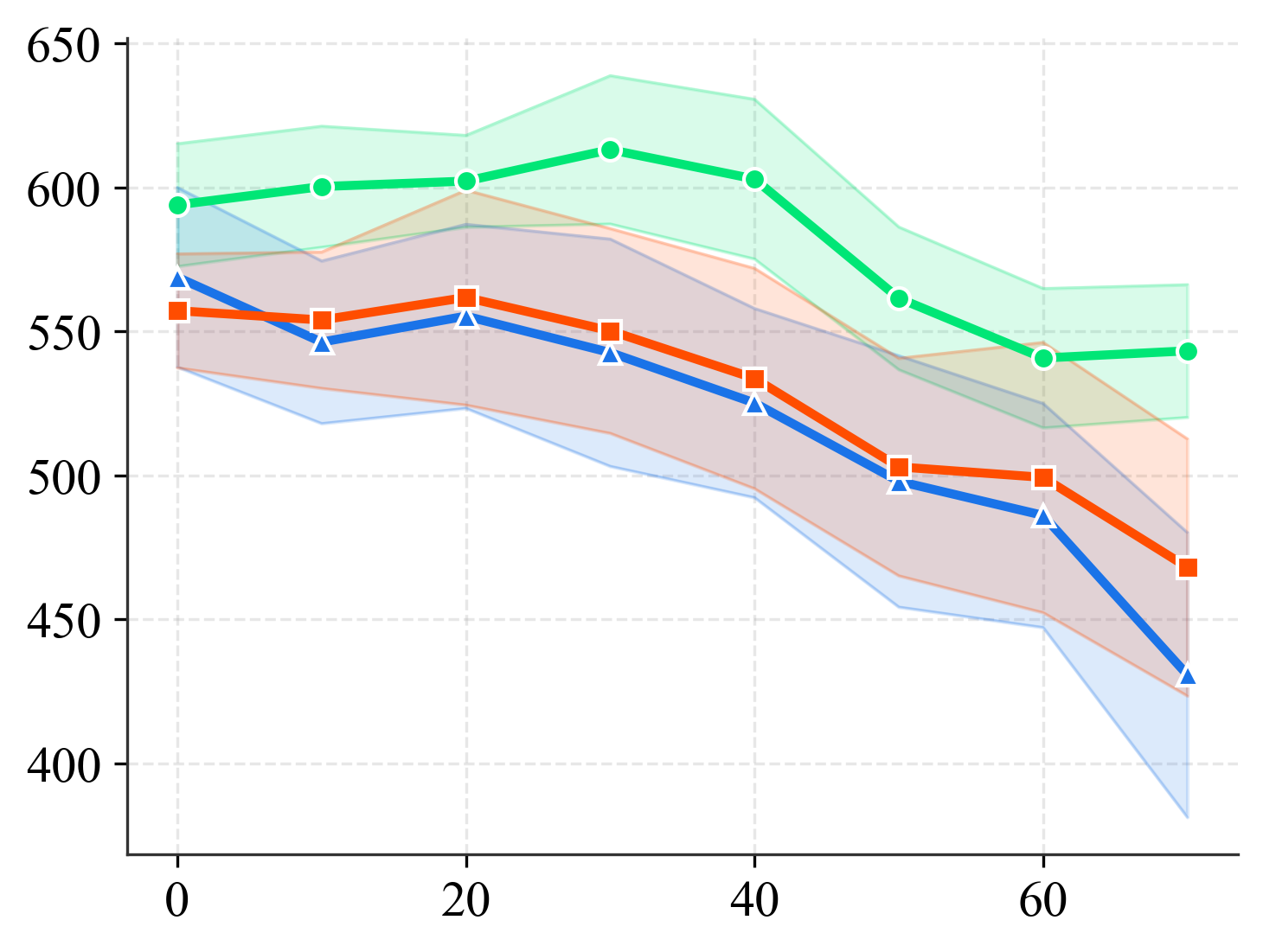} &
\cellcolor{fliporange}\includegraphics[width=\linewidth]{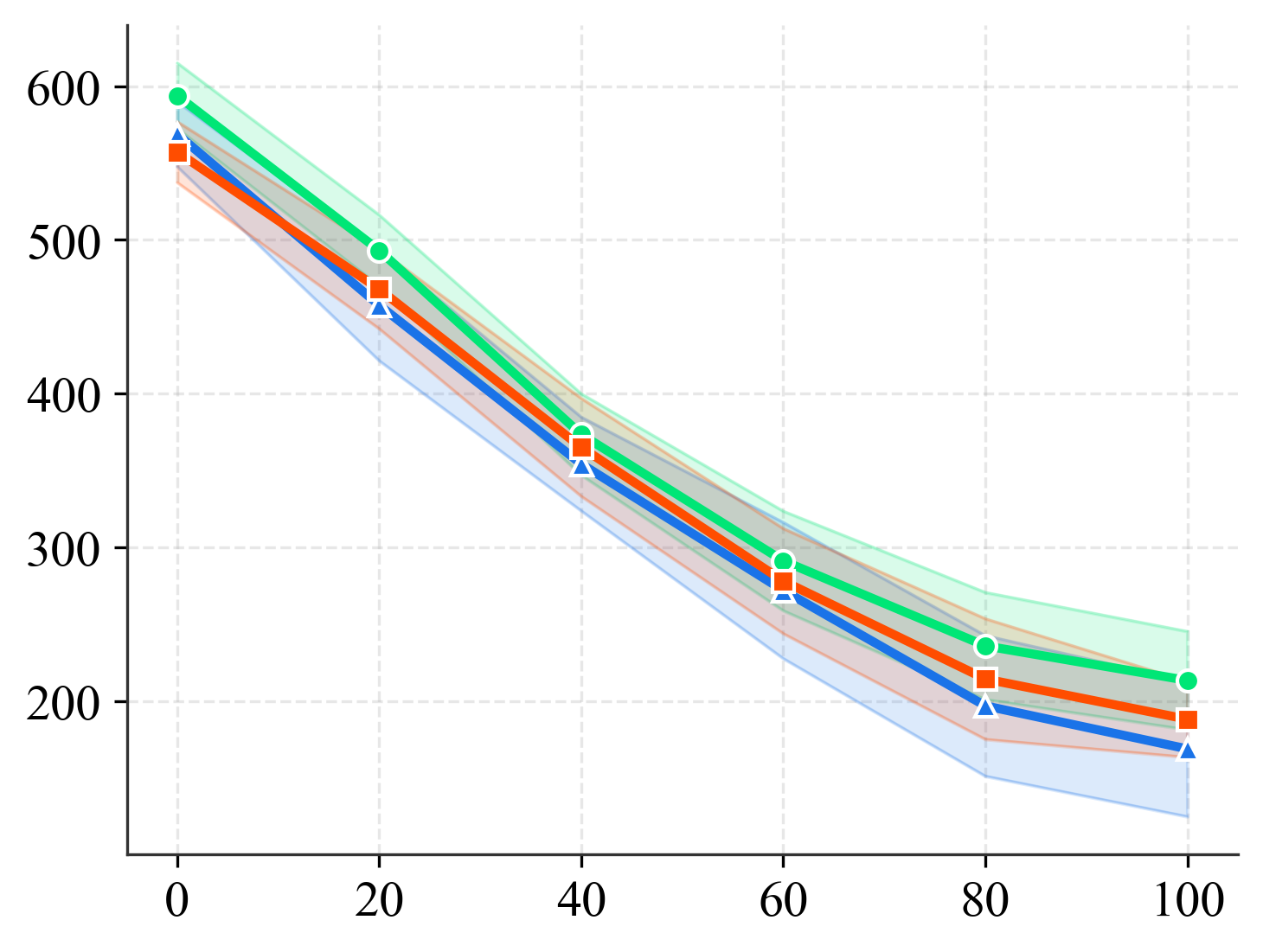} &
\cellcolor{fliporange}\includegraphics[width=\linewidth]{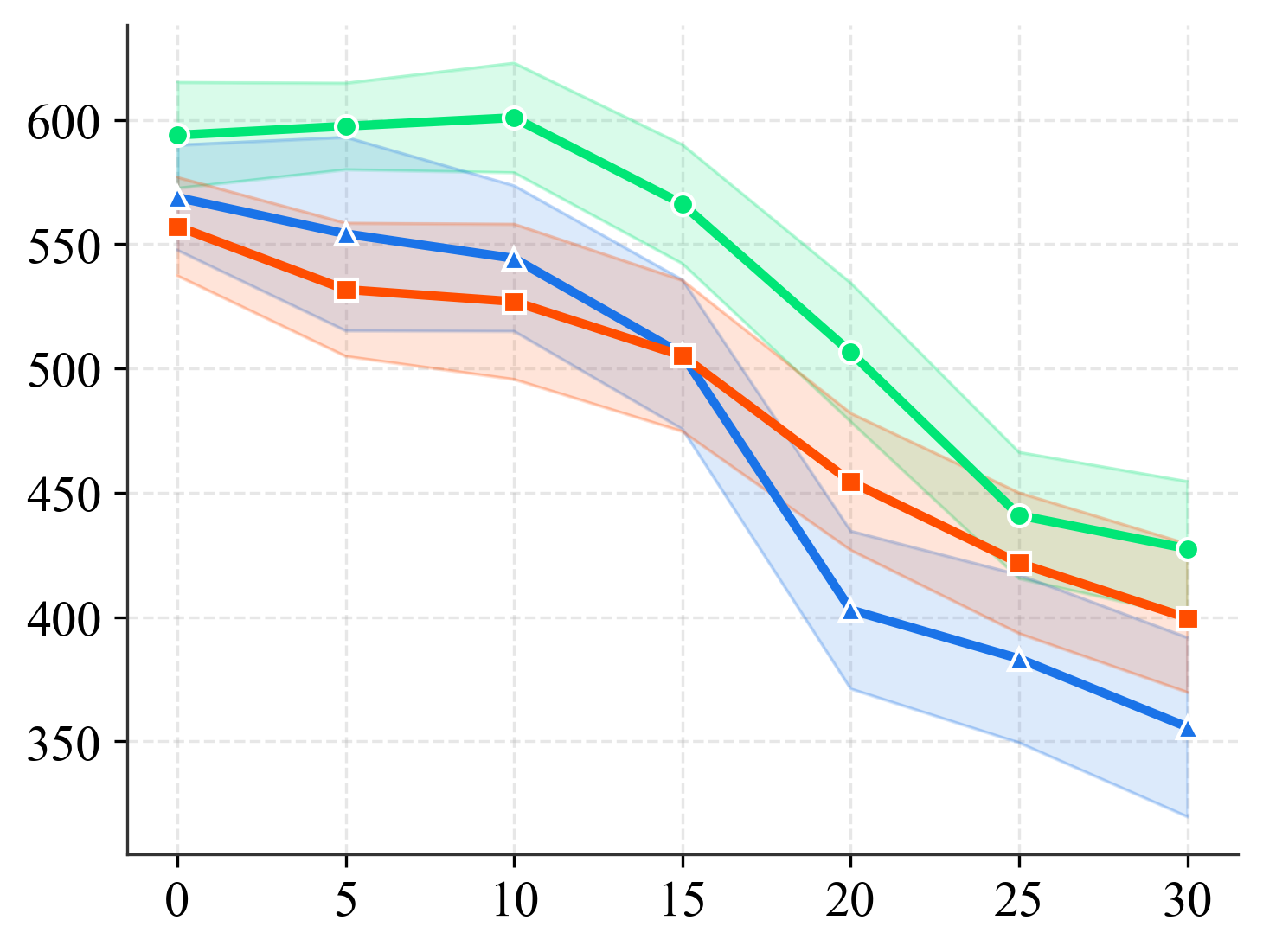} \\[1pt]

\cellcolor{standpurple}\rotatebox{90}{\small\textbf{Stand}} &
\cellcolor{standpurple}\includegraphics[width=\linewidth]{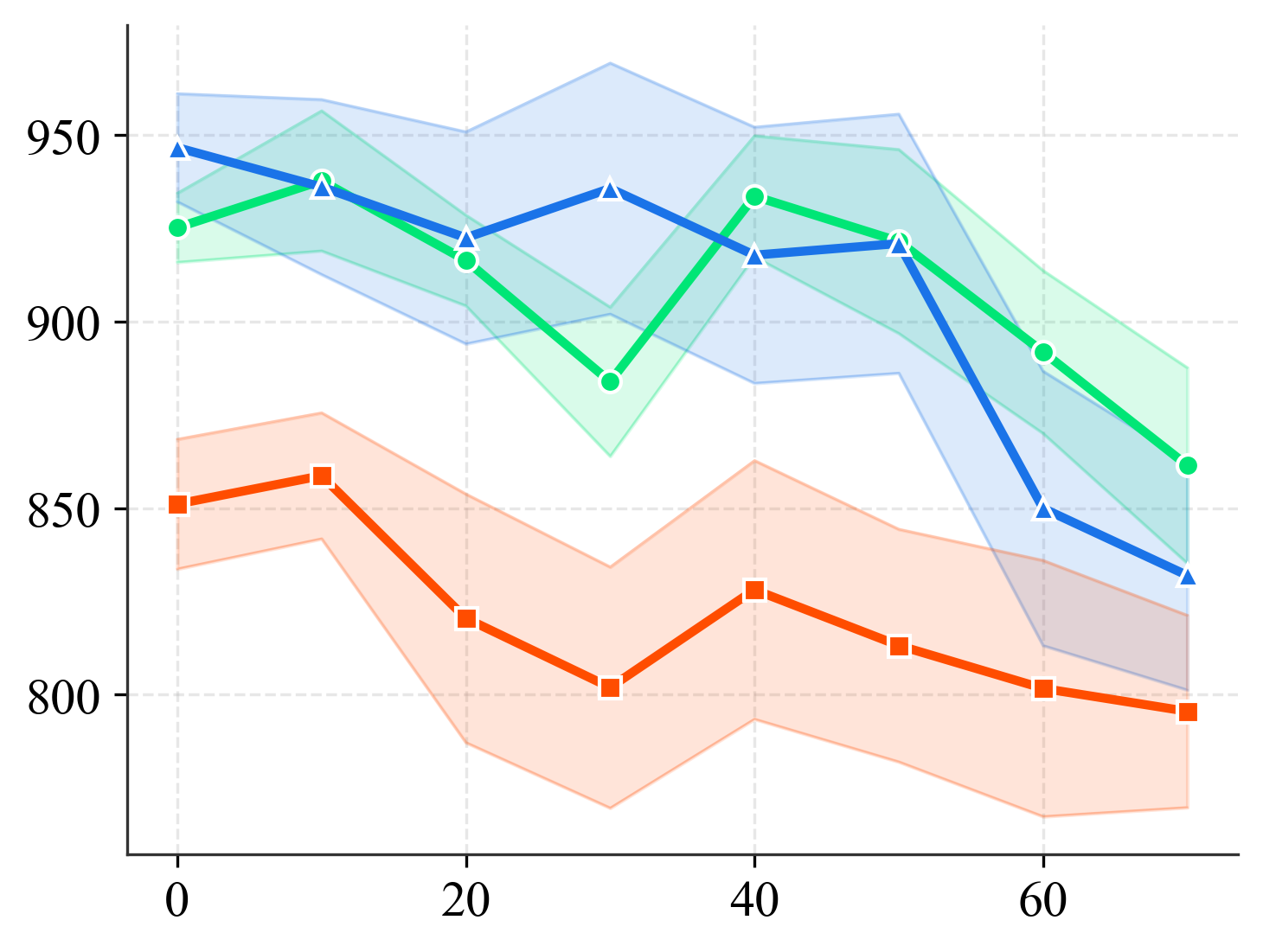} &
\cellcolor{standpurple}\includegraphics[width=\linewidth]{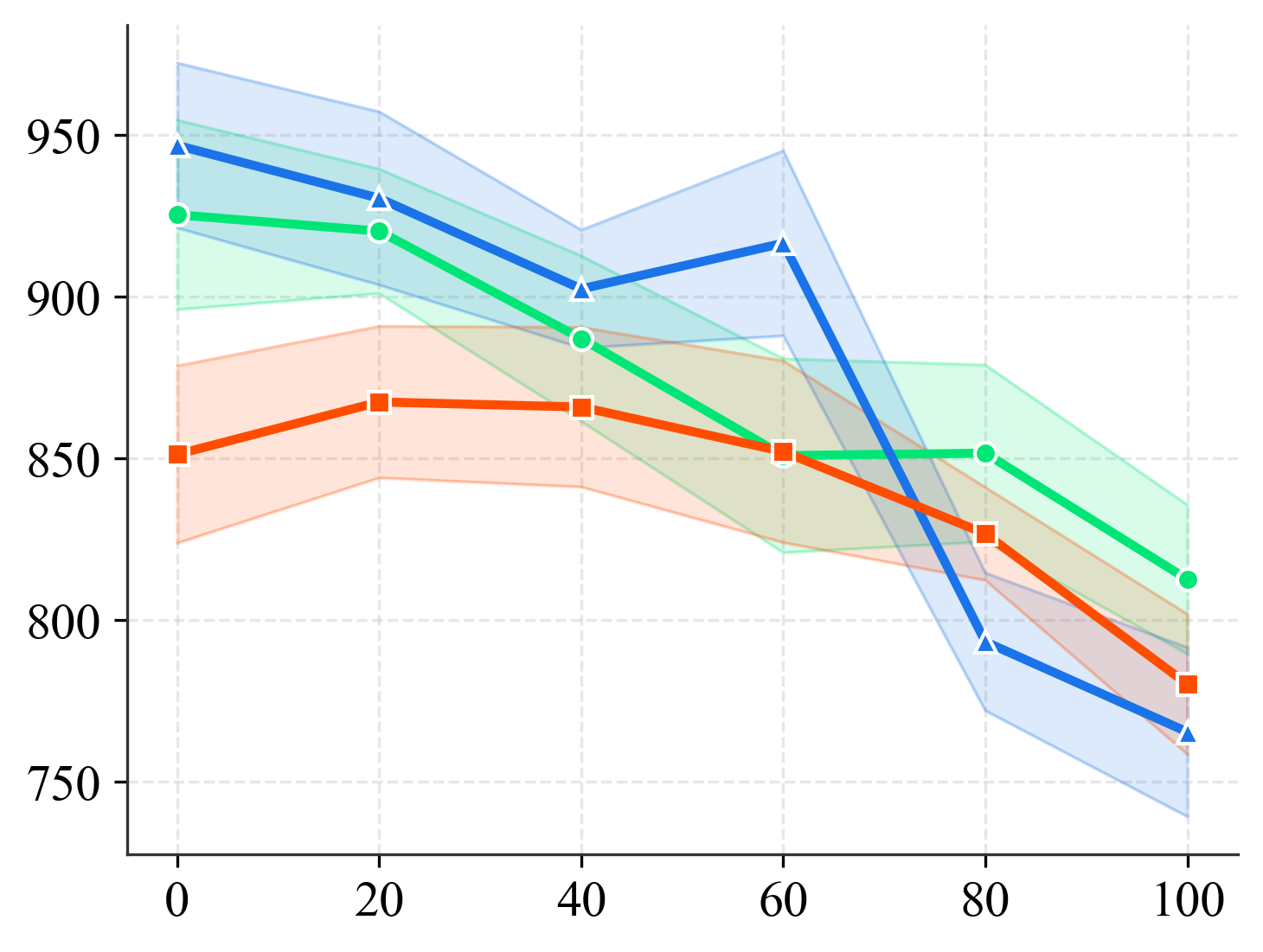} &
\cellcolor{standpurple}\includegraphics[width=\linewidth]{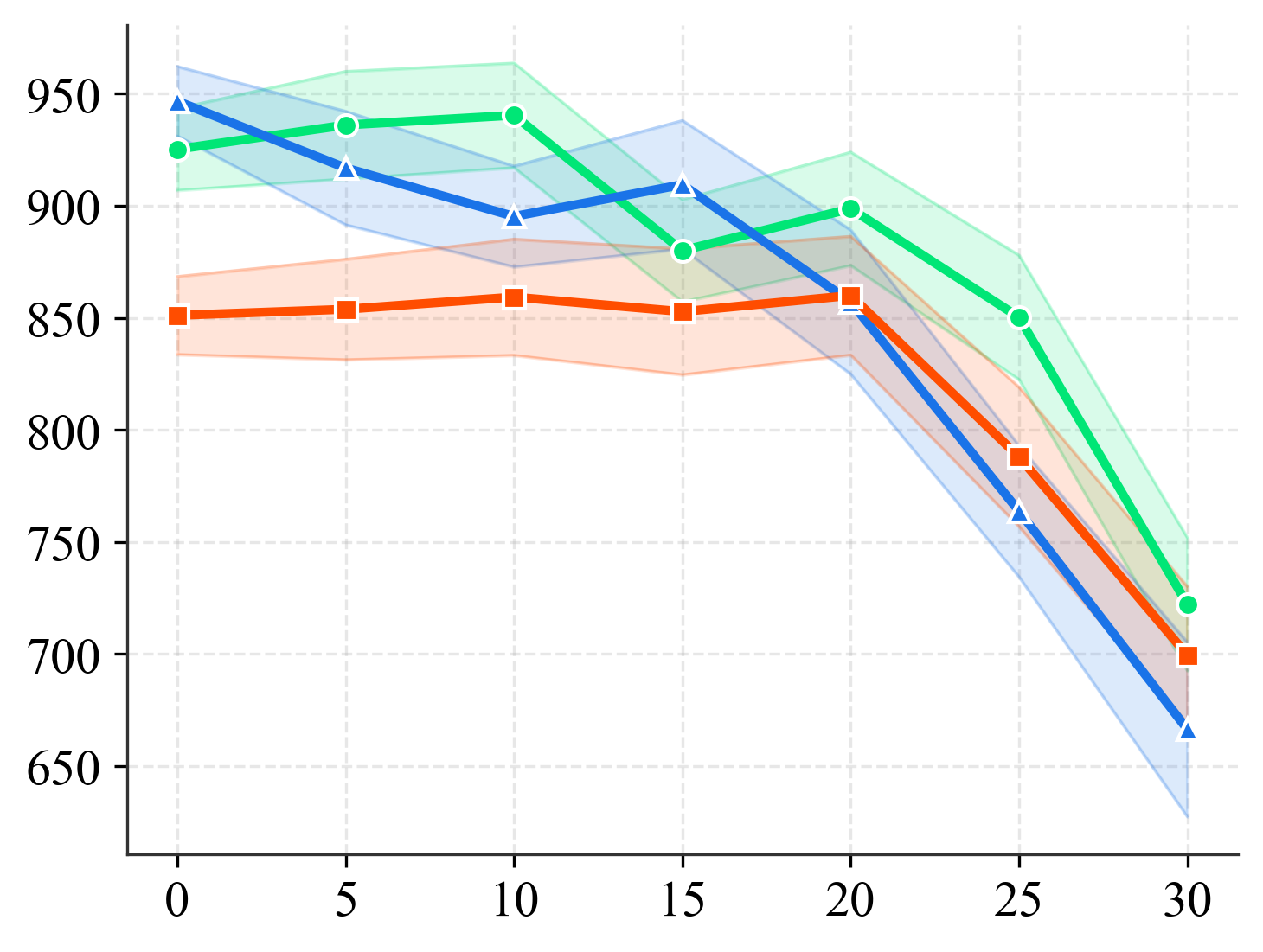} \\

\end{tabular}

\vspace{0.2cm}
\begin{minipage}{0.95\textwidth}
\centering
\footnotesize
\legendcircle{rbfmheavy}{RBFM-Heavy}\hspace{3.0em}
\legendtriangle{fbil}{FB-IL}\hspace{3.0em}
\legendsquare{rbfmlight}{RBFM-Light}
\end{minipage}

\caption{Walker performance with $95\%$ confidence interval across four tasks (rows) and three perturbation types (columns): gravity and body mass increase the physical load on the robot (\% change from nominal); joint friction loss adds passive resistive torque at each joint, simulating mechanical wear (absolute Nm per joint). Pretrained on PROTO data.}
\label{fig:walker_proto}
\end{figure}



\appsection{Discussion and Limitation}
\label{sec:app_limitations}

While answering (Q4) in Section~\ref{sec:experiments}, we observed that when the pretraining dataset is insufficient, performance becomes highly unpredictable as perturbations increase. This suggests that our approach relies on sufficiently rich exploratory data to learn well-structured FB representations, which are then used to derive robust policies for downstream tasks within a given environment. Figure~\ref{fig:pretrain_comparison_100k_vs_500k} illustrates the IQM metric during pretraining for the Walker domain, comparing models trained on a $100k$ RND dataset versus a $500k$ dataset. We observe that with only $100k$ samples, the FB representations fail to learn effectively. Consequently, the IQM scores on representative tasks such as run, stand, flip, and walk are significantly lower than those achieved with the $500k$ dataset. This confirms that the failure of our approach stems from insufficient pretraining, which adversely impacts task inference and, hence, the policy.


\begin{figure}[ht!]
\centering

\begin{subfigure}{0.45\textwidth}
    \centering
    \includegraphics[width=\linewidth]{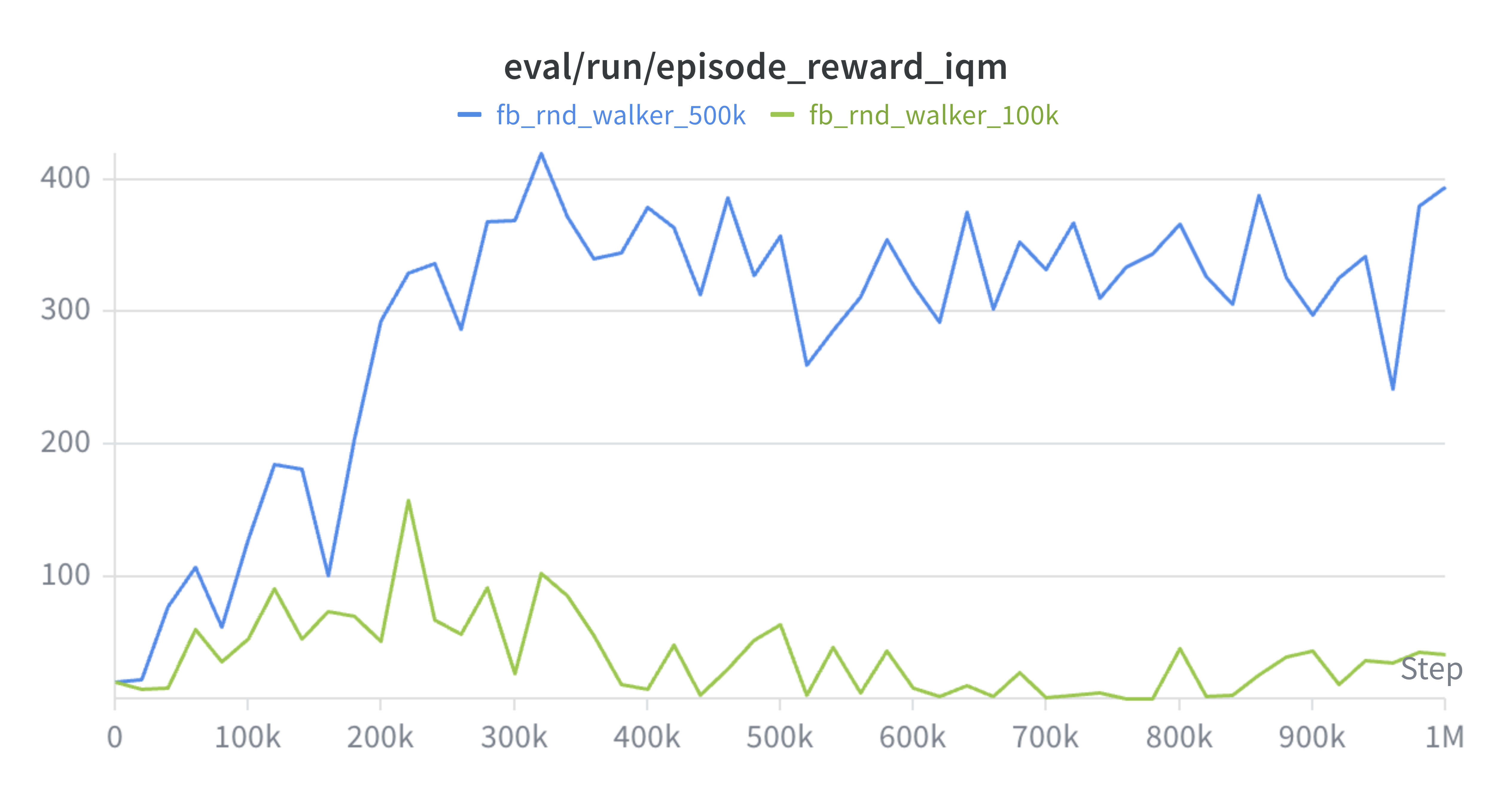}
    \caption{Run}
\end{subfigure}
\hfill
\begin{subfigure}{0.45\textwidth}
    \centering
    \includegraphics[width=\linewidth]{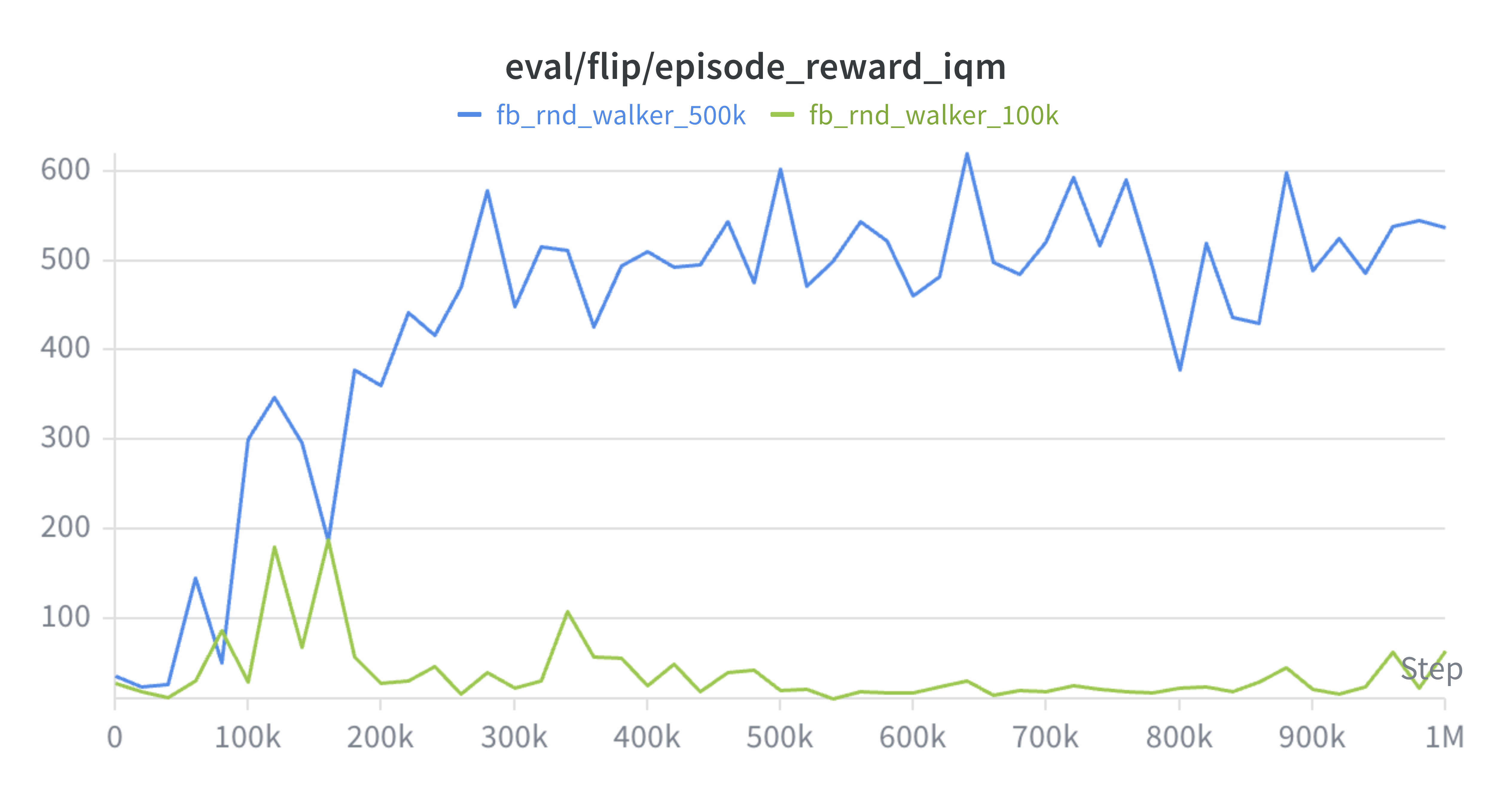}
    \caption{Flip}
\end{subfigure}

\vspace{0.5cm}

\begin{subfigure}{0.45\textwidth}
    \centering
    \includegraphics[width=\linewidth]{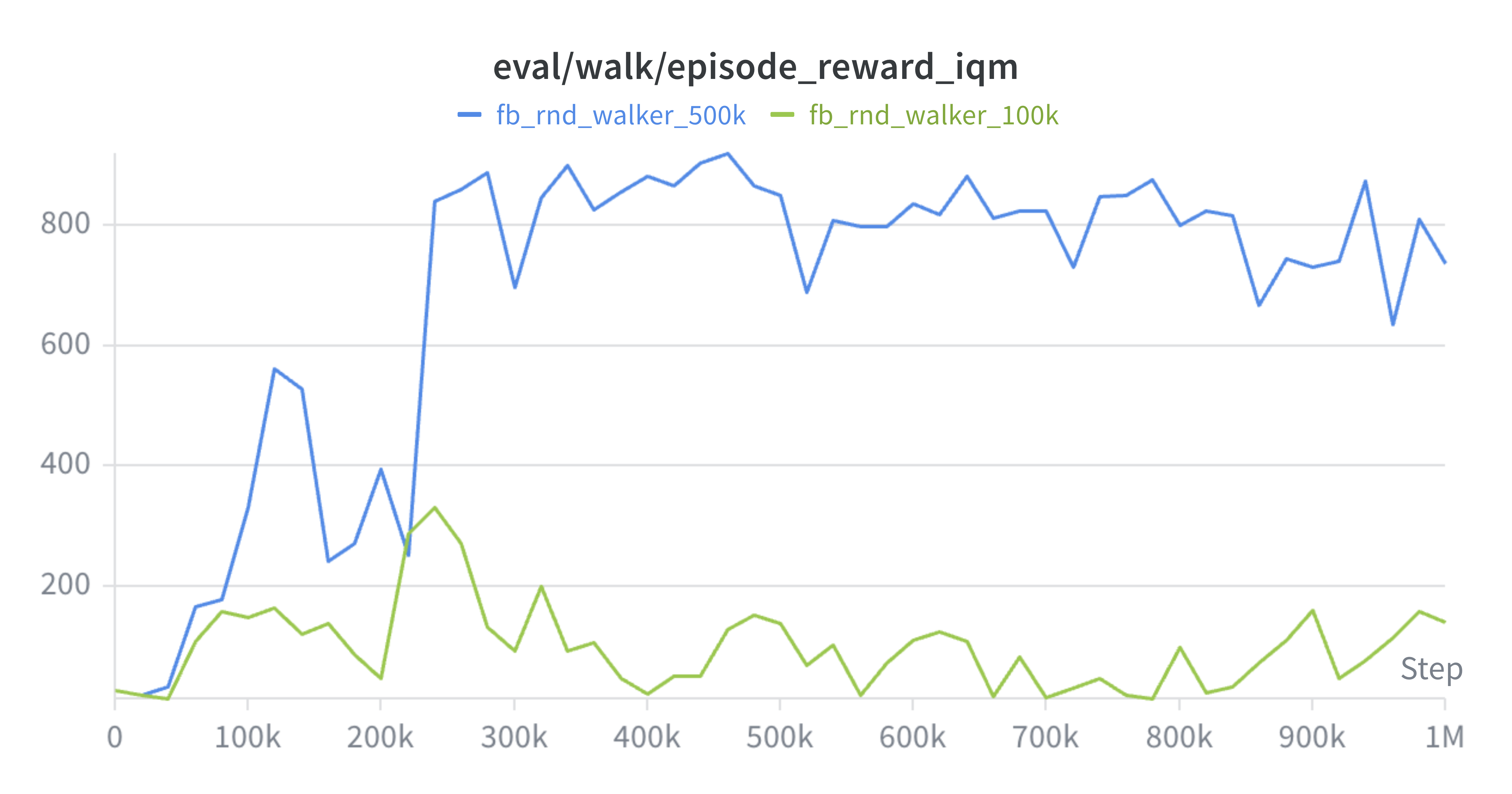}
    \caption{Walk}
\end{subfigure}
\hfill
\begin{subfigure}{0.45\textwidth}
    \centering
    \includegraphics[width=\linewidth]{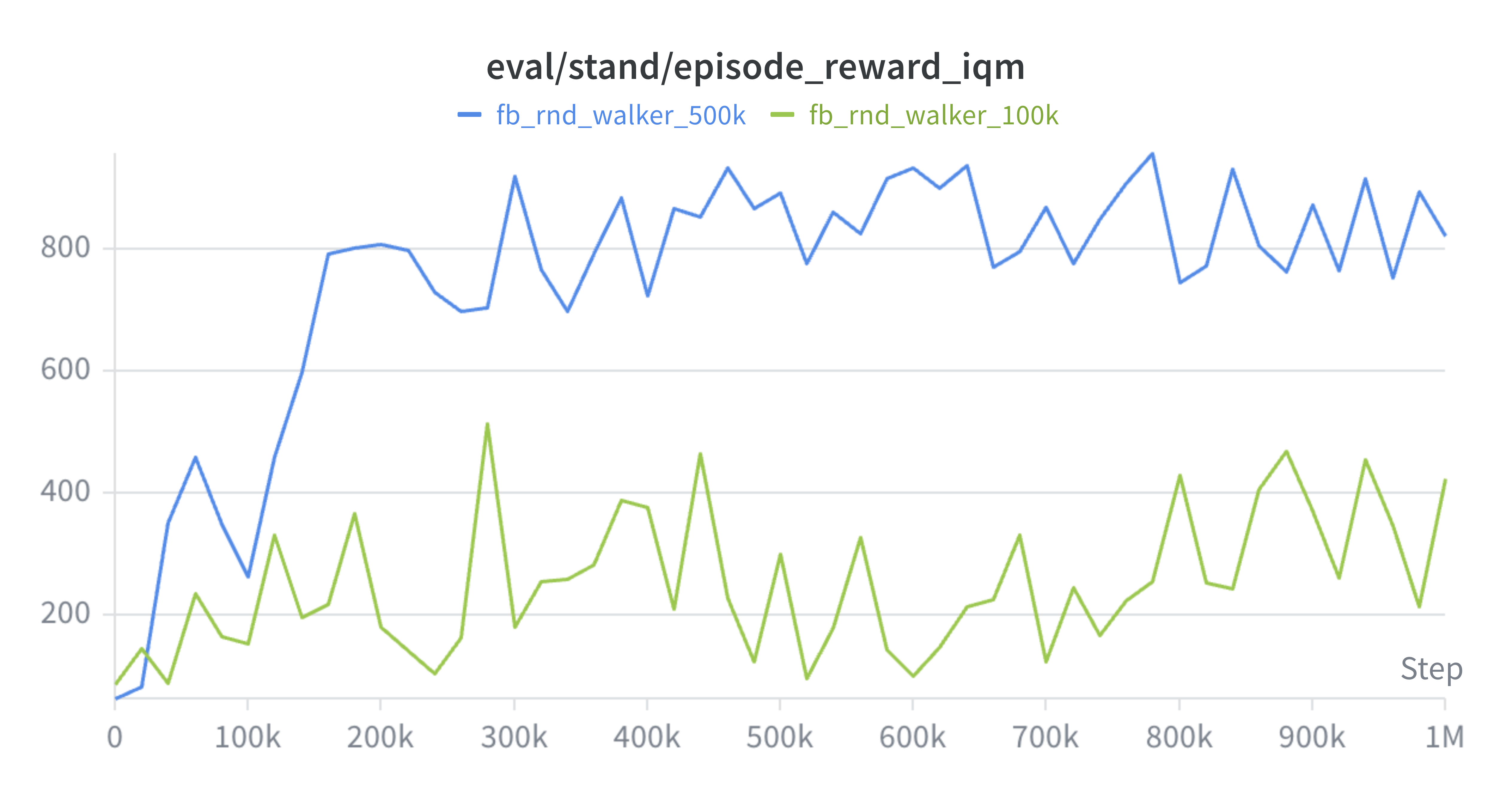}
    \caption{Stand}
\end{subfigure}

\caption{Comparison of pretraining of Walker on $100k$ RND dataset vs $500k$ RND dataset, where the models are evaluated every $20,000$ timesteps where we perform $10$ rollouts and
record the IQM.}
\label{fig:pretrain_comparison_100k_vs_500k}
\end{figure}

\clearpage

\end{document}